\documentclass{article}

\usepackage{microtype}
\usepackage{graphicx}
\usepackage{subfigure}
\usepackage{booktabs} 

\usepackage{hyperref}

\usepackage[accepted]{icml2024}

\usepackage{amsmath}
\usepackage{amssymb}
\usepackage{mathtools}
\usepackage{amsthm}

\usepackage[capitalize,noabbrev]{cleveref}

\theoremstyle{plain}
\newtheorem{theorem}{Theorem}[section]
\newtheorem{proposition}[theorem]{Proposition}
\newtheorem{lemma}[theorem]{Lemma}

\newtheorem{assumption}[theorem]{Assumption}
\theoremstyle{definition}
\newtheorem{definition}[theorem]{Definition}
\newtheorem{example}[theorem]{Example}
\theoremstyle{remark}
\newtheorem{remark}[theorem]{Remark}

\usepackage[textsize=tiny]{todonotes}

\usepackage{makecell}
\usepackage{xspace, bm, bbm, enumitem, color, wrapfig}
\usepackage{boldline, multirow}
\usepackage{titletoc, appendix}

\long\def\comment#1{} 

\DeclareMathOperator*{\argmin}{arg\,min} 

\newcommand{\xmath}[1] {\ensuremath{#1}\xspace}
\newcommand{\blmath}[1] {\xmath{\bm{#1}}}

\newcommand{\norm}[1] {\xmath{\left\| #1 \right\|}}

\newcommand{\ub}{{\blmath u}}

\newcommand{\wb}{{\blmath w}}
\newcommand{\xb}{{\blmath x}}
\newcommand{\yb}{{\blmath y}}

\newcommand{\Ac}{\mathcal{A}}
\newcommand{\Bc}{\mathcal{B}}
\newcommand{\Cc}{\mathcal{C}}

\newcommand{\Nc}{\mathcal{N}}
\newcommand{\Qc}{\mathcal{Q}}
\newcommand{\Tc}{\mathcal{T}}

\newcommand{\Xc}{\mathcal{X}}

\newcommand{\Yc}{\mathcal{Y}}

\newcommand{\Nd}{{\mathbb N}}
\newcommand{\Rd}{{\mathbb R}}

\newcommand{\Dc}{{{\mathcal D}}}

\newcommand{\beq}{\begin{equation}}
\newcommand{\eeq}{\end{equation}}
\newcommand{\beqa}{\begin{eqnarray}}
\newcommand{\eeqa}{\end{eqnarray}}

\newcommand{\lambdab}{\boldsymbol{\lambda}}

\newcommand{\indicator}[1]{\mathbbm{1}_{\{#1\}}}

\newcommand{\floor}[1]{\left\lfloor#1\right\rfloor}
\newcommand{\ceil}[1]{\left\lceil#1\right\rceil} 
\usepackage{graphicx}
\graphicspath{{figs/}}

\newcommand{\ReLUtwo}[3]{#1 \stackrel{\sigma}{\rightarrow} #2 \rightarrow #3}
\newcommand{\ReLUTwo}[3]{#1 \stackrel{\sigma}{\rightarrow} #2 \stackrel{\sigma}{\rightarrow} #3}
\newcommand{\ReLUthree}[4]{#1 \stackrel{\sigma}{\rightarrow} #2 \stackrel{\sigma}{\rightarrow} #3 \rightarrow #4}
\newcommand{\ReLUThree}[4]{#1 \stackrel{\sigma}{\rightarrow} #2 \stackrel{\sigma}{\rightarrow} #3 \stackrel{\sigma}{\rightarrow} #4}
\newcommand{\ReLUfour}[5]{#1 \stackrel{\sigma}{\rightarrow} #2 \stackrel{\sigma}{\rightarrow} #3 \stackrel{\sigma}{\rightarrow} #4 \rightarrow #5}
\newcommand{\ReLUFour}[5]{#1 \stackrel{\sigma}{\rightarrow} #2 \stackrel{\sigma}{\rightarrow} #3 \stackrel{\sigma}{\rightarrow} #4 \stackrel{\sigma}{\rightarrow} #5}
\newcommand{\exact}{feasible architecture~} 
\newcommand{\exacts}{feasible architectures~} 
\newcommand{\SIG}{\texttt{SIG}}
\newcommand{\width}{m} 
\renewcommand{\algorithmiccomment}[1]{\hfill$\triangleright\;${#1}}

\icmltitlerunning{Defining Neural Network Architecture through Polytope Structures of Dataset}

\begin{document}
\allowdisplaybreaks

\twocolumn[
\icmltitle{Defining Neural Network Architecture through Polytope Structures of Dataset}




\begin{icmlauthorlist}
\icmlauthor{Sangmin Lee}{mat}
\icmlauthor{Abbas Mammadov}{mat,comp}
\icmlauthor{Jong Chul Ye}{mat,ai}
\end{icmlauthorlist}

\icmlaffiliation{mat}{Department of Mathematical Science, KAIST, Daejeon, Korea}
\icmlaffiliation{comp}{School of Computing, KAIST, Daejeon, Korea}
\icmlaffiliation{ai}{Kim Jaechul Graduate School of AI, KAIST, Daejeon, Korea}

\icmlcorrespondingauthor{Jong Chul Ye}{jong.ye@kaist.ac.kr}

\icmlkeywords{Machine Learning, ICML}

\vskip 0.3in
]



\makeatletter\def\Hy@Warning#1{}\makeatother 
\printAffiliationsAndNotice{}  

\begin{abstract}
Current theoretical and empirical research in neural networks suggests that complex datasets require large network architectures for thorough classification, yet the precise nature of this relationship remains unclear. 
This paper tackles this issue by defining upper and lower bounds for neural network widths, which are informed by the polytope structure of the dataset in question. 
We also delve into the application of these principles to simplicial complexes and specific manifold shapes, explaining how the requirement for network width varies in accordance with the geometric complexity of the dataset.
Moreover, we develop an algorithm to investigate a converse situation where the polytope structure of a dataset can be inferred from its corresponding trained neural networks. 
Through our algorithm, it is established that popular datasets such as MNIST, Fashion-MNIST, and CIFAR10 can be efficiently encapsulated using no more than two polytopes with a small number of faces.
\end{abstract}

\section{Introduction} \label{sec: intro}

\begin{figure*}[t]
    \centering
    \subfigure[]{\includegraphics[width=0.3\textwidth]{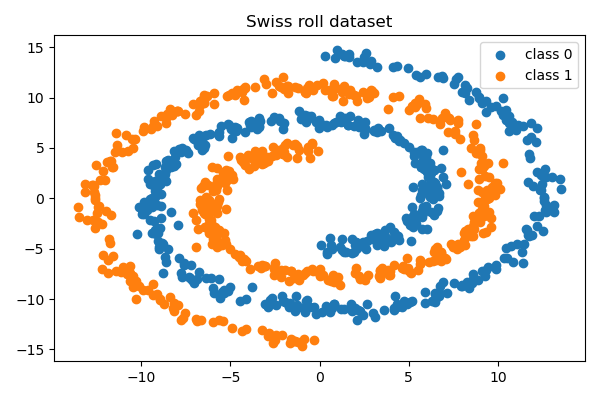}}
    \hfill 
    \subfigure[]{\includegraphics[width=0.3\textwidth]{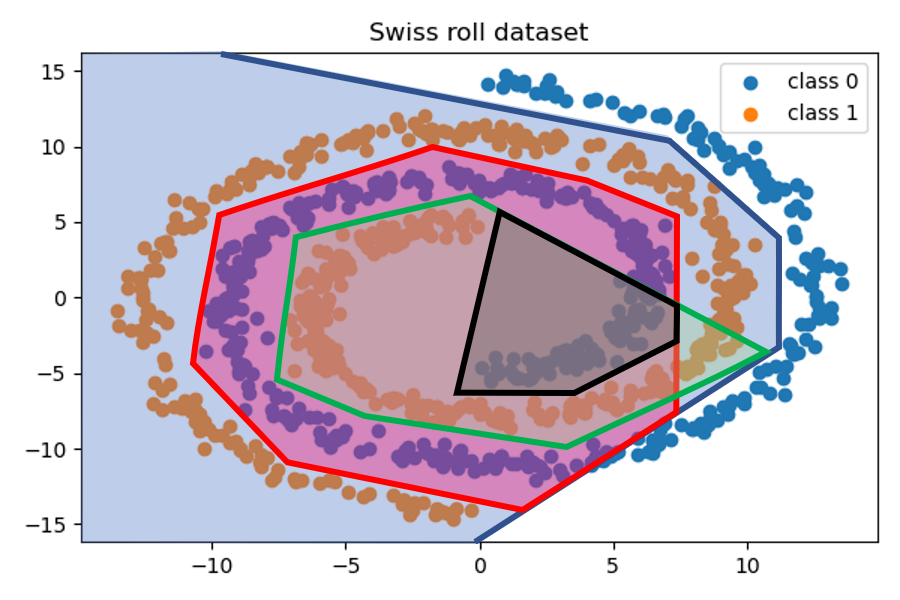}}
    \hfill 
    \subfigure[]{\includegraphics[width=0.3\textwidth]{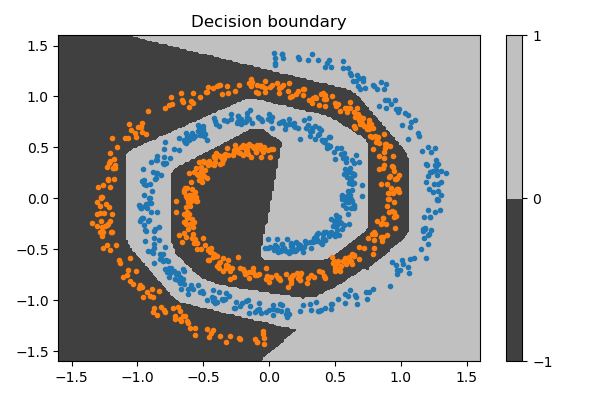}}
    \caption{
    What type of neural network architecture is capable of effectively classifying the swiss roll dataset depicted in (a)? By establishing a collection of covering polytopes to enclose one class, as illustrated in (b), our result demonstrates that a three-layer ReLU network with the architecture $\;\ReLUTwo{2}{20}{4} \rightarrow1$ can successfully achieve this classification task, as exemplified in (c). 
    }
    \label{fig: swiss}
\end{figure*}

To comprehend the remarkable performance of deep neural networks (DNNs), extensive research has delved into their architectures and the universal approximation property (UAP). The UAP of two-layer neural networks on compact sets was initially proven by \citet{cybenko1989approximation}, sparking widespread exploration of the UAP in diverse settings for DNNs. Studies have focused on determining the minimal depths and widths of deep ReLU networks required for UAP \citep{hornik1991approximation, park2020minimum}. These foundational results contribute to unraveling the intricate relationship between approximation power and neural network architectures.

However, a converse problem to address the influence of the characteristics of training datasets necessary to attain the UAP in neural networks has received relatively less attention.
For example, when analyzing the swiss roll dataset shown in Figure \ref{fig: swiss}(a), an important practical question emerges: \emph{What are the effective depth and width required for the complete classification of this dataset?} Despite the practical relevance of this inquiry in the context of training neural networks, existing theoretical results only offer basic lower bounds (minimum depth of 2 \citep{hornik1991approximation} and a width of $\max\{d_x+1, d_y\}$ \citep{park2020minimum}), which are often impractical for real applications. 
While a range of empirical evidence indicates that increasing the depth or width of networks could lead to successful outcomes, there remains an absence of theoretical assurances to foresee these results.

In this paper, we therefore tackle the challenge of identifying the optimal neural network architecture for classifying a given dataset. This task is approached through the lens of the polytope structure of deep ReLU networks, a subject that has garnered considerable attention in recent studies \citep{black2022interpreting, grigsby2022transversality, berzins2023polyhedral, huchette2023deep}. In fact, our primary theoretical goal is to address the ``multiple manifold problem," introduced by \citet{buchanan2020deep}:
\emph{For given two disjoint topological spaces $\Xc_+$ and $\Xc_-$, what is the optimal architecture for the neural network $\Nc$ such that $\Nc(\xb) > 0$ for all $\xb \in \Xc_+$ and $\Nc(\xb) < 0$ otherwise?} 

By utilizing the geometric properties of DNNs, here we provide a comprehensive answer to this question. Our approach involves determining both upper and lower bounds for the depth and widths of networks required for dataset classification, based on the polytope covering of the datasets. 
Specifically, we explicitly construct a neural network with practical applicability.
For example, our discovery in Theorem \ref{thm: compact} reveals that the swiss roll dataset in Figure \ref{fig: swiss}(a) can be efficiently classified using a three-layer ReLU network with 24 neurons, as depicted in Figure \ref{fig: swiss}(c).

Another important contribution of this paper is  the investigation into the converse situation, demonstrating that
 trained neural networks inherently capture the geometric properties of the dataset and enable the extraction of the dataset's polytope structure. 
As for demonstrating practical use, we uncover and discuss simple geometric traits of real-world datasets such as MNIST, Fashion-MNIST, and CIFAR10, achieved through the training of neural networks. 

Importantly, our contributions can be summarized as follows:
\vspace{-2mm}
\begin{itemize}
    \item 
    \textbf{Explicit construction of networks.}
    We introduce the novel concept of a \emph{polytope-basis cover} (Definition \ref{def: convex polytope cover}), which serves to describe the geometric structure of the dataset in detail. Building on this, we propose the design of a three-layer ReLU network, specifically tailored to efficiently classify the dataset in question, using its polytope-basis cover as a guiding framework (Theorem \ref{thm: compact}). 

%
    \item 
    \textbf{Bounds on network widths.}
    We define both upper and lower bounds for the width of a neural network necessary to classify a given convex polytope region, taking into account the number of its faces (Proposition \ref{prop: convex polytope}). Furthermore, we derive upper bounds on network widths when the dataset $\Xc$ is structured as a simplicial complex or can be covered by a difference of prismatic polytopes (Theorem \ref{thm: simplicial complex} and \ref{thm: betti numbers}). 
    These bounds are correlated with the number of facets or the Betti numbers of $\Xc$, demonstrating an interplay between the dataset's inherent geometry and the required network architecture.
    

     \item 
     \textbf{Investigating dataset geometry.} Building on our findings, we demonstrate that it is possible to investigate the geometric features of the dataset by training a neural network (Theorem \ref{thm: three-layer polytope cover}).
     Specifically, we develop algorithms that are able to identify a polytope basis-cover for given datasets (Algorithm \ref{alg: compressing}).
     Our results show that each class within the MNIST, Fashion-MNIST, and CIFAR10 datasets can be effectively distinguished using no more than two convex polytopes, each consisting of fewer than 30 faces (Table \ref{tab: result}). 
\end{itemize}



\section{Preliminaries} \label{sec: preliminaries} 
\paragraph{Notation.}

In this paper, we focus primarily on the binary classification problem, aiming to separate two disjoint topological spaces denoted as $\Xc_+$ and $\Xc_-$, or two classes of finite points denoted as $\Dc_+$ and $\Dc_-$. The training dataset is represented as $\Dc=\Dc_+\cup \Dc_- = \{(\xb_i, y_i)\}_{i=1}^n$, where $\xb_i\in\Rd^d$ and $y_i \in \{0,1\}$. 
Throughout the paper, we denote scalars by lowercase letters and vectors by boldface lowercase letters. 
For a positive integer $m$, $[m]$ represents the set $\{ 1,2,\cdots,m \}$.
The ReLU activation function is denoted by $\sigma(x):= \textup{ReLU}(x)=\max\{0, x \}$, and it is applied to a vector coordinate-wisely.
The sigmoid activation function is denoted as $\texttt{SIG}(x) = \frac{1}{1+e^{-x}}$.
The max pooling operation is represented as $\texttt{MAX}:\Rd^d \rightarrow \Rd$, which returns the maximum component of the input vector.
The $\varepsilon$ neighborhood of a topological space $\Xc \subset \Rd^d$ is defined by $\Bc_\varepsilon(\Xc):=\{\xb \in \Rd^d : \min_{\yb\in\Xc} \norm{\xb - \yb}_2 <\varepsilon\}$.
The indicator function is denoted by
\begin{align*}
    \indicator{c} := 
    \begin{cases}
    	1, \qquad \text{if $c$ is true}  , \\ 
    	0, \qquad \text{otherwise.}
    \end{cases}
\end{align*}
Additionally, we define a \emph{convex polytope} as an intersection of hyperspaces, as defined below: 
\begin{definition} \label{def: convex polytope} 
    A nonempty set $C \subset \Rd^d$ is called a \emph{convex polytope with $m$ faces} if there exist $\wb_k \in \Rd^d$ and $b_k \in \Rd$ for $k \in [m]$ such that $C = \bigcap\limits_{k=1}^m \{\xb\in\Rd^d~ |~\wb_k^\top\xb+b_k \le0\}$. \vspace{-2mm}
\end{definition}

\paragraph{Network architectures.}
In this paper, the terminology \emph{architecture} refers to the structure of a neural network, which means the depth and the width of hidden layers, and is often denoted by $\Ac$. 
A $D$-layer neural network $\Nc:\Rd^d \rightarrow \Rd$ with hidden layer widths $d_1, d_2, \cdots, d_{D-1}$ and activation functions $\texttt{ACT}_1,  \texttt{ACT}_2, \cdots, \texttt{ACT}_{D}$ is represented by $d\stackrel{\texttt{ACT}_1}{\rightarrow}d_1\stackrel{\texttt{ACT}_2}{\rightarrow}d_2\stackrel{\texttt{ACT}_3}{\rightarrow} \cdots \stackrel{\texttt{ACT}_{D-1}}{\rightarrow} d_{D-1} \stackrel{\texttt{ACT}_D}{\rightarrow} 1$. When the activation function is the identity, we add nothing on the arrow.
 For example, $\ReLUtwo{d}{\width}{1}$ denotes a two-layer ReLU network with $\width$ neurons, presented by 
\begin{align} \label{eq: two layer relu}
    \Nc(\xb) = v_0 + \sum_{k=1}^\width v_k \sigma(\wb_k^\top\xb + b_k).
\end{align}  

\begin{figure*}[!htp]
    \centering
    \subfigure[]{\includegraphics[width=0.24\textwidth]{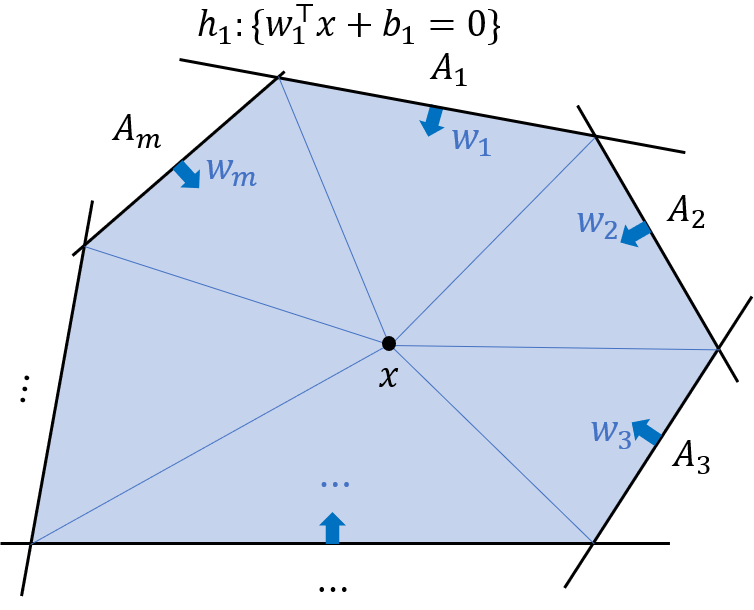}}
    \hfill
    \subfigure[]{\includegraphics[width=0.42\textwidth]{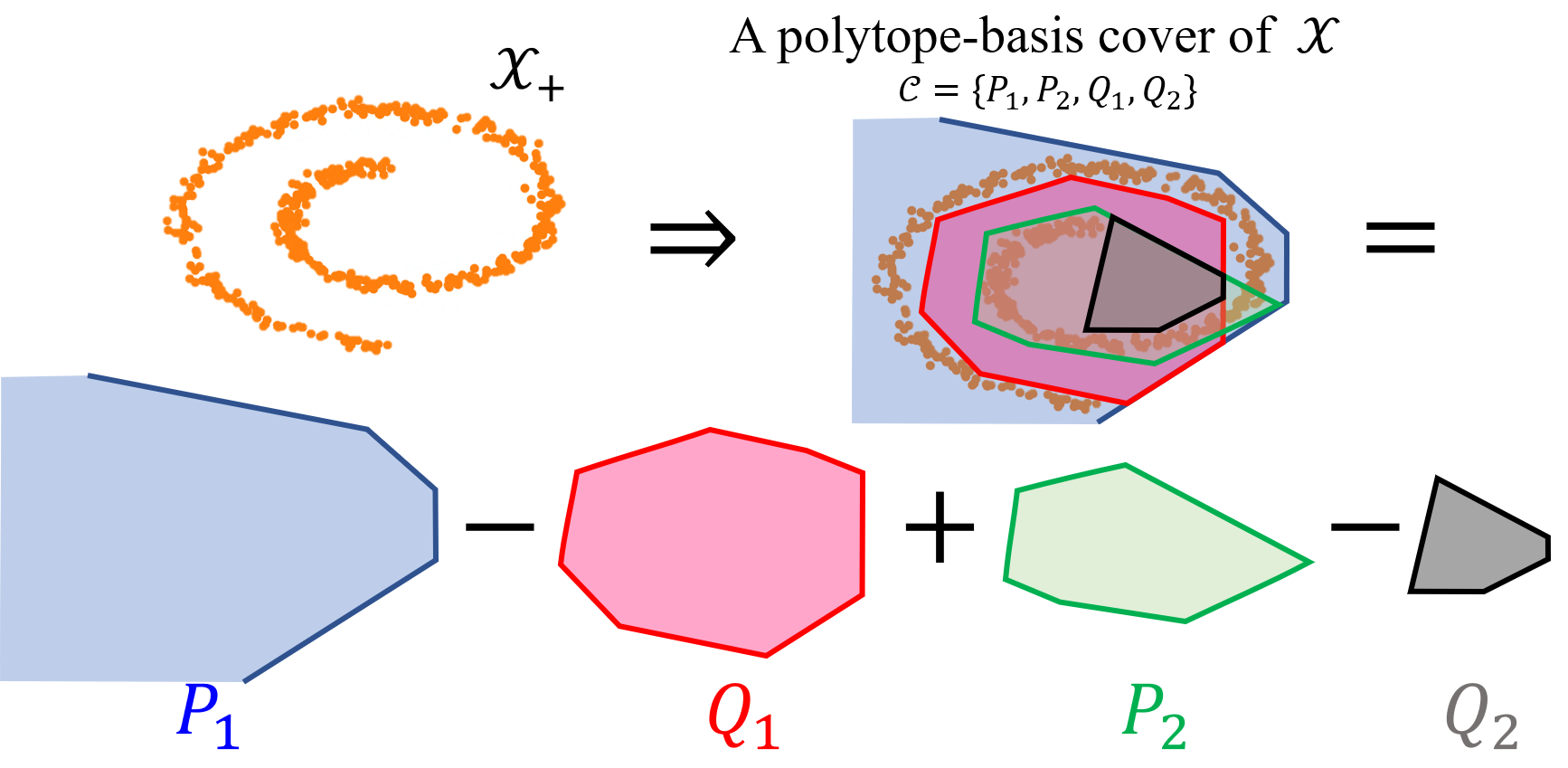}}
    \hfill
    \subfigure[]{\includegraphics[width=0.25\textwidth]{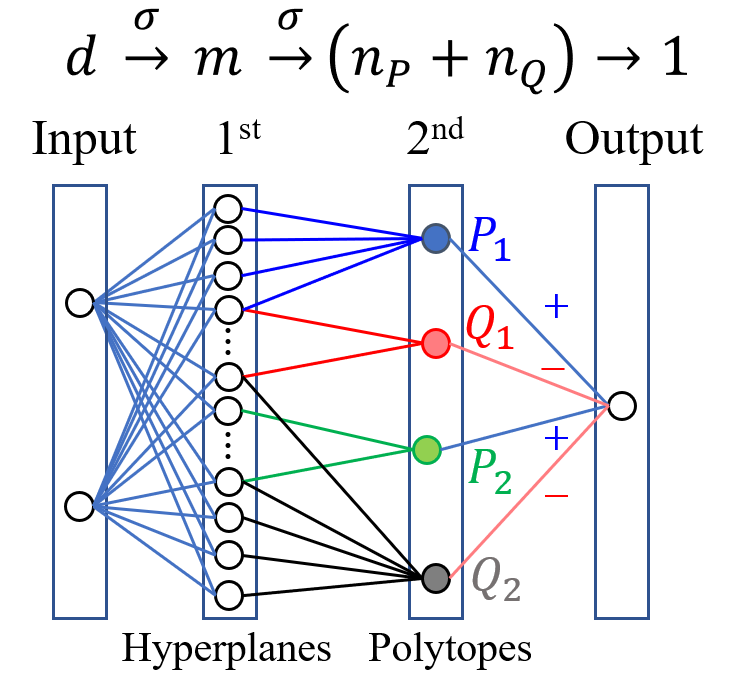}}
    \caption{ 
    The fundamental ideas in our work.
    (a) A convex polytope enclosed by $\width$ hyperplanes can be decomposed by $\width$ small pyramids.
    (b) For the topological space $\Xc_+$, the collection of polytopes $\Cc = \{P_1, P_2, Q_1, Q_2\}$ forms a polytope-basis cover of $\Xc$. 
    (c) 
    The constructive proof in Theorem \ref{thm: compact} further exhibits the role of neurons in hidden layers: the width of first layer means the number of total faces in $\Cc$, and the neurons in the second layer corresponds to the polytopes in $\Cc$.
    }
    \label{fig: convex polytope} 
\end{figure*}

\begin{definition}
    Let $\Xc:=\Xc_+\cup\Xc_- \subset \Rd^d$ be a union of two disjoint topological spaces. A neural network architecture $\Ac$ is called a \emph{\exact on $\Xc$} if there exists a neural network with the architecture $\Ac$ such that
    \begin{align*}
        \Nc(\xb) &> 0 \qquad \text{if} \qquad \xb \in \Xc_+, \\
        \Nc(\xb) &< 0 \qquad \text{if} \qquad \xb \in \Xc_-.
    \end{align*}
\end{definition}
\vspace{-3mm}
In other words, {\em \exact on $\Xc$} refers to a network architecture capable of fully discriminating between the two specified manifolds, $\Xc_+$ and $\Xc_-$. This paper aims to explore the connection between feasible architectures and the geometrical properties of the dataset.


\section{Main Contributions} \label{sec: main}

In this section, we present our main findings in two forms. Firstly, we establish the upper and lower limits of network width required for classifying a specific dataset. Secondly, we illustrate how trained neural networks inherently capture the geometric characteristics of the dataset they handle. 
 
\subsection{Data Geometry-Dependent Bounds on Widths}
Let $C\subset\Rd^d$ be a convex polytope with $\width$ faces. 
Our objective is to establish bounds on the widths of a ReLU neural network necessary for it to be \exact on $C$.
Applying piecewise linearity of ReLU networks and the volume formula of convex polytopes, the following proposition provides
the answer. 
\begin{proposition} \label{prop: convex polytope}
    Let $C \subset \Rd^d$ be a convex polytope enclosed by $\width$ hyperplanes, and consider $\Xc=\Xc_+\cup\Xc_-$ where $\Xc_+ := C, \Xc_- := \Bc_\varepsilon(\Xc_+)^c$. Then, $\ReLUtwo{d}{m}{1}$ is a \exact on $\Xc$ with minimal depth.
    Conversely, if $\ReLUThree{d}{d_1}{d_2}{\cdots}\ReLUtwo{}{d_k}{1}$ is a \exact on $\Xc$, then  
    $$ \hspace{-0.15cm}
    d_1 \cdot \prod_{j=2}^k (2d_j+1) \ge
    \begin{cases}
        \ceil{\frac{m}{2}} + (d-2), \qquad &\text{ if }\quad  m \geq 2d+1, \\
        2d-1, \qquad &\text{ if }\quad m=2d-1, 2d, \\
        d+1, \qquad &\text{ if }\quad m<2d-1.
    \end{cases}
    $$
    This lower bound is optimal when $k=1$ and $d=2$ (i.e., two-layer network on $\Rd^2$).
\end{proposition} 

\begin{proof}[Proof sketch.]
    We briefly introduce the main idea here.
    Let $A_1, \cdots, A_\width$ be faces of $C$, and $\xb$ be a point in $C$. Since $C$ is convex, it can be decomposed to $m$ pyramids whose common apex is $\xb$ (see Figure \ref{fig: convex polytope}(a)). Then, the volume (in Lebesgue sense) of $C$ is equal to the sum of the volume of $m$ pyramids. Mathematically, it is represented by
    \begin{align*}
        \text{Vol}_d(C) &= \frac1d \sum_{k=1}^\width \text{Vol}_{d-1}(A_k)\sigma(\wb_k^\top\xb+b_k)
    \end{align*}
    where $\wb_i$ is a unit vector of the hyperplane $A_i$, and $\text{Vol}_d$ denotes the $d$-dimensional volume. 
    From this equation, we define a two-layer ReLU network 
    $$
    \Nc(\xb):= 1 + M (\text{Vol}_d(C) - \frac1d \sum_{k=1}^\width \text{Vol}_{d-1}(A_k)\sigma(\wb_k^\top\xb+b_k))
    $$
    for some constant $M$. Then $\Nc(\xb) = 1$ for all $\xb\in C$, and we can prove that $\ReLUtwo{d}{\width}{1}$ is a feasible architecture on the polytope, by adjusting the value of $M$. 
%
    The detailed proof can be found in Appendix \ref{app: proof 1}. 
\end{proof}

In the proof of Proposition \ref{prop: convex polytope}, $\sigma(\Nc)$ is a two-layer ReLU network that approximates the indicator function on a convex polytope.\footnote{We also mention that the approximation of indicator functions directly induces UAP of neural networks (Theorem \ref{thm: regression}).}
Building upon this proposition, we are interested in extending our findings to arbitrary topological spaces, specifically those that can be distinguished by a collection of polytopes. To facilitate this extension, we introduce an additional terminology.

\begin{definition} \label{def: convex polytope cover}
    Let $\Xc := \Xc_+ \cup \Xc_- \subset \Rd^d$ be a union of two disjoint topological spaces. A finite collection of polytopes $\Cc:=\{P_1,\cdots, P_{n_P}, Q_1,\cdots, Q_{n_Q}\}$ is called a \emph{polytope-basis cover of $\Xc$} if it satisfies
    \vspace{-1mm}
    \begin{align*}
        \sum_{k=1}^{n_P} \indicator{\xb \in P_k} &\;>\; \sum_{k=1}^{n_Q} \indicator{\xb \in Q_k} \qquad \text{for all} \quad \xb \in \Xc_+, \\
        \sum_{k=1}^{n_P} \indicator{\xb \in P_k} &\;\le\; \sum_{k=1}^{n_Q} \indicator{\xb \in Q_k} \qquad \text{for all} \quad \xb \in \Xc_-.
    \end{align*}
\end{definition}
\vspace{-3mm}
Roughly speaking, a polytope-basis cover of $\Xc$ is a polytope covering of $\Xc_+$ and $\Xc_-$ that admits overlapping, where the difference number of overlapped covers is restricted to be positive or negative with respect to the label. Below, we provide an example of a polytope-basis cover for the swiss roll dataset described in Figure \ref{fig: swiss}(a).

\begin{example} 
Let $\Xc_+, \Xc_-$ be the orange and blue class in Figure \ref{fig: swiss}(a), respectively. Figure \ref{fig: convex polytope}(b) demonstrates a polytope-basis cover of $\Xc$ consists of four convex polytopes: $P_1, P_2, Q_1,Q_2$. It is easily checked that $\sum_{k=1}^2\indicator{\xb \in P_k} - \sum_{k=1}^2\indicator{\xb \in Q_k} =1 >0$ for $\forall \xb \in \Xc_+$, while $\sum_{k=1}^2\indicator{\xb \in P_k} - \sum_{k=1}^2\indicator{\xb \in Q_k} =0$ for $\forall \xb \in \Xc_-$. 
\end{example}

The usefulness of polytope-basis covers appears in the following theorem: we can derive an upper bound of \exact on $\Xc$ from its polytope-basis cover, by applying the constructive proof used in Proposition \ref{prop: convex polytope}. 

\begin{theorem} \label{thm: compact}
    For a given topological space $\Xc =\Xc_+\cup\Xc_-\subset \Rd^d$, let $\Cc = \{P_1,\cdots, P_{n_{P}}, Q_1,\cdots, Q_{n_Q}\}$ be a polytope-basis cover of $\Xc$. Let $\width$ denote the total number of faces of the convex polytopes in $\Cc$.
    Then, 
    \begin{align*}
        \ReLUthree{d}{\width}{(n_P+n_Q)}{1}
    \end{align*}
    \vspace{-2mm}
    is a \exact on $\Xc$.
\end{theorem}
The proof can be found in Appendix~\ref{app: proof compact}.
One of the important contributions of Theorem~\ref{thm: compact} is that its construction exhibits the exact role of each neuron in the hidden layers. It demonstrates that for a given polytope-baiss cover, a three-layer ReLU network with widths $\#$(hyperplanes) and $\#$(polytopes) in the first and second hidden layer, respectively, is a feasible architecture. 
For instance, we recall the polytope-basis cover represented in Figure \ref{fig: convex polytope}(b).
Each neuron in the first hidden layer represents a hyperplane in the input space, where each neuron in the second hidden layer represents a convex polytope ($P_i$ or $Q_j$) in $\Cc$ that is formed by connected neurons in the first layer as depicted in Figure \ref{fig: convex polytope}(c).

Building upon Theorem~\ref{thm: compact}, we can further explore the relationship between the topological properties of a dataset and the maximum width achievable by feasible network architectures. 
Specifically, we concentrate on simplicial complexes and Betti numbers, which are fundamental tools for investigating the topological structure of point cloud datasets in topological data analysis (TDA).

A simplicial $j$-complex is a specific type of simplicial complex where the highest-dimensional simplex has dimension $j$. Within a given simplicial complex $K$, a facet is a simplex with the highest dimension that is not a face (subset) of any larger simplex \citep{magai2022topology}.
With these definitions established, we proceed by proposing a feasible network architecture and deriving upper bounds on its width when one class $\Xc_+$ forms a simplicial complex.

\begin{theorem} \label{thm: simplicial complex}
    Let $\Xc=\Xc_+\cup\Xc_-$ be a union of two disjoint topological spaces, where $\Xc_+\subset \Rd^d$ is a simplicial $J$-complex consists of $k$ facets. Let $k_j$ be the number of $j$-dimensional facets of $\Xc_+$ for $j=1,\cdots,J$. Then, $\ReLUthree{d}{d_1}{k}{1}$ is a \exact on $\Xc$, where $d_1$ is bounded by
    \vspace{-2mm}
    \begin{align} \label{eq: simplicial complex}
        d_1 & \le
        \min\Bigg\{ k(d+1) - (d-1)\bigg\lfloor{ \sum_{j=0}^{\floor{\frac{d-1}{2}}} \frac{k_j}{2}}\bigg\rfloor, \nonumber \\
        & 
        (d+1)\! \bigg\lfloor\sum_{j\le \frac{d}{2}} \left(  k_j \frac{j+2}{d-j} + \frac{j+2}{j+1} \right) + \sum_{j>\frac{d}{2}} k_j \bigg\rfloor \Bigg\}. 
    \end{align}
\end{theorem}
\vspace{-2mm}
The proof can be found in Appendix \ref{app: proof 2}.
Theorem \ref{thm: simplicial complex} reveals that the width $d_1$ is bounded by in terms of the number of facets $k$ of the provided simplicial complex. 
From a geometric perspective, it is generally intuitive that a smaller number of facets suggests a simpler structure of the simplicial complex. 
This notion is mathematically expressed in \eqref{eq: simplicial complex}, which suggests that the first value in \eqref{eq: simplicial complex} results in $d_1 \lesssim \frac{k}{2}(d+3)$, which magnifies as $k$ increases.
Similarly, when the maximal dimension $J$ is smaller than $\frac{d}{2}$ and $k$ is fixed, the summation in the second term in \eqref{eq: simplicial complex} reduces to $d_1 \lesssim (d+1) \left( k\frac{J+2}{d-J} + 2 \right)$, which rapidly diminishes as $J$ decreases.
This analysis demonstrates that a smaller dimension $J$ demands smaller widths, which aligns with the intuition that \emph{the lower-dimensional manifold could be approximated with the smaller number of neurons.}

Now, we demonstrate how the result in Theorem \ref{thm: compact} can be further leveraged to ascertain a neural network architecture with width bounds defined in terms of the Betti numbers. 
The Betti number is a key metric used in TDA to denote the number of $k$-dimensional `holes' in a data distribution, which are frequently employed to study the topological characteristics of topological spaces \citep{naitzat2020topology}.

Recall that Theorem \ref{thm: compact} offers an upper bound on widths when $\Xc$ can be depicted as a difference between two groups of convex sets. Expanding on this, assuming the polytope-basis cover consists of prismatic polytopes\footnote{For the formal definition of prismatic polytopes, see Appendix \ref{app: proof 3}}, we can derive a bound for network architecture in relation to its Betti numbers.
The result is concretely explained in the following theorem.

\begin{theorem} \label{thm: betti numbers}
    Let $\Xc=\Xc_+\cup\Xc_-$ be a union of two disjoint topological spaces, where $\Xc_-$ can be separated from $\Xc_+$ by disjoint bounded prismatic polytopes having at most $\width$ faces. 
    Let $\beta_k$ be the $k$-th Betti number of the polytope-basis cover. Then, the following three-layer architecture
    \begin{align} \label{eq: betti numbers}
        d & \stackrel{\sigma}{\rightarrow} \left( \width + 2(\beta_0-1) + \sum_{k=1}^{d} \left(\width-2(d-k-1)\right)\beta_k \right) \nonumber \\
        & \ReLUtwo{}{\left(\sum_{k=0}^{d} \beta_k \right)}{1}
    \end{align}
    is a \exact on $\Xc$. Conversely, for any such $\Xc$, suppose $\ReLUFour{d}{d_1}{d_2}{\cdots}{d_D}{\rightarrow}{1}$ is a \exact on $\Xc$. Then, the network widths must satisfy
    \begin{align} \label{eq: betti lower bound}
         \sum_{i=1}^D \prod_{j=i}^D d_j \ge 2\sum_{k=0}^{d} \beta_k - 2.
    \end{align}
\end{theorem} 

The proof is provided in Appendix \ref{app: proof 3}.
One of the important implications of Theorem \ref{thm: betti numbers}  is the upper and lower bounds on network widths in terms of the Betti numbers of $\Xc$, which reveals the interplay between the topological characteristics of the dataset and network architectures.
In Appendix,
we also show in Proposition \ref{prop: no topology} that topological property alone cannot determine the feasible architecture,
highlighting the significance of prismatic polytopes assumption in Theorem \ref{thm: betti numbers}.
In other words, the result in Proposition \ref{prop: no topology} implies that the geometrical assumptions in this theorem are indispensable to connecting topological features with bounds on the network widths.

Interestingly, the sum of Betti numbers $\sum_{k=0}^d \beta_k$ which appears in the third layer in \eqref{eq: betti numbers}, is often called the \emph{topological complexity} of $\Xc$. This quantity is recognized as a measure of the complexity of a given topological space in some previous works \citep{bianchini2014complexity, naitzat2020topology}, and can be bounded by Morse theory \citep{milnor1963morse} or Gromov's Betti number Theorem \citep{gromov1981curvature}. 

Furthermore, the lower bound on widths \eqref{eq: betti lower bound} shows that the sum of product of widths should be greater than the sum of Betti numbers. 
This finding also confirms the increased significance of widths in deeper layers as compared to earlier ones, highlighting the advantageous impact of depth in network architecture.
It also verifies that the contribution of the width in deeper layers holds greater significance compared to previous layers, i.e., the positive effect of depth.

\subsection{Polytope-Basis Cover Search Algorithm} \label{sec: algorithms}

\begin{figure*}[ht]
    \centering
    \includegraphics[width=0.98\textwidth]{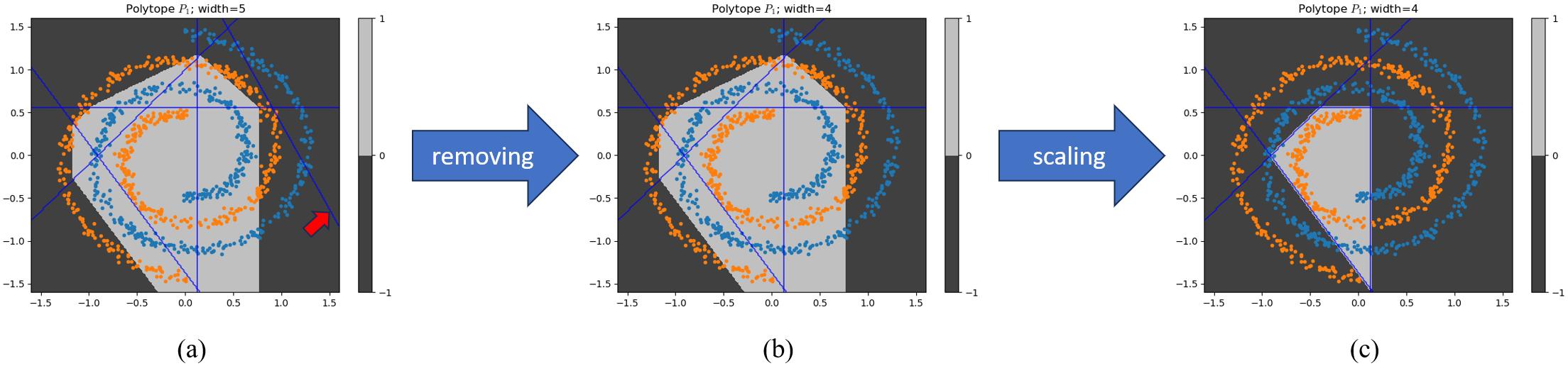}
    \caption{
    Visualization of Algorithm \ref{alg: compressing}. For a given two-layer network $\Tc$ defined by \eqref{eq: two-layer constant bias}, it strategically removes and scales specific neurons of $\Tc$ to encapsulate the characteristics of a single convex polytope. In essence, the algorithm compresses the network to reveal the minimal representation of a polytope structure.}
    \label{fig: compressing}
\end{figure*}

\begin{figure*}[ht]
    \centering
    \subfigure[]{\includegraphics[width=0.14\textwidth]{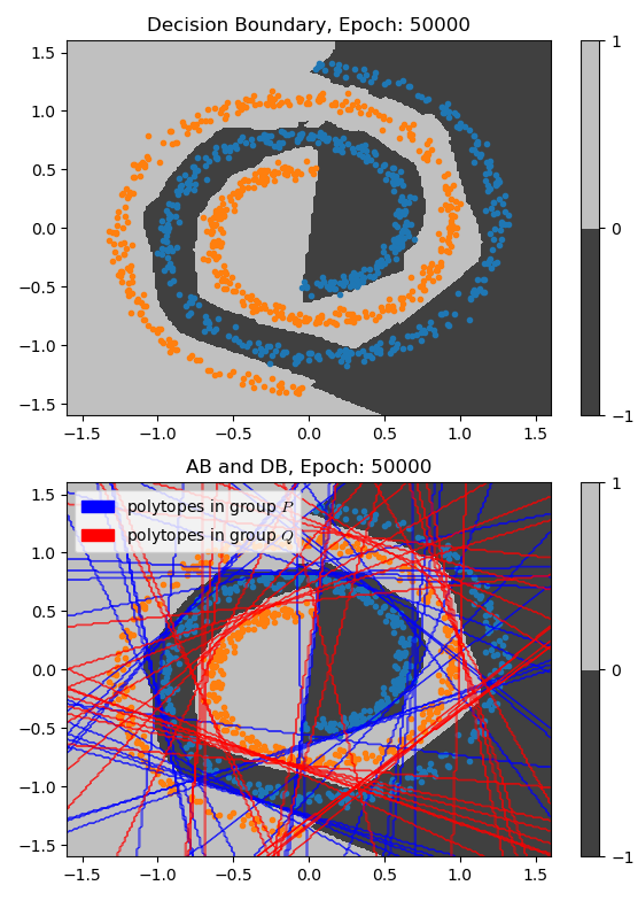}}
    \hfill
    \subfigure[]{\includegraphics[width=0.42\textwidth]{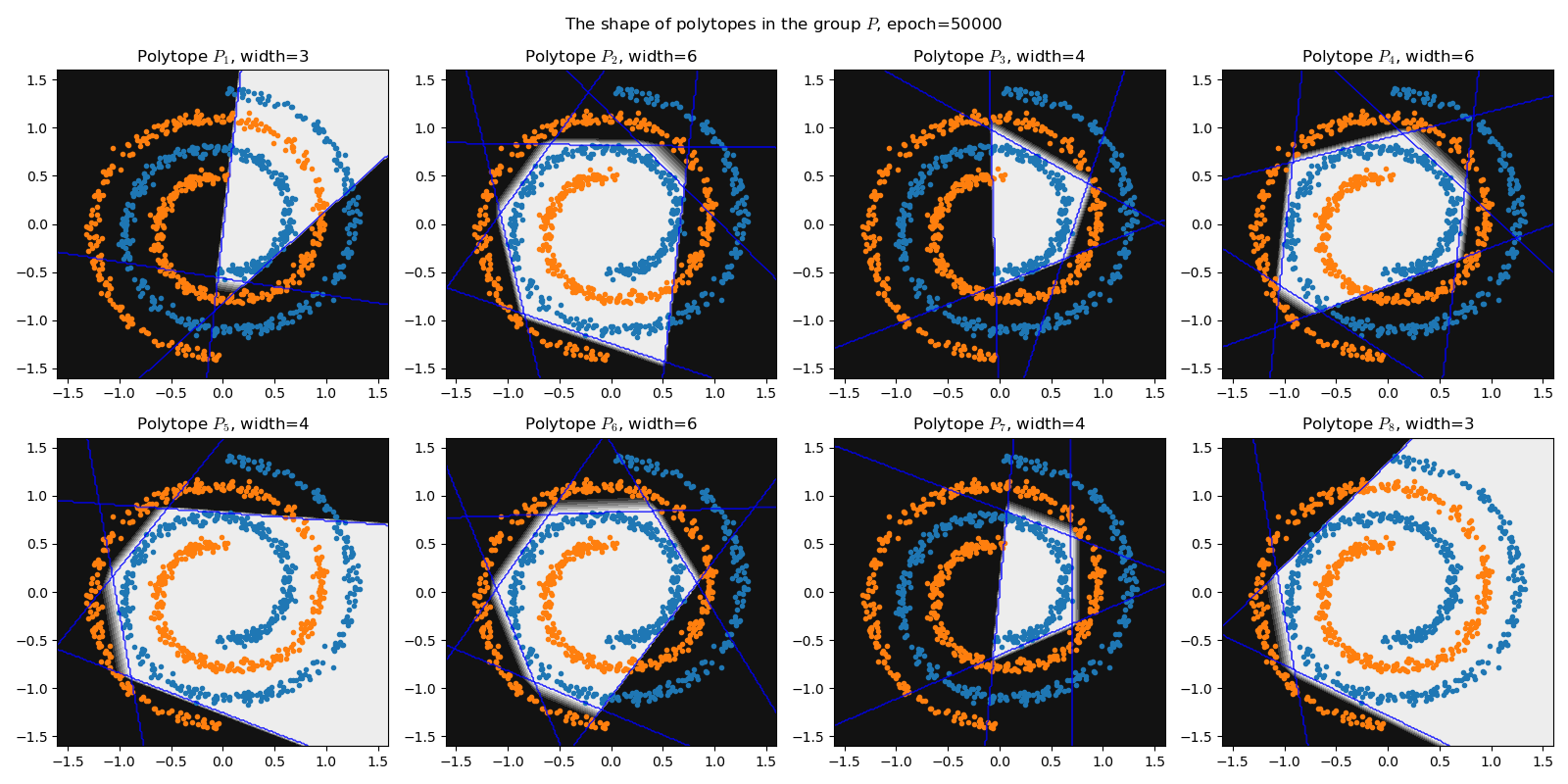}}
    \hfill
    \subfigure[]{\includegraphics[width=0.42\textwidth]{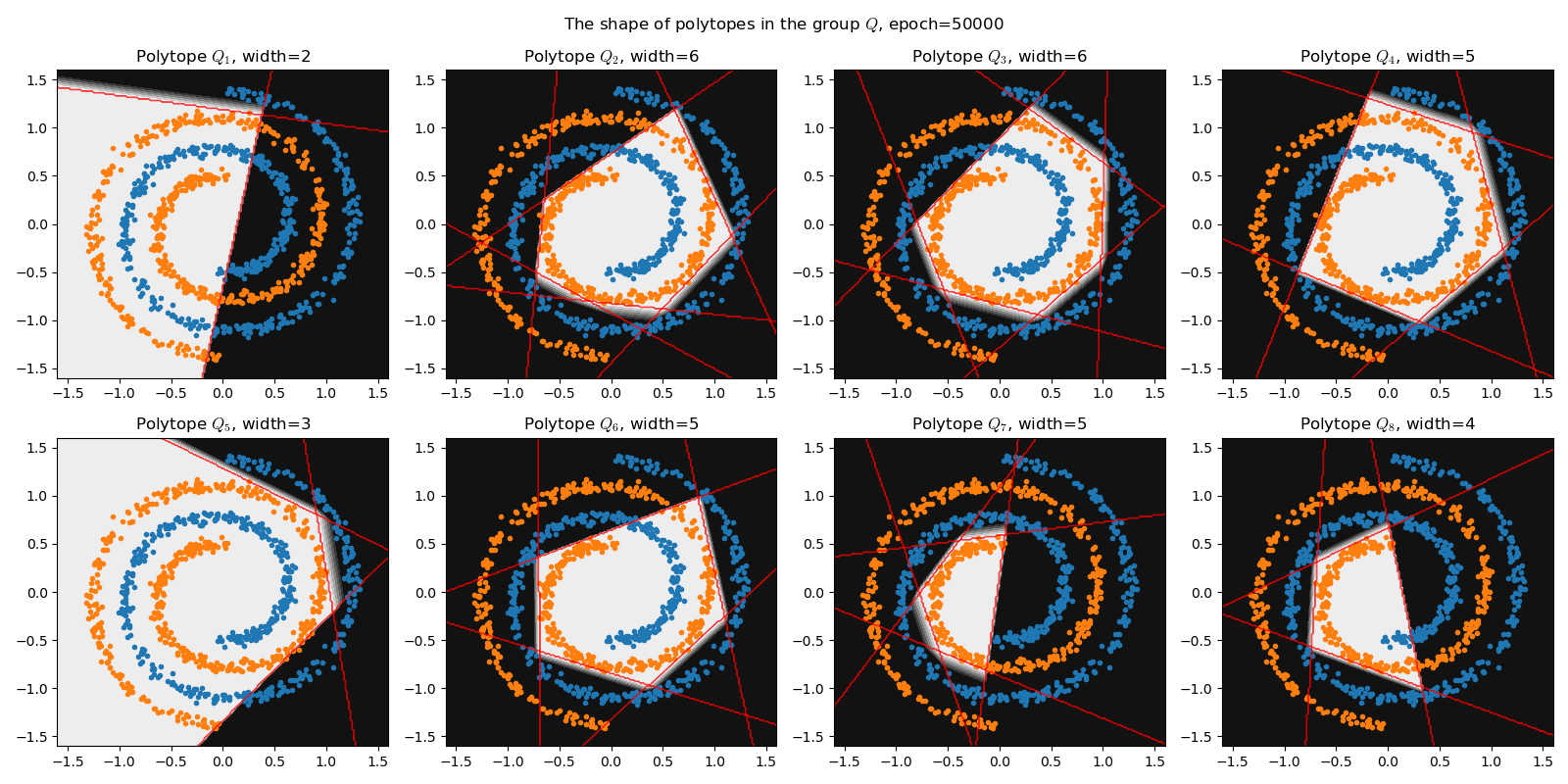}}
    \caption{
    A polytope-basis cover derived from a trained three-layer ReLU network defined in \eqref{eq: three-layer}, obtained from Algorithm \ref{alg: compressing}. The decision boundary and activation boundaries\footnotemark$\;$ of the trained network are depicted in (a). Each polytope corresponding to $a_j=+1$ and $a_j=-1$ is illustrated in (b) and (c), respectively. The result constitutes a polytope-basis cover of the swiss roll dataset.
    }
    \label{fig: three-layer}
\end{figure*}

So far, we have demonstrated how feasible architecture can be determined from the geometric characteristics of a topological space $\Xc$, in terms of its polytope-basis cover. 
In this section, we delve into the converse scenario: given a trained neural network on the dataset $\Dc=\{(\xb_i,y_i)\}$, can we obtain a polytope-basis cover of $\Dc$? 
We tackle this question by leveraging the convexity of two-layer ReLU networks, which has been studied in a few previous works \cite{amos2017input, sivaprasad2021curious, balestriero2022deephull} (see Appendix \ref{sec: related} for related works). Our focus also extends to the precise computation of the number of faces. 

\newcommand{\lambdabias}{\lambda}
\begin{theorem} \label{thm: three-layer polytope cover}
    Let $\Tc_j$ and $\Nc$ be two-layer and three-layer ReLU networks defined by 
    \begin{align} 
        \Tc_j(\xb) &:= \lambdabias + \sum_{k=1}^{\width_j} v_{jk} \sigma(\wb_{jk}^\top \xb + b_{jk}), \quad \forall v_{jk}<0, \label{eq: two-layer constant bias} \\
        \Nc(\xb) &:= -\frac12 \lambdabias + \sum_{j=1}^J a_j\sigma(\Tc_j(\xb)), \quad \forall a_j \in \{\pm 1\} \label{eq: three-layer}
    \end{align}
    for a positive constant $\lambdabias$.
    For a given dataset $\Dc=\{(\xb_i, y_i)\}$, suppose $\Nc$ satisfies 
    \begin{align} \label{eq: three-layer condition}
        \sigma(\Tc_j(\xb_i)) = 0 \;\textrm{ or }\; \lambdabias, \quad \forall \xb_i\in \Dc, \;\forall j\in[J].
    \end{align}
    Then, the collection of polytopes $\{C_j\}_{j\in[J]}$, defined by $C_j:=\{\xb\in\Rd^d~|~ \Tc_j(\xb)=\lambdabias\}$, becomes a polytope-basis cover of $\Dc$ whose accuracy is same with $\Nc$.
\end{theorem}


\begin{algorithm}[t]
\caption{Compressing algorithm} \label{alg: compressing}
\begin{algorithmic}
\STATE {\bf Input:} a pretrained two-layer network $\Tc$ defined in \eqref{eq: two-layer constant bias}, training dataset $\Dc=\{(x_i, y_i)\}_{i=1}^n$, $\lambda_{scale}>1$

\STATE $m \leftarrow \text{ the width of }\Tc$
\STATE $K \leftarrow \emptyset$
\FOR{$k, l\in [m]$} 
\IF{$\wb_l^\top\xb_i+b_l>0$ implies $\wb_k^\top\xb_i+b_k>0$ for all $i\in[n]$}
    \STATE $K\leftarrow K \cup \{k\}$
\ENDIF
\ENDFOR
\IF{$K\neq\emptyset$}
    \STATE $k \leftarrow \argmin_{k\in K} |v_k| \cdot \norm{\wb_k}$ 
    \STATE remove the $k$-th neuron $(v_k, \wb_k, b_k)$ from $\Tc$
    \STATE $m \leftarrow m-1$
\ENDIF
\FOR{$k\in [m]$}
\IF{$\wb_k^\top\xb_i+b_k>0$ and $0<\Tc(\xb_i)<1$ for some $i\in[n]$}
\STATE $(v_k, \wb_k, b_k) \leftarrow \lambda_{scale} \times (v_k, \wb_k, b_k)$
\ENDIF
\ENDFOR
\STATE \textbf{Output:} $\Tc$ 
\end{algorithmic}
\end{algorithm}

The proof of this theorem can be found in Appendix \ref{app: proof 5}.
The constant $\lambdabias$ in \eqref{eq: two-layer constant bias} and \eqref{eq: three-layer condition} is a positive scalar value determined from the ratio of labels in the dataset. In experiments, we practically adopted $\lambdabias=5$. See Appendix \ref{app: real world algorithm} for further details.

This theorem establishes that if a trained three-layer network $\Nc$ defined in \eqref{eq: three-layer} satisfies the condition outlined in \eqref{eq: three-layer condition}, then we can derive a polytope-basis cover of the training dataset $\Dc$ from $\Nc$.
Below, we outline two strategies we employed to satisfy both \eqref{eq: two-layer constant bias} and \eqref{eq: three-layer} conditions.

Firstly, to meet \eqref{eq: two-layer constant bias}, we utilize the implicit bias of gradient descent established by \citet[Theorem 2.1]{du2018algorithmic}, stated in Proposition \ref{prop: balanced property}.
Specifically, we initialize the network weights to satisfy
\begin{align} \label{eq: initialization}
    v_{jk} < - \sqrt{ \norm{\wb_{jk}}^2 + b_{jk}^2} \quad \forall j \in[J], k \in [m].
\end{align} 
Then, the implicit bias preserves the inequality \eqref{eq: initialization}, thus ensures $v_{jk}<0$ for all $j\in[J]$ and $k\in[m]$ on the gradient flow. This satisfies the first condition.

Secondly, to achieve \eqref{eq: three-layer condition}, we introduce a novel approach named the ``compressing algorithm." This algorithm aims to `compress' a given two-layer network $\Tc_j$ defined in \eqref{eq: two-layer constant bias} to identify a minimal convex polytope representation. The process is precisely outlined in Algorithm \ref{alg: compressing}.

More precisely, the algorithm operates in two main steps: 1) identifying and removing a neuron that do not significantly influence the decision boundary, and 2) amplifying the weights to delineate the faces of the decision boundary polytope. We illustrate the functionality of the algorithm in Figure \ref{fig: compressing}. 
In Figure \ref{fig: compressing}(a), the red arrow highlights a neuron identified as non-essential for the decision boundary. This neuron is subsequently removed as shown in (b). Subsequently, by scaling the weights $(v_k, \wb_k, b_k)$ by a factor $\lambda_{scale}>1$, the decision boundary shrinks into a convex polytope with a number of faces equal to the width of $\Tc$ (depicted in Figure \ref{fig: compressing}(c)). 

\begin{figure}[t]
    \centering
    \includegraphics[width=0.98\columnwidth]{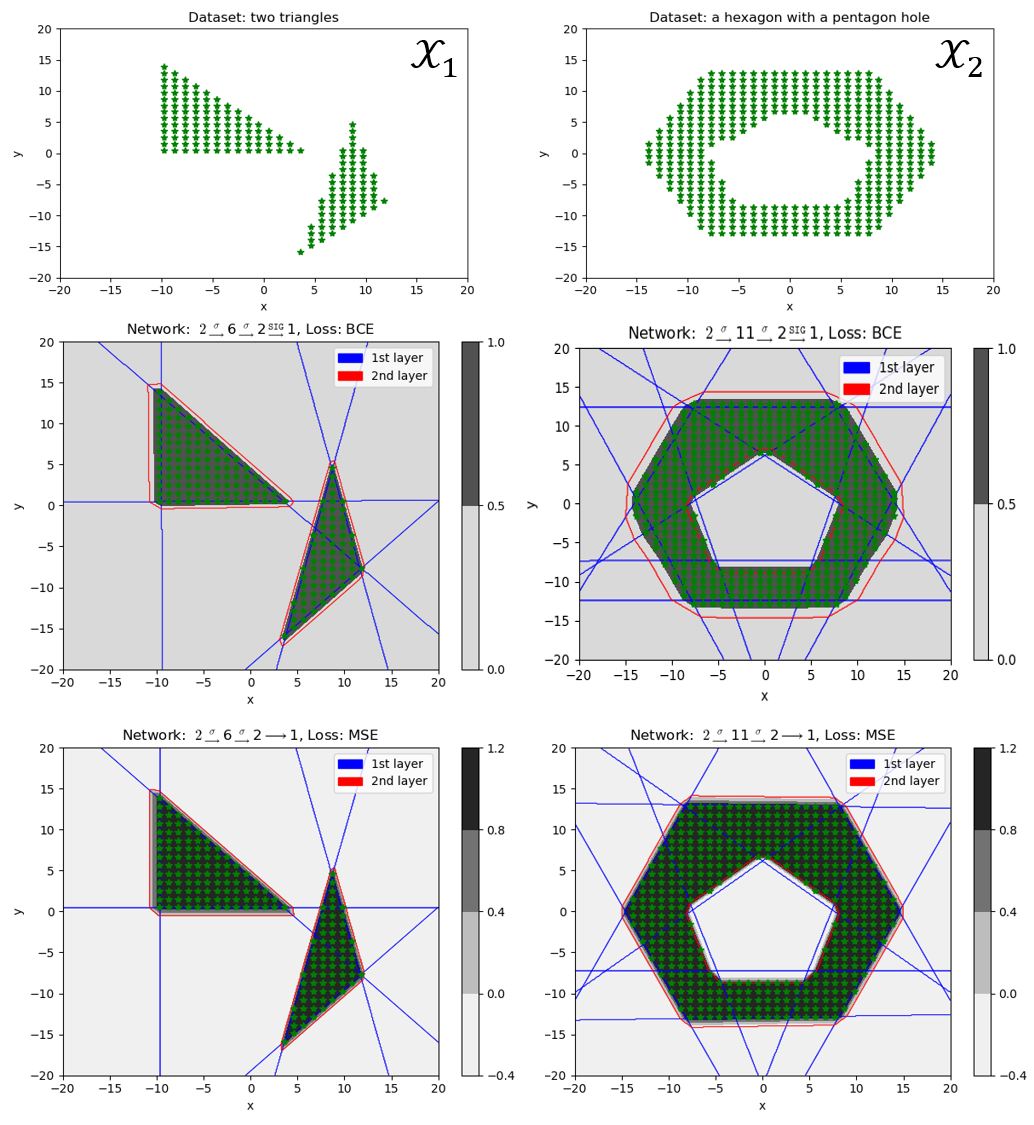}
    \caption{Experimental verification of convergence of gradient descent.
    Two columns exhibit the shape of two topological spaces, which are `two triangles' and `a hexagon with a pentagon hole.'
    The second and third row show the converged networks by gradient descent under the BCE loss and the MSE loss, respectively. 
    The first layer (blue) and second layer (red) represent the hyperplane and polytopes, respectively, described in Section \ref{sec: main}.
    \vspace{-5mm}
    }
    \label{fig: experiments}
\end{figure}

Note that a single execution of Algorithm \ref{alg: compressing} may not immediately yield the network satisfying \eqref{eq: three-layer condition}. However, we prove in Proposition \ref{prop: algorithms} that by iterating this algorithm a sufficient number of times, the output of the algorithm always satisfies both \eqref{eq: two-layer constant bias} and \eqref{eq: three-layer condition}, making it suitable for the application of Theorem \ref{thm: three-layer polytope cover}. 
Therefore, we apply the algorithm once every thousand iterations during the gradient descent optimization process, and the end of whole training. It is precisely described in Algorithm \ref{alg: three-layer} in Appendix \ref{app: three-layer}. We additionally mention that Algorithm~\ref{alg: compressing} is compatible with non-pretrained networks, although this flexibility may come at the cost of increased training time.

Consequently, we present a polytope-basis cover of the swiss roll dataset derived from a trained three-layer network in Figure \ref{fig: three-layer}. The polytopes comprising the resulting cover are visualized in Figure \ref{fig: three-layer} (b) and (c). This outcome demonstrates the polytope-basis cover inherent in the trained network can be identified through Algorithm \ref{alg: compressing}. It is worth noting that the decision boundary and activation boundary$^{\ref{DB AB}}\;$ displayed in Figure \ref{fig: three-layer}(a) is combination of several polytopes.

Lastly, we mention that we further propose two alternative methods for obtaining a polytope-basis cover in Appendix \ref{app: algorithms}. Specifically, one approach involves deriving such a cover from any trained two-layer network (Algorithm \ref{alg: two-layer}), while the other method entails training several neural networks in order (Algorithm \ref{alg: polytopes in order}). 
Further details and comparisons of these additional algorithms are provided in Appendix \ref{app: algorithms}.

\footnotetext{\label{DB AB} For the formal definition of decision boundary (DB) and activation boundaries (ABs), see Definition \ref{def: activation} in Appendix \ref{app: algorithms}.}

\begin{figure*}[t]
    \centering
    \subfigure[]{\includegraphics[width=0.48\textwidth]{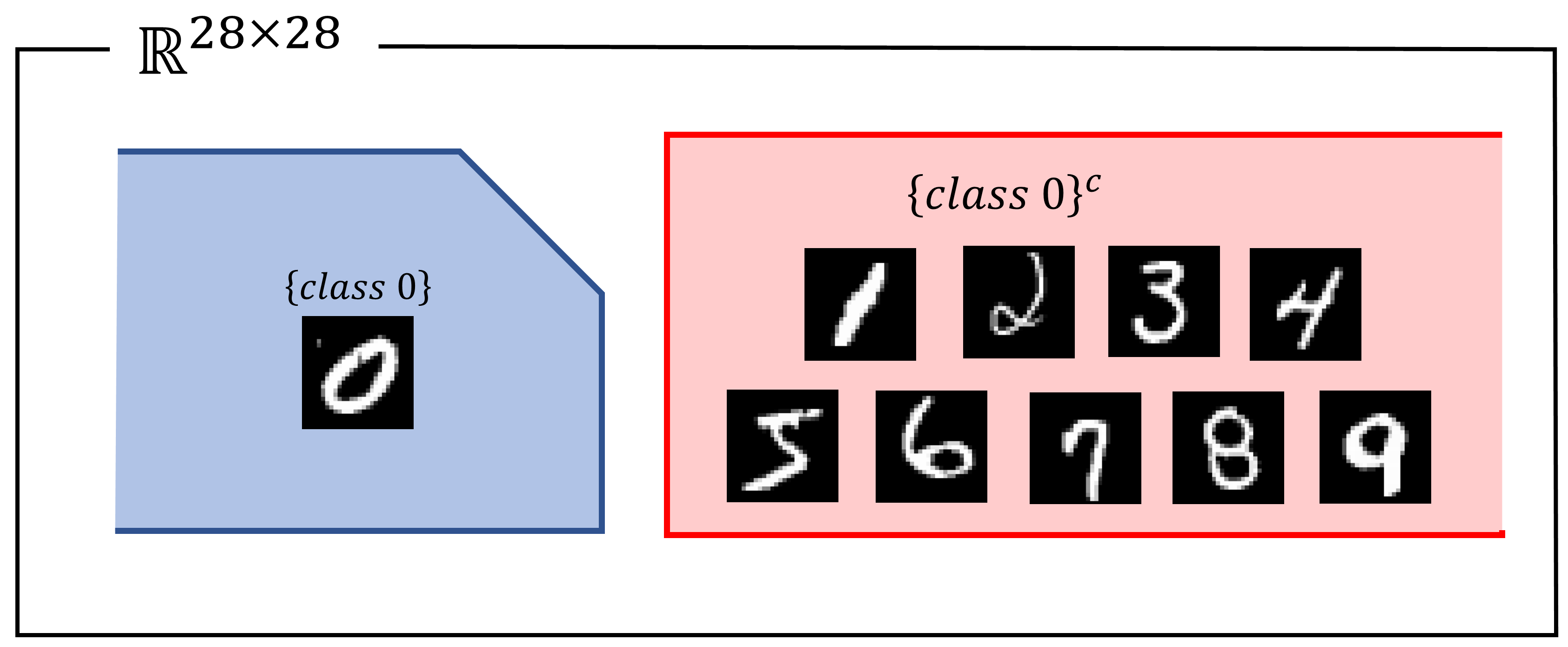}}
    \hfill
    \subfigure[]{\includegraphics[width=0.48\textwidth]{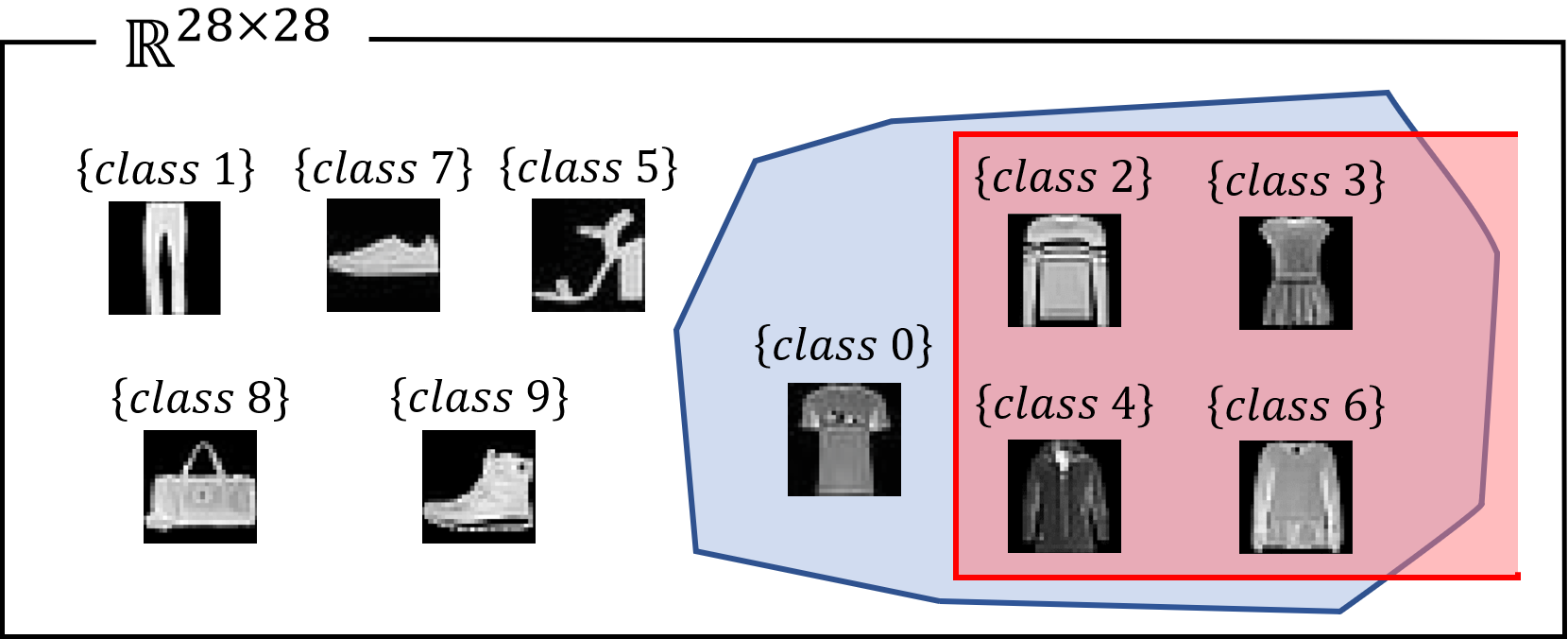}}
    \caption{
    Illustration of a polytope-basis cover of the real datasets.
    (a) The class $\{0\}$ of MNIST can be separated by a single convex polytope with four faces, while the complement class $\{0\}^c$ can be separated with three faces. 
    (b) The class $\{0\}$ of Fashion-MNIST can be separated by the difference of two polytopes, one of which contains similar images.
    Other classes also have simple polytope-basis covers as described in Table \ref{tab: result}.
    }
    \label{fig: real dataset}
\end{figure*}

\section{Experimental Results} \label{sec: experiments}
In Section \ref{sec: main}, we studied the relationship between the dataset geometry and neural network architectures. In this section, we provide two empirical results: 1) gradient descent indeed converges to the networks we unveil, and 2) we can investigate the geometric features of high-dimensional real-world datasets through our proposed algorithm. 

\subsection{Convergence on Polytope-Basis Covers}
We begin by demonstrating that gradient descent indeed converges to the networks we proposed in the preceding section, without additional regularization terms. We consider two illustrative topological spaces, $\Xc_1$ and $\Xc_2$, as depicted in Figure \ref{fig: experiments}. $\Xc_1$ represents a simplicial $2$-complex in $\Rd^2$, comprising two triangles, while $\Xc_2$ is a hexagon with a pentagonal hole, possessing a straightforward polytope-basis cover. The objective is to classify points within these spaces in $\Rd^2$ against others.
We evaluate the performance for two loss functions, which are the mean squared error (MSE) loss and the binary cross entropy (BCE) loss functions. For the BCE loss, we applied $\texttt{SIG}$ on the last layer.

For the first dataset $\Xc_1$, Theorem \ref{thm: simplicial complex} suggests that $\ReLUthree{2}{6}{2}{1}$ is a \exact on $\Xc_1$. 
To facilitate a clearer examination of weight vectors in each layer, we plot the activation boundaries\textsuperscript{\ref{DB AB}} in blue (the 1st hidden layer) and red (the 2nd hidden layer) colors, where the gray color denotes the decision boundary of the converged network. 
The weight vectors in the first layer accurately enclose the two triangles, reflecting the geometrical shape of $\Xc_1$.
Similarly, for the second dataset $\Xc_2$, Theorem \ref{thm: compact} suggests that $\ReLUthree{2}{11}{2}{1}$ is a \exact. Specifically, the eleven neurons in the first layer correspond to the eleven hyperplanes delineating the boundaries of the outer hexagon and the inner pentagon, while two neurons in the second hidden layer align with the two polygons. These outcomes precisely correspond to our network constructions depicted in Figure \ref{fig: convex polytope}(b, c). 


We conclude this section by providing theoretical insights into the convergence behavior of gradient descent. In Appendix \ref{sec: convergence}, utilizing our explicit construction of neural networks, we construct an explicit path that loss strictly decreases to zero (the global minima), when the network is initialized close to the target polytope (see Theorem \ref{thm: convergence}). The specific conditions governing the initialization region are described in terms of the distribution of the dataset along the convex polytope.
Although this result does not mean that gradient descent must converge to the global minimum but only guarantees the existence of such a path, it is strong evidence for the convergence to the global minima. 
For a thorough understanding and precise statements, we refer readers to Appendix \ref{sec: convergence}. 

\subsection{Polytope-Basis Cover for Real Datasets}

\begin{figure*}[t]
    \centering
    \subfigure[]{\includegraphics[width=0.48\textwidth]{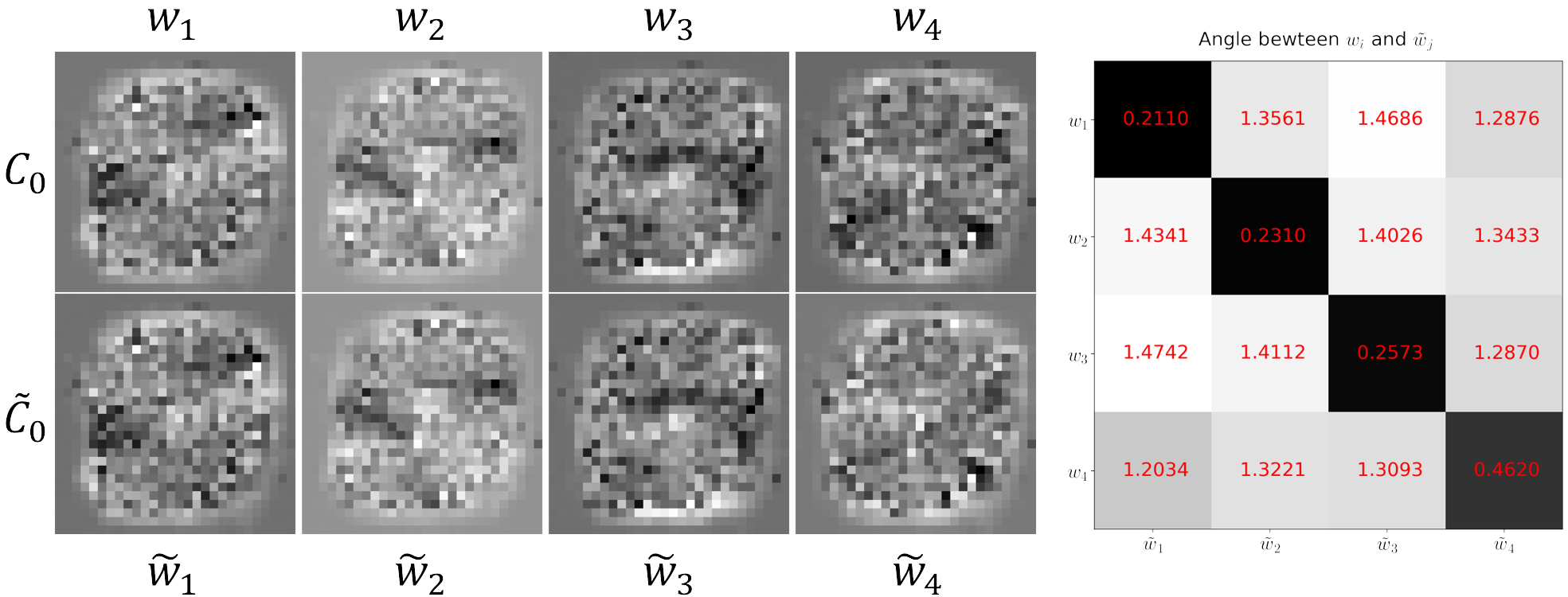}} \hfill
    \subfigure[]{\includegraphics[width=0.48\textwidth]{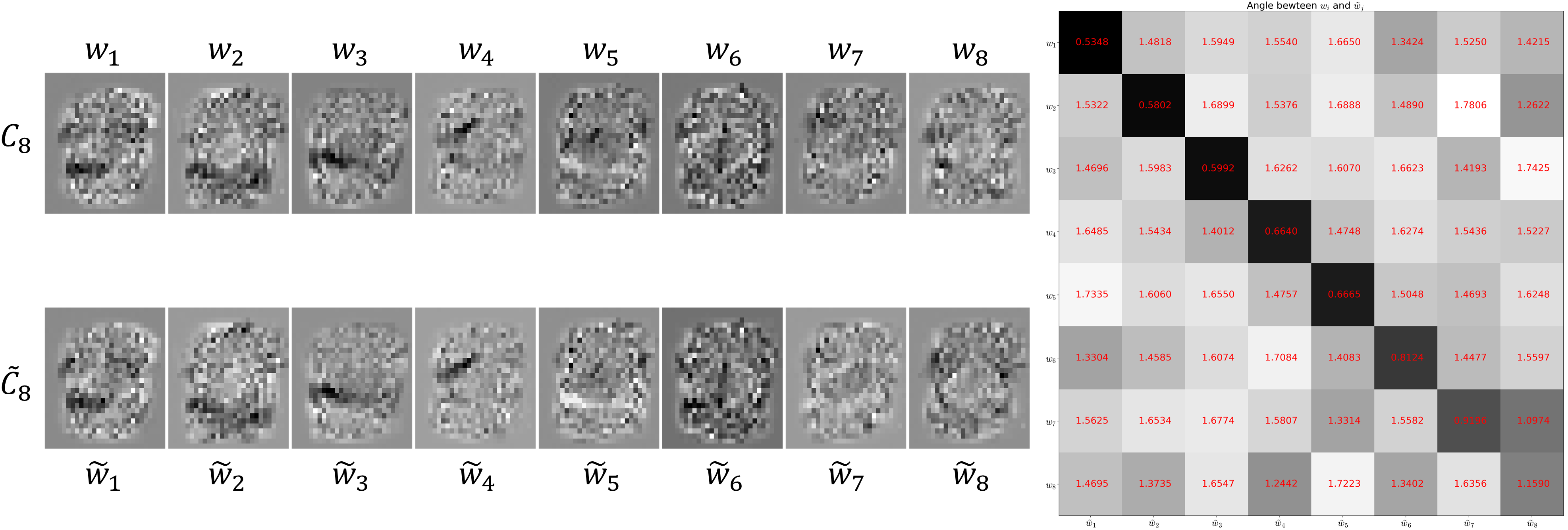}} 
    \caption{
    Visualization of two polytope covers for MNIST classes $\{0\}$ and $\{8\}$.
    (a) The four faces of two distinct polytope covers for class $\{0\}$ in the MNIST dataset are depicted. 
    (b) The eight faces of two distinct polytope covers for class $\{8\}$ in the MNIST dataset are depicted. 
    The distribution of angles between the vectors in two polytope covers are plotted on the right panel. 
    }
    \label{fig: face visualization}
\end{figure*}

Here, we delve into the polytope-basis cover analysis of three real-world datasets: MNIST, Fashion-MNIST, and CIFAR10.
We focus on binary classification tasks, specifically distinguishing one class from all other classes to obtain a polytope-basis cover of the class. For every class, we also consider its complement, denoted as $\{class\}^c$. We employed Algorithm \ref{alg: three-layer} to get the minimal polytope achieving $99.99\%$ accuracy on the union of the training and test sets.

Our empirical results are presented in Table \ref{tab: result}. Each column in the table corresponds to a class in the dataset, where each row presents the type of the class. The values in the table denotes the number of polytopes and their faces (we use notation $a$+$b$ to denote two polytopes with $a$ and $b$ faces, respectively).
For instance, the value in the first row and the first column shows that the $\{0\}$ class images in MNIST dataset can be distinguished from other classes by a single convex polytope with four faces. On the other hand, the second row in the first column shows that the complement of the class, namely $\{0\}^c := \{1,2,3,4,5,6,7,8,9\}$, can be separated from $\{0\}$ by a convex polytope with three faces, as illustrated in Figure \ref{fig: real dataset}(a). 

\begin{table}[t]
\centering
\resizebox{\columnwidth}{!}{ 
\begin{tabular}{cccccccccccc} 
\multicolumn{2}{c||}{}
&\multicolumn{10}{c}{Class} \\ 
\multicolumn{2}{c||}{Datasets} 
& 0& 1& 2& 3& 4& 5& 6& 7& 8& 9 \\ \hlineB{3}
\multirow{2}{*}{MNIST}  & \multicolumn{1}{c||}{\{class\} $\hphantom{^c} \vphantom{\Big|}$} &
4&4&7&8&5&7&4&8&8&7 \\ \cline{2-12}  
& \multicolumn{1}{c||}{\{class\}$^c\vphantom{\Big|}$} &
3&3&4&5&4&5&4&4&9&9\\ \hlineB{3}
\multirow{2}{*}{\begin{tabular}[c]{@{}c@{}}Fashion-\\ MNIST\end{tabular}} & \multicolumn{1}{c||}{\{class\}$\hphantom{^c}\vphantom{\Big|}$} &
9+3&4&9+5&9+3&9+6&8&9+7&9+1&6&10 \\ \cline{2-12} 
& \multicolumn{1}{c||}{\{class\}$^c\vphantom{\Big|}$} & 
16&3&22&11&20&4&28&6&4&5 \\ \hlineB{3}
\multirow{2}{*}{CIFAR10} & \multicolumn{1}{c||}{\{class\}$\hphantom{^c}\vphantom{\Big|}$} &
29+3&19&23+3&24+4&19&16+3&21&21&18&21\\ \cline{2-12} 
& \multicolumn{1}{c||}{\{class\}$^c\vphantom{\Big|}$}&
29&7&27+3&26&26+4&17&13&10&20+4&8
\\ \hlineB{3}
\end{tabular}
}
\caption{
Polytope-basis covers of each class in MNIST, Fashion-MNIST, and CIFAR10 datasets. Here, $a$+$b$ denotes two polytopes with $a$ and $b$ faces. 
For certain classes that cannot be covered by a single convex polytope with less than 30 faces, a second polytope is additionally computed. 
Overall, each class of real-world datasets can be covered by at most two polytopes, indicating the geometric simplicity of real datasets.
\vspace{-2mm}
}
\label{tab: result}
\end{table}

For Fashion-MNIST and CIFAR10 datasets, certain classes that cannot be covered by a convex polytope with less than 30 faces are covered by two polytopes. Figure \ref{fig: real dataset}(b) visually illustrates the classification of the class $\{0\}$ in Fashion-MNIST, accomplished through the difference of two polytopes.
In the case of CIFAR10 dataset, the number of faces in the polytopes tends to be higher compared to other datasets, consistent with the expectation that CIFAR10 exhibits a more intricate geometric structure than MNIST or Fashion-MNIST. 

Furthermore, we identify the geometric complexity of each class from Table \ref{tab: result}. In Fashion-MNIST, the complement of the class in $\{0,2,3,4,6\}$ prominently require more faces than other classes, and they all pertain to top clothes and share visual similarities (Figure \ref{fig: real dataset}(b)). Furthermore, it fails to find a single convex polytope cover of $\{0\}$ (with less than 30 faces) since the obtained polytope always contains many images in $\{2,3,4,6\}$ classes as illustrated in Figure~\ref{fig: real dataset}(b). 
In contrast, the class $\{1\}$ and its complement $\{1\}^c$ are separated by the fewest faces, suggesting they are less entangled with other classes. This observation is consistent with the distinctive, unique shape of the ``Trouser'' class in Fashion-MNIST dataset. This result highlights \emph{how neural networks can be utilized as a tool for quantifying the geometric complexity of datasets}. We provide additional interesting empirical examples in Appendix \ref{app: geometry}.

We further investigate the uniqueness of the obtained polytope covers in MNIST dataset. When Algorithm \ref{alg: compressing} is applied to networks initialized with small norm, the obtained covers exhibit noticeable similarity. 
For MNIST class $\{0\}$, we compute two distinct covering polytopes $\Cc_0$ and $\tilde\Cc_0$ with four faces. In Figure \ref{fig: face visualization}(a), the weight vectors $\wb_i$ of these covers are visualized, and the angles between the vectors of these two polytopes are displayed. It is easily verified that there is a clear one-to-one correspondence between vectors in the two polytopes, both visually and numerically.

For MNIST class $\{8\}$, which has a polytope cover with eight faces, a similar result is provided in Figure \ref{fig: face visualization}(b). Although the correspondence is slightly weaker than that of class $\{0\}$ due to the increased number of faces, most vectors still exhibit a strong one-to-one correspondence. 
From this, it can be seen that Algorithm \ref{alg: compressing} experimentally provides a unique polytope cover, offering further epexegetic support for the geometric simplicity of MNIST dataset.

Now, we provide \exacts for multi-class classifier for real datasets.
By combining the results in Table \ref{tab: result} with Theorem \ref{thm: compact}, we can ascertain the \exacts of these datasets, based on their geometric characteristics. Note that this is the first result on the minimal network architectures that can completely classify the given datasets, utilizing the geometric features of the datasets. It is provided in the remark below. 

\begin{remark} \label{rmk: minimal architecture}
    Adopting the covering polytopes with minimal number in Table \ref{tab: result} for each class, we deduce that
    \begin{align*}
        \text{MNIST : } &\quad \ReLUthree{784}{47}{10}{10} \\
        \text{Fashion-MNIST : } &\quad \ReLUthree{784}{90}{14}{10} \\
        \text{CIFAR10 : } &\quad \ReLUthree{3072}{178}{12}{10} 
    \end{align*}
    are \exacts for these datasets. Furthermore, the second and third weight matrices in these networks are highly sparse, as demonstrated in the proof of Theorem \ref{thm: compact} and illustrated in Figure \ref{fig: convex polytope}(c).
\end{remark}

It is worth noting that our findings stem from Algorithm \ref{alg: compressing}, which selectively removes and adjusts neurons within the network. Given the sparse connectivity observed in these networks, we anticipate an inherent connection between our results and the Lottery Ticket Hypothesis \citep{frankle2018lottery, malach2020proving}.
In other words, the sparse pruned neural networks suggested in Remark \ref{rmk: minimal architecture} can be understood as an example of the `winning ticket' in LTH that is explicitly constructed.

Additionally, our results offer another perspective on understanding LTH.
For instance, the assumptions in Theorem \ref{thm: three-layer polytope cover} shed light on the significance of masked and unmasked weights, and why maintaining the sign of weight values is important \citep{zhou2019deconstructing}. 
Specifically, the unmasked weights can be associated with the connection of faces to polytopes, and preserving the signs is crucial for maintaining the convex structure of these polytopes, as specified in \eqref{eq: two-layer constant bias} and \eqref{eq: three-layer condition}.
We hope our study contributes to future research efforts aimed at elucidating the principles underlying LTH. 




\section{Conclusion} \label{sec: conclusion}

In this paper, we investigated the geometric characteristics of datasets and neural network architecture.
Specifically, we established both upper and lower bounds on the necessary widths of network architectures for classifying given data manifolds, relying on its polytope-basis cover. 
Furthermore, we extended these insights to simplicial complexes or spaces consisting of prismatic polytopes, shedding light on how the width bound varies in response to the complexity of the dataset.
Conversely, we also demonstrated that the polytope structure of datasets can be inspected by training neural networks. 
We proposed such an algorithm, and our experimental results unveiled that each class within MNIST, Fashion-MNIST, and CIFAR10 datasets can be distinguished by at most two polytopes, implying geometric simplicity of real-world datasets.
We further analyzed that the number of faces in the polytope serves as an indicator of the geometric complexity of each class. 
Our empirical investigations unveil that neural network indeed converges to a polytope-basis cover of dataset, and conversely, it is possible to inspect geometrical features of the dataset from trained neural networks.

\paragraph{Limitations and future work.}
The optimality of the polytope-basis cover obtained by Algorithm \ref{alg: compressing} has not been clarified, which remains an interesting avenue for future research. 
Furthermore, while we only considered fully-connected networks in this work, it is desired to investigate the geometry in other network architectures like convolutional neural networks (CNNs).

\newpage
\section*{Acknowledgements}
This work was supported by Institute of Information \& communications Technology Planning \& Evaluation (IITP) grant funded by the Korea government(MSIT, Ministry of Science and ICT) (No. 2022-0-00984, Development of Artificial Intelligence Technology for Personalized Plug-and-Play Explanation and Verification of Explanation), by National Research Foundation of Korea(NRF) (**RS-2023-00262527**), by Institute of Information \& communications Technology Planning \& Evaluation (IITP) grant funded by the Korea government(MSIT)  
(No.2019-0-00075, Artificial Intelligence Graduate School Program(KAIST)), and by the National Research Foundation of Korea under Grant RS-2024-00336454

\section*{Impact Statement}
This paper presents work on describing the polytope structure in deep ReLU networks. This framework provides insights into the geometric roles of neurons and layers, potentially leading to a deeper understanding of high-dimensional dataset geometry and polytope structure of deep ReLU networks. The proposed approach may contribute to investigating data representation, the geometry of feature spaces, and understanding LTH. Lastly, there are no potential societal consequences of this work.


\bibliographystyle{icml2024}


\newpage
\appendix
\onecolumn
{\Large \textbf{Appendix}} \\
\newcommand{\indentHere}{\-\ \hspace{1.75cm}} 
{\large
\rule{\textwidth}{0.4pt}

\textbf{Appendix \ref{sec: related}.} Related Works

\textbf{Appendix \ref{app: geometry}.} Geometric Simplicity of Real-World Datasets

\textbf{Appendix \ref{app: algorithms}.} Algorithms for Finding Polytope-Basis Covers \\
\indentHere \ref{app: three-layer} A Polytope-Basis Cover Derived from a Trained Three-Layer ReLU Network \\
\indentHere \ref{app: two-layer} A Polytope-Basis Cover Derived from a Trained Two-layer ReLU Network \\
\indentHere \ref{app: real world algorithm} An Efficient Algorithm to Find a Simple Polytope-Basis Cover \\
\indentHere \ref{app: compare algorithms} Comparison of the Proposed Algorithms.

\textbf{Appendix \ref{sec: convergence}.} Convergence on the Proposed Networks

\textbf{Appendix \ref{app: proofs}.} Proofs \\
\indentHere \ref{app: proof 1} Proof of Proposition \ref{prop: convex polytope} \\
\indentHere \ref{app: proof compact} Proof of Theorem \ref{thm: compact} \\
\indentHere \ref{app: proof 2} Proof of Theorem \ref{thm: simplicial complex} \\
\indentHere \ref{app: proof 3} Proof of Theorem \ref{thm: betti numbers} \\
\indentHere \ref{app: proof 5} Proof of Theorem \ref{thm: three-layer polytope cover} \\
\indentHere \ref{app: proof 6} Proof of Proposition \ref{prop: algorithms} \\
\indentHere \ref{app: proof 4} Proof of Theorem \ref{thm: convergence}

\textbf{Appendix \ref{app: rebuttal}.} Additional Propositions and Lemmas

\rule{\textwidth}{0.4pt}
}


\section{Related Works} \label{sec: related} 

\paragraph{Geometric approaches to ReLU networks.}
Various geometric methodologies have been employed to explore the approximation capabilities of deep ReLU networks. \citet{hanin2019deep}, for instance, introduced the concept of bent hyperplanes in the input space, assessing its theoretical and empirical complexities. The methodology known as the 'bent hyperplane arrangement' has been applied across various research domains. Notably, it has been utilized in the analysis of decision regions \citep{beise2021decision, grigsby2022transversality, black2022interpreting} or characterizing linear regions within ReLU networks \citep{rolnick2020reverse}, which is also intricately connected to our proofs in Theorem \ref{prop: convex polytope} and \ref{thm: betti numbers}.

On the other hand, there has been a burgeoning interest in investigating polytope structures induced by deep ReLU networks \citep{fawzi2018empirical, xu2021traversing, alfarra2022decision, vincent2022reachable, black2022interpreting, haase2022lower, liu2023relu, fan2023deep, huchette2023deep, vallin2023geometric, piwek2023exact}. 
\citet{masden2022algorithmic} introduced algorithms capable of extracting the polytope structure inherent in networks and deriving topological properties of the decision boundary. \citet{carlsson2019geometry} and \citet{vallin2023geometric} considered the pre-image of ReLU networks, characterizing the geometric shapes of the decision boundary to gain an understanding of the polytope partitions of deep ReLU networks.

Nonetheless, there has been a scarcity of exploration into the explicit construction of neural networks for the purpose of classifying a given dataset, as illustrated in Figure \ref{fig: swiss}. 
In this study, we address the practical challenge of distinguishing between the two data manifolds, and we introduce practical algorithms for constructing covering polytopes based on the properties of ReLU networks. 
This approach can be considered a complementary method for investigating the approximation capabilities of neural networks in terms of polytopes - a geometric aspect that has not been extensively explored.

\paragraph{Exploring input convexity in neural networks.}
Recent years have witnessed a surge in research dedicated to unraveling the convexity of ReLU networks concerning their input. Crucially, it has been established that the ReLU activation function demonstrates convexity concerning its input when composited weights are all positive (Proposition 1 in \cite{amos2017input}, cf. Lemma \ref{lem: two-layer convex}). This discovery led to the inception of Input Convex Neural Networks (ICNNs) by \citet{amos2017input}, a characteristic that has been effectively leveraged in various applications \citep{makkuva2020optimal, chen2020input, fan2021scalable, bunning2021input, balestriero2022deephull}. Notably, \citet{balestriero2022deephull} applied this idea and proposed \emph{DeepHull}, the algorithm to approximate the convex hull by convex deep networks. 
On the other hand, \citet{sivaprasad2021curious} asserted that these convex neural networks exhibit self-regularization effects, demonstrating superior generalization performance in specific tasks.  
Additionally, research has delved into representing neural networks as the difference of convex (DC) functions \citep{sankaranarayanan2022cdinn, piwek2023exact, beltran2024dc}. \citet{sankaranarayanan2022cdinn} employed Linear Programming to optimize polyhedral DC functions, capitalizing on the piecewise linearity induced by the ReLU activation function. \cite{beltran2024dc} showed that DC NNs have an implicit bias avoiding overfitting in 1-D nonlinear regularization. \cite{sankaranarayanan2022cdinn} proposed a new network architecture called Convex Difference Neural Network (CDiNN), and suggested using convex concave procedure for optimization. 
These findings of previous studies suggest that investigating neural networks through the lens of differences in convex functions holds significant potential for future research.

In this study, we decompose a trained ReLU network into the difference of several convex functions and leverage this decomposition to induce a polytope-basis cover for a given dataset. Specifically, we employ the ICNN architectures to derive a convex polytope cover, revealing how trained neural networks capture the geometric characteristics of the training dataset. While \citet{balestriero2022deephull} explored a similar approach by minimizing the volume of the polytope as a regularizer, our main focus is on minimizing the number of neurons to reduce a polytope with a small number of faces. Consequently, our empirical results can be considered as an application of ICNNs to extract geometric features of datasets in terms of polytopes.

Moreover, it is noteworthy that previous studies imposed restrictions on the weights, either by enforcing $W_k \ge 0$ \citep{amos2017input, sivaprasad2021curious} or by substituting them with squared values $W_k^2 \ge 0$ \citep{sankaranarayanan2022cdinn}, to maintain network convexity. In contrast, we achieve this solely by adjusting the initialization conditions, leveraging the implicit bias of gradient descent \citep{du2018algorithmic}.

\paragraph{Complexity of datasets and neural network architectures.}
Several studies have delved into the intricate relationship between the geometric features of datasets and neural network training, often referred to as the \emph{multiple manifold problem} \citep{goldt2020modeling, buchanan2020deep, wang2021deep, chen2022nonparametric, tiwari2022effects}. This problem typically involves the binary classification of two low-dimensional manifolds by neural networks.
For instance, \citet{buchanan2020deep} and \citet{wang2021deep} focused on the task of distinguishing between two curves (i.e., one-dimensional manifolds), investigating the convergence speed and generalization concerning the geometric features of the dataset. Similarly, in the context of low-dimensional data manifolds, \citet{tiwari2022effects} examined the effects of data geometry on the complexity of trained neural networks, measuring the distance to the manifold. In a related vein, \citet{dirksen2022separation} considered the separation problem with randomly initialized ReLU networks, explicitly linking required widths to weight initialization. 

Furthermore, there are other numerous empirical studies that have supported the implicit relationship between network architecture and the geometric complexity or topological structure of the data manifold \citep{fawzi2018empirical, kim2020pllay, cohen2020separability, birdal2021intrinsic, barannikov2022representation, barannikov2021manifold, naitzat2020topology, hajij2020topological, hajij2021topological, magai2022topology}.
Additionally, \citet{li2018measuring} empirically investigated the loss landscape to measure the intrinsic dimension of datasets, which is deeply related to the minimal neural network architecture. 
These theoretical and empirical findings suggest a high correlation between neural network architecture and the training dataset: \emph{a more complicated dataset requires a more complex architecture}.

However, a notable gap still exists in the literature when it comes to explicitly constructing a neural network in practice. For example, as introduced in Section \ref{sec: intro}, it remains unknown which architecture of neural networks can or cannot completely classify a given dataset with explicit construction of neural networks. 



\paragraph{Estimating dataset characteristics from a trained network.}
Once a neural network is trained, it is well-established that the trained network encapsulates information or characteristics of the training dataset \citep{carlini2019secret, carlini2021extracting, haim2022reconstructing}. This phenomenon, implies that neural networks can be employed to extract information about the dataset through the training process. 
From a topological perspective, \citet{paul2019estimating} introduced a method for estimating the Betti numbers of 2D or 3D datasets using convolutional neural networks, later extended to 4D datasets by \cite{hannouch2023topology}.
In the geometric realm, \citet{kantchelian2014large} proposed the large-margin convex polytope machine with a training algorithm to find a convex polytope that encloses one class with a large margin. Their empirical results demonstrated that the digit '2' class of MNIST has a simple convex polytope cover that generalizes well. 

However, there are still opportunities for more precise investigations into the dataset geometry in terms of polytopes, which are induced from ReLU networks. In this paper, we propose an algorithm that derives a polytope-basis cover of a given dataset by learning neural networks, building upon our theoretical analysis. Moreover, it reveals the number of faces of the polytopes, which can describe the geometric complexity of dataset. 



\paragraph{Our contributions.}

In this paper, we harmonize diverse perspectives on neural networks, offering insights into the intrinsic relationship between the geometric complexity of datasets and network architectures. Our focus centers on the polytope structure induced by ReLU networks (Theorems \ref{thm: compact}, \ref{thm: simplicial complex}, \ref{thm: betti numbers}). 
The principal contribution of our work lies in delineating lower and upper bounds on widths within deep ReLU networks for the classification of a given dataset, drawing from the polytope structure inherent in the data. Our theoretical results not only elucidate how the geometric complexity of a dataset influences the required widths of neural networks but also illuminate the nuanced role of neurons in deep layers.

Furthermore, we present algorithms aimed at deriving a polytope-basis cover for given datasets, thereby highlighting the inherent link between trained neural networks and the polytope structure of the training datasets. While prior research \citep{kantchelian2014large, sivaprasad2021curious} merely showcased the polyhedral separability of certain classes in MNIST or CIFAR10, we determine the exact number of faces required for covering polytopes of each class. Essentially, our work introduces a novel methodology for exploring the intricate relationship between dataset geometry and neural networks.

\section{Geometric Simplicity of Real-World Datasets} \label{app: geometry}

In this section, we present empirical results illustrating how the number of faces of covering polytopes reflects the geometric characteristics of a class. Specifically, we examine a classification task under varying levels of noisy labels. We focus on the class $\{1\}$ in each dataset and manipulate the noise levels (denoted as $r$) across values of 0, 0.01, 0.1, 0.25, 0.5, and 0.90 while maintaining the total number of data points in the class. The noisy class with noise level $r$ is denoted by $\{1\}_r^{noise}$.

Here we give an example: When $r=0.00$, the $\{1\}_{r=0.00}^{noise}$ class is identical to the $\{1\}$ class in Fashion-MNIST, comprising 6000 T-shirt images. As $r$ increases, such as $r=0.01$, the noised class $\{1\}_{r=0.01}^{noise}$ still consists of 6000 images, but only 99\% of them are T-shirt images, with the remaining 1\% randomly selected from other classes. Consequently, when $r=0.9$, $\{1\}_{r=0.9}^{noise}$ contains 600 T-shirt images and 5400 randomly selected images from other classes, i.e., totally randomly selected images in the Fashion-MNIST dataset. The task is finding a single covering polytope that contains $\{1\}_{r=0.9}^{noise}$ against the other data.

We utilize Algorithm \ref{alg: compressing} to identify a single polytope with minimal width, achieving an accuracy greater than 99.9\% on the noised dataset \footnote{Note that Tables~\ref{tab: result} and~\ref{tab: random label} employ different accuracy criteria. With a stricter criterion of 99.99\% in Table~\ref{tab: result}, the identified polytopes are likely more precise but potentially require the larger number of faces. Conversely, the looser criterion of 99.9\% used in Table~\ref{tab: random label} may lead to slightly loose polytope, but easier to identify.}. To mitigate the randomness inherent in image selection, we compute the average value over five repetitions. The results are presented in Table \ref{tab: random label}. 

\begin{table}[ht]
\centering
\resizebox{0.6\columnwidth}{!}{ 
\begin{tabular}{ccccccc}
Dataset \textbackslash\ noise level $r\vphantom{\Big|}$ & 0.00 & 0.01 & 0.10 & 0.25 & 0.50 & 0.90 \\ \hlineB{3}
MNIST class $\{1\}\vphantom{\Big|}$  &
3 & 5.2 & 29.2  & 46.6  & 64.0 & 52.4    \\ \hline
Fashion-MNIST class $\{1\}\vphantom{\Big|}$  &
3&  4.6 & 39.0  &  57.4 &  77.0 & 65.6 \\ \hline 
CIFAR10 class $\{1\}\vphantom{\Big|}$ &
12 & 12.8  &  16.2 & 19.6 & 28.4 &  35.6  \\ \hline
\end{tabular}
}
\caption{
The average number of faces required for a convex polytope to cover the noisy class $\{1\}^{noise}_{r}$, with varying levels of random labels $(r)$. As the proportion of random labels in the class rises, there is a noticeable corresponding increase in the requisite number of faces.
}
\label{tab: random label} 
\end{table}

The results illustrate a positive correlation between the prevalence of noisy labels and the increasing complexity of covering polytopes: as the level of noise rises, the requisite number of faces also increases accordingly. Specifically, Table \ref{tab: random label} demonstrates that images within the same class can be effectively distinguished by a convex polytope with a smaller number of faces, highlighting the inherent geometric complexity of real-world datasets.

Moreover, our result can demonstrate some previous works that classified these datasets using convex polytopes. Notably, for MNIST and CIFAR10 datasets, \citet{kantchelian2014large} and \citet{sivaprasad2021curious} observed that convex polytope classifiers exhibit self-regularization effects and robustness to label noise. Our results presented in Table \ref{tab: random label} validate this observation based on the geometric features of the datasets: each class in the datasets inherently possesses simple geometry, enabling convex classifiers to generalize effectively. 

A particularly intriguing finding emerges when comparing the results with 50\% corrupted labels ($r=0.5$) to those with fully-mixed labels ($r=0.9$). Notably, for MNIST and Fashion-MNIST datasets, Table \ref{tab: random label} suggests that the 50\% corrupted dataset is even more challenging to segregate than the fully-mixed counterpart. In other words, a convex polytope with fewer faces must encompass the original images from the uncorrupted class $\{1\}$ -the 'digit 1 images' in MNIST or the 'Trouser images' in Fashion-MNIST-, necessitating a larger number of faces for the polytope to effectively segregate them. This observation emphasizes the geometric simplicity of MNIST and Fashion-MNIST datasets, contrasting with the behavior observed in CIFAR10, which does not exhibit this trend.
\begin{wrapfigure}{r}{0.3\textwidth}
  \centering
  \includegraphics[width=0.18\textwidth]{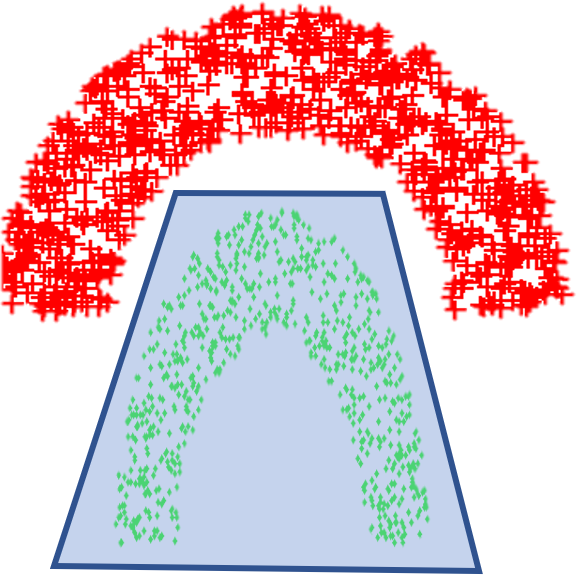}
  \caption{Although two manifolds are polyhedrally separable, both manifolds may not exhibit convexity.}\vspace{-1cm}
  \label{fig: bird}
\end{wrapfigure}   

\vspace{1mm}
\begin{remark}
To prevent potential misunderstandings regarding the results in Table \ref{tab: random label}, we offer two insights. Firstly, in the case of the fully-mixed class ($r=0.9$), it is not surprising that an arbitrary set of points in high-dimensional space can be separated by a single convex polytope. This phenomenon is rooted in the curse of dimensionality, where in higher dimensions, points are inherently easier to separate \citep{pestov2013k, gorban2018blessing}. Indeed, Table \ref{tab: random label} illustrates that random images in CIFAR10 are separated with fewer faces compared to other datasets. Secondly, it's worth noting that although two manifolds may be polyhedrally separable, this does not necessarily imply that each manifold is convex, as depicted in Figure \ref{fig: bird}.
\end{remark}

\section{Algorithms for Finding Polytope-basis Covers} \label{app: algorithms}

This section delves into the inner workings of our algorithms. Each algorithm is presented with a breakdown of its steps, the rationale behind its design, and its theoretical underpinnings. Additionally, two novel algorithms dedicated to polytope-basis cover reduction are introduced. For clarity, the section is organized into four sections.

In Section \ref{app: three-layer}, we delve into the motivation behind Algorithm \ref{alg: compressing} and provide a comprehensive explanation of its implementation.
We introduce Algorithm \ref{alg: two-layer} in Section \ref{app: two-layer}, which extracts a polytope-basis cover from any trained two-layer ReLU network. 
Section~\ref{app: real world algorithm} introduces Algorithm~\ref{alg: polytopes in order}, an efficient algorithm for generating a minimal polytope-basis cover for a given dataset. Finally, Section~\ref{app: compare algorithms} presents a comparative analysis of all proposed algorithms.

\subsection{A Polytope-Basis Cover Derived From a Three-Layer ReLU Network} \label{app: three-layer} 

In Section \ref{sec: algorithms}, we introduced the compressing algorithm (Algorithm \ref{alg: compressing}) that deforms a two-layer network to represent a single convex polytope. 
Before we provide detail descriptions, we introduce some terminologies.

\begin{definition} \label{def: activation}
    Let $\Nc(\xb):= v_0 + \sum_{k=1}^m v_k\sigma(\wb_k^\top\xb +b_k)$ be a two-layer ReLU network. For each $k\in[m]$, we refer a pair $(v_k, \wb_k, b_k)$ as a \emph{neuron} of $\Nc$. 
    We say \emph{a neuron $\sigma(\wb_k^\top\xb+b_k)$ activates (or deactivates) $\xb$} if $\wb_k^\top\xb+b_k>0$ (or $\wb_k^\top\xb+b_k<0$, respectively).
    The \emph{activation boundary (AB)} of a neuron $(v_k, \wb_k, b_k)$ is defined by $\{\xb\in\Rd^d ~|~ \wb_k^\top\xb+b_k = 0\}$.
    Similarly, the \emph{decision boundary (DB)} of $\Nc$ is defined by the set $\{\xb\in\Rd^d ~|~ \Nc(\xb)=0\}$.
\end{definition}
Roughly speaking, the activation boundaries of $\Tc$ are non-differentiable points of $\Tc$.
For example, the leftmost column in Figure \ref{fig: three-layer} or \ref{fig: general two-layer} shows the decision boundary and activation boundaries of trained networks. 

Now, we provide in-depth explanations of the algorithms introduced in the main text. 
Let $\Tc$ be a two-layer ReLU network defined in \eqref{eq: two-layer constant bias}, thus $v_k<0$ for all $k\in[\width]$. 
Let $S$ be the region defined by $S:= \{\xb~|~\Tc(\xb)=\lambdabias\}$. If it is nonempty, then it is a convex polytope with $\width$ faces by Definition \ref{def: convex polytope}.
However, its decision boundary $R:= \{\xb~|~\Tc(\xb)>0\} \supset S$ may contain more data points than $S$. For the given dataset $\Dc$, to achieve \eqref{eq: three-layer condition}, the objective of the compressing algorithm is to achieve
\begin{align} \label{eq: goal of scaling}
   R \cap \Dc = S \cap \Dc.
\end{align}

Note that \eqref{eq: goal of scaling} directly implies \eqref{eq: three-layer condition}. Below, we demonstrate a detailed examination of the compressing algorithm step-by-step, which was briefly introduced in Section \ref{sec: main}.
The algorithm comprises two parts: \textbf{1.} eliminating a redundant neuron, and \textbf{2.} scaling some neurons.


\begin{itemize}

\item {\bf PART 1. Eliminating a redundant neuron.}
The first part involves the removal of remaining redundant neurons that may activate some data points but do not contribute significantly to the network output. Specifically, in Figure \ref{fig: compressing}(a), the neuron indicated by the red arrow does not contribute to the change of the decision boundary since the training data points activated by this neuron already exhibit negative outputs, i.e., $\Tc(\xb)<0$ already. In essence, eliminating such neurons does not significantly alter the decision boundary of $\Tc$. 
However, although the removal of this neuron maintains the decision boundary of $\Tc$, it does affect the output value of $\Tc$, consequently influencing other subnetworks in the three-layer network $\Nc$. To address this, the algorithm removes only one neuron at once, which has the smallest value of $|v_k| \cdot \norm{\wb_k}$.

\item {\bf PART 2. Scaling neurons.}
The second part is deforming the given network to satisfy \eqref{eq: goal of scaling} by magnifying neurons.
From the definition of $\Tc(\xb)=\lambdabias+\sum_k v_k\sigma(\wb_k^\top\xb +b_k)$, recall that all $v_k<0$. If we take $v_k \rightarrow -\infty$ for all $k$, then the region $R=\{\xb\in\Rd^d ~|~ 0<\Tc(\xb)\}$ shrinks to the region $S=\{\xb\in\Rd^d ~|~ \Tc(\xb)=\lambdabias\}$ (cf. Figure \ref{fig: assumptions}(b)). This is the trick we used to figure out the minimal polytope representation of the decision region. 

In this step, we just increase the magnitude of a neuron $(v_k, \wb_k, b_k)$ for $k\in[m]$ if it has an activated data $\xb_i$ such that $\Tc_j(\xb) \neq \lambdabias$.
The multiplication constant is referred to $\lambda_{scale}$ to $(v_k, \wb_k, b_k)$ in the algorithm. Furthermore, note also that the updated neuron still satisfies \eqref{eq: initialization}, thus gradient desecent algorithm can be parallelized with keeping $v_k<0$. 
It is noteworthy that in Figure~\ref{fig: compressing}(b), the decision boundary demonstrably shrinks to the polytope shown in (c).
\end{itemize}


However, removing a neuron (PART 1) and adjusting the scaling of some neurons (PART 2) will inevitably alter the network's output, which could potentially decrease its classification accuracy. Therefore, as highlighted in Section \ref{sec: main}, it is advisable to apply the compressing algorithm in conjunction with an optimization process, as outlined in Algorithm \ref{alg: three-layer}.

\begin{algorithm}[htbp]
\caption{Extracting a polytope-basis cover from a three-layer ReLU network}
\label{alg: three-layer}
\begin{algorithmic}
\STATE {\bf Require:} a pretrained three-layer network $\Nc(\xb)$ defined in \eqref{eq: three-layer}, the training dataset $\Dc=\{(x_i, y_i)\}_{i=1}^n$, Epochs
\FOR{$epoch=1,\cdots, Epochs$}
\FOR{$iteration=1,2,\cdots,1000$}
\STATE one-step gradient descent for $\Nc$ under the BCE loss \eqref{eq: BCE loss}
\ENDFOR
\FOR{$j=1,\cdots,J$}
\STATE $\Tc_j \leftarrow$ \texttt{COMP}$(\Tc_j)$
\COMMENT{the compressing algorithm (Algorithm \ref{alg: compressing})}
\ENDFOR
\ENDFOR
\IF{there exists $i\in[n]$ and $j\in[J]$ such that $0<\Tc_j(\xb_i)<1$ }
\REPEAT
\STATE $\Tc_j \leftarrow$ \texttt{COMP}$(\Tc_j)$
\COMMENT{the compressing algorithm (Algorithm \ref{alg: compressing})}
\UNTIL{$\sigma(\Tc_j(\xb_i))$ is either $0$ or $1$ for all $i\in[n]$ and $j\in[J]$}
\ENDIF
\STATE {\bf Output:} $\Nc$ 
\end{algorithmic}
\end{algorithm}

Figure \ref{fig: three-layer} illustrates the results obtained by applying Algorithm \ref{alg: three-layer}. A three-layer network $\Nc$ defined in \eqref{eq: three-layer} with $J=16$ polytopes, each consisting of 20 neurons (architecture $\ReLUthree{2}{320}{16}{1}$), was pre-trained on the swiss-roll dataset. Subsequently, Algorithm \ref{alg: compressing} was applied to compress each $\Tc_j$ with fine tuning, resulting in a compressed three-layer network with architecture $\ReLUthree{2}{72}{16}{1}$. 
If the obtained network still completely classify the dataset, then we can derive the polytope basis cover by Theorem \ref{thm: compact}.
More precisely, the polytopes defined by $C_j := \{\xb\in\Rd^d~|~\Tc_j(\xb) = \lambdabias\}$, the collection of polytoeps $\Cc = \{C_j\}_{j \in [J]}$ becomes a polytope-basis cover of the dataset, comprising these 16 polytopes of $\Nc$. The obtained polytopes are illustrated in Figure \ref{fig: three-layer}(b, c).

Furthermore, it is noteworthy to mention the time efficiency of both Algorithm \ref{alg: compressing} and \ref{alg: three-layer}. Despite the presence of multiple \textbf{for} loops, these algorithms are implemented efficiently using parallel computing in PyTorch. Empirically, they demonstrate quick performance with practical neural network widths. On a theoretical level, Proposition \ref{prop: algorithms} ensures that Algorithm \ref{alg: three-layer} terminates within a finite timeframe and provides a polytope-basis cover that maintaining the accuracy of the converged neural network $\Nc$.

\subsection{A Polytope-Basis Covers Derived From a Two-Layer ReLU Network} \label{app: two-layer}

\begin{figure*}[t]
    \centering
    \includegraphics[width=0.98\textwidth]{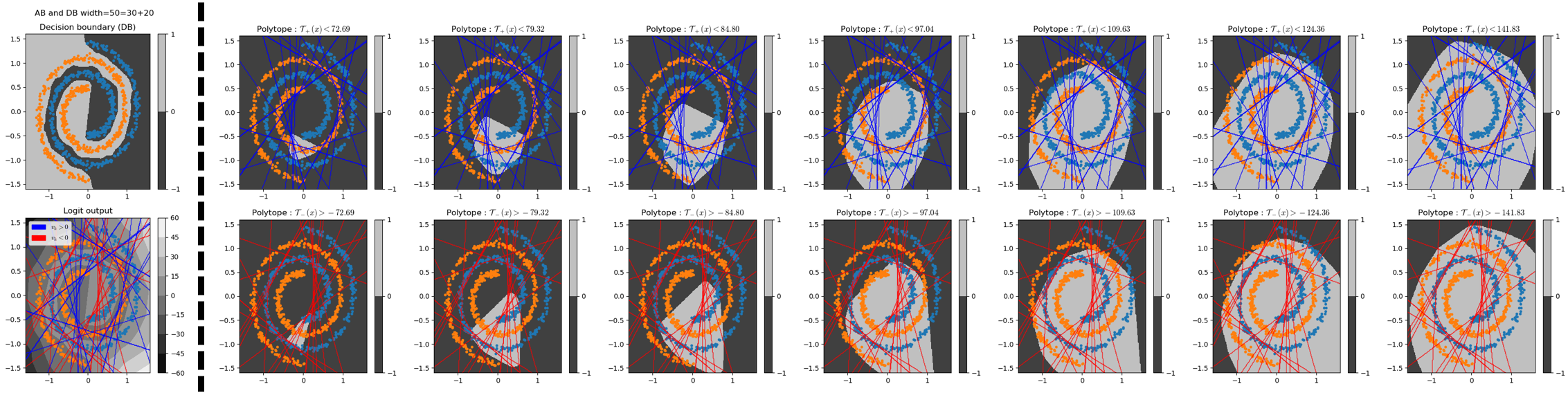}
    \caption{
    A polytope-basis cover derived by Algorithm \ref{alg: two-layer} from a trained two-layer ReLU network with architecture $\ReLUtwo{2}{50}{1}$. 
    The decision boundary and all activation boundaries of the converged network are depicted in the leftmost column. The algorithm provides a polytope-basis cover consists of 68 polytopes, and some of them are illustrated in other columns. 
    }
    \label{fig: general two-layer}
\end{figure*}

In this section, we propose another algorithm that extracts a polytope-basis cover from a trained two-layer ReLU network $\Nc$ defined in \eqref{eq: two layer relu}. 
First, we decompose $\Nc$ as the sum of convex and concave functions by aligning it according to the sign of the weight values.
\begin{align*}
    \Nc(\xb) &= v_0 + \sum_{k=1}^m v_k\sigma(\wb_k^\top\xb_k+b_k) \\
    &= \left( \frac12 v_0 + \sum_{v_k>0} v_k\sigma(\wb_k^\top\xb_k+b_k) \right)
    + \left( \frac12 v_0 + \sum_{v_k<0} v_k\sigma(\wb_k^\top\xb_k+b_k) \right) \\
    &=: \Nc_+(\xb) + \Nc_-(\xb).
\end{align*}
Note that both $\Nc_+$ and $\Nc_-$ are convex and concave functions, respectively, by Lemma \ref{lem: two-layer convex}.
Now, we consider the network output of each data. For any $\xb_i\in \Dc$, we have
\begin{align*}
    \Nc(\xb_i) &> 0 \qquad \Leftrightarrow \qquad \Nc_+(\xb_i) > - \Nc_-(\xb_i), \\ 
    \Nc(\xb_i) &< 0 \qquad \Leftrightarrow \qquad \Nc_+(\xb_i) < - \Nc_-(\xb_i).
\end{align*}
Then, we quantize these functions to derive a polytope-basis cover. With the similar idea of Lebesgue integration, we can approximate the convex function $\Nc_+(\xb)$ by a \emph{simple function}. Here, the \emph{simple function} means a linear combination of indicator functions, basically considered in mathematical field like Lebesgue theory \citep{rudin1976principles}. Define $M:=\max_{\xb\in\Dc}\Nc_+(\xb)$. Then, for a given $\varepsilon>0$, we can approximate $\Nc_+(\xb)$ by 
\begin{align*}
    \Nc_+(\xb) &\approx M - \sum_{l=0}^\infty l\varepsilon \cdot \indicator{M-(l+1)\varepsilon <\Nc_+(\xb)<M-l\varepsilon}(\xb) \\
    &= M- \varepsilon \sum_{C\in \Cc_Q} \indicator{\xb\in C}
\end{align*}
where $\Cc_Q$ is the collection of polytopes defined by $C_l:=\{\xb~|~\Nc_+(\xb)<M - l\varepsilon\}$ for $l=0,1,\cdots$. 
This decomposition can be understood as quantization of $\Nc_+(\xb)$ by slices with height $\varepsilon$. Moreover, the above approximation would be accurate as $\varepsilon\rightarrow0$. Similarly, $\Nc_-(\xb)$ can be approximated in the similar manner. The main idea to obtain a polytope-basis cover $\Cc$ from $\Nc$ is selecting sufficiently many $\varepsilon$'s to quantize $\Nc_+$ and $\Nc_-$. 


Empricially, we construct a polytope-basis cover from the values of $\Nc$. If $\Cc$ is not a polytope-basis cover yet, we select an incorrectly classified data point $\hat{\xb}$ with the smallest confidence value, i.e., $\hat\xb := \argmin_{\xb_i \in \Dc} |\Nc(\xb_i)|$. Then, we choose an intermediate value $c$ between $\Nc_+(\hat{\xb})$ and $\Nc_-(\hat{\xb})$ and add two polytopes $C_+ := \{\xb~|~\Nc_+(\xb) < c\}$ and $C_- := \{\xb~|~\Nc_-(\xb) > -c\}$ to the polytope-basis cover $\Cc$.  Then, the value
$$
\left| \sum_{C \in \Cc_P} \indicator{\xb \in C} (\hat\xb) - \sum_{C \in \Cc_Q} \indicator{\xb \in C} (\hat\xb) \right|
$$
is decreased by one since $\hat{\xb}$ is contained in either $C_+$ or $C_-$. Therefore, by repeating this process sufficiently many times, $\Cc$ will correctly classify $\hat{\xb}$.
Based on this idea, we provide Algorithm \ref{alg: two-layer} that extracts a polytope-basis cover from a given trained two-layer ReLU network $\Nc$. 

\begin{algorithm*}[htbp]
\caption{Extracting a polytope-basis cover from a trained two-layer ReLU network} \label{alg: two-layer} 
\begin{algorithmic}
\STATE {\bf Require:} a pretrained two-layer ReLU network $\Nc$ defined in \eqref{eq: two layer relu}, training dataset $\Dc=\{(x_i, y_i)\}_{i=1}^n$
\STATE Declare the empty collections $\Cc_P$ and $\Cc_Q$.
\REPEAT
\STATE Define $o_i := \sum\limits_{C\in\Cc_P} \indicator{\xb_i \in C} - \sum\limits_{C\in\Cc_Q} \indicator{\xb_i \in C} -\frac12 $ for all $i\in[n]$.
\IF{$\Cc$ is not a polytope-basis cover of $\Dc$}
\STATE $\hat\xb \leftarrow \argmin\limits_{\text{sgn}(o_i) \neq \text{sgn}(\Nc(\xb_i)) } 
\left|\; \Nc_+(\xb_i) + \Nc_-(\xb_i) \;\right|$
\COMMENT{$\hat\xb$ is not correctly covered by $\Cc$, and has the smallest confidence.} 
\STATE $c \leftarrow \frac12 ( \Nc_+(\hat\xb) - \Nc_-(\hat\xb))$ 
\STATE Add the polytope $C:=\{\xb~|~\Nc_-(\xb)>-c\}$ in $\Cc_P$. 
\STATE Add the polytope $C:=\{\xb~|~\Nc_+(\xb)<c\}$ in $\Cc_Q$.  
\STATE $\Cc \leftarrow \Cc_P \cup \Cc_Q$
\COMMENT{Now, $\hat\xb$ is correctly covered by $\Cc$.}
\ENDIF
\UNTIL{$\Cc$ becomes a polytope-basis cover of $\Dc$}
\STATE {\bf Output:} $\Cc$
\COMMENT{The polytope-basis cover of $\Dc$ derived from $\Nc$.}
\end{algorithmic}
\end{algorithm*}

The result of Algorithm \ref{alg: two-layer} on a two-layer ReLU network with architecture $\ReLUtwo{2}{50}{1}$, which is trained on the Swiss roll dataset, is illustrated in Figure \ref{fig: general two-layer}. Specifically, the algorithm generates a polytope-basis cover of the dataset consists of 68 polytopes ($\Cc_P$ and $\Cc_Q$ consists of 34 polytopes, respectively). 14 polytopes in $\Cc$ is illustrated in 2nd column to 8th column in Figure \ref{fig: general two-layer}. 
Proposition \ref{prop: algorithms} guarantees that Algorithm \ref{alg: two-layer} must terminate in finite time, and produces a polytope-basis cover of $\Dc$ which has the same accuracy with the given $\Nc$. 
Therefore, training a two-layer ReLU network to 100\% accuracy on the dataset $\Dc$, Algorithm \ref{alg: two-layer} allows to derive a polytope-basis cover of the given dataset.

There are some pros and cons in Algorithm \ref{alg: two-layer}. The advatages of the algorithm are 1. it can be applied to arbitrary two-layer ReLU networks, and 2. it does not modify the trained network. Therefore, it unveils the inherent polytope-basis cover and convex polytope structures in trained two-layer ReLU networks.
However, two drawbacks in this algorithm are 1. generally it induces many polytopes in practice, and 2. the number of faces of each polytope is unknown. 
For detail comparison with other algorithms, see Section \ref{app: compare algorithms}.

\subsection{An Efficient Algorithm to Find a Simple Polytope-Basis Cover} \label{app: real world algorithm}

In the preceding subsections, we introduced two algorithms in Sections \ref{app: three-layer} and \ref{app: two-layer} that extract a polytope-basis cover from trained two-layer or three-layer ReLU networks. However, the results obtained from both algorithms, as illustrated in Figure \ref{fig: three-layer} and \ref{fig: general two-layer}, still exhibit too many polytopes on the training dataset. As demonstrated for the Swiss roll dataset in Figure \ref{fig: swiss}, we have previously shown the existence of a polytope-basis cover comprising only four polytopes (refer to Figure \ref{fig: swiss} and \ref{fig: convex polytope}). 

In this section, to address the aforementioned issue, we present an efficient algorithm designed to find a polytope-basis cover with a reduced number of polytopes. This algorithm is outlined in Algorithm \ref{alg: polytopes in order}. The key distinction of this algorithm from the previous ones is that it does not derive a polytope cover from a trained network. Instead, it sequentially identifies a convex polytope by training several two-layer ReLU networks defined in \eqref{eq: two-layer constant bias}. Consequently, this algorithm only requires access to the training dataset.

Before we demonstrate the algorithm, we provide modified network and loss functions. Specifically, we consider the following two types of two-layer ReLU networks. Here, $\lambda_{bias}>0$ is a hyperparameter that enhances the convergences as introduced in \eqref{eq: two-layer constant bias}.
\begin{align}
    \Tc_+(\xb) &:= \lambda_{bias} + \sum_{k=1}^m v_k \sigma(\wb_k^\top\xb+b_k), \qquad \forall v_k < 0 \label{eq: TC+} \\
    \Tc_-(\xb) &:= -\lambda_{bias} + \sum_{k=1}^m v_k \sigma(\wb_k^\top\xb+b_k),  \qquad \forall v_k > 0 \label{eq: TC-} 
\end{align} 
If $\lambda_{bias}$ is large, then $\Tc_+$ has a large output value at initialization, and gradient descent optimization is heavily affected by data with $0$ labels. A similar situation happens for $\Tc_-$, and it helps to find a single polytope that contains whole class data. Practically, it is enough to use $\lambda_{bias}=5$ to find such polytopes.

The modified loss function is defined by
\begin{align} \label{eq: reinforced bce loss}
    L_{BCE, \lambdab} (\Theta) := - \frac{1}{|\Dc_0|} \sum_{y_i = 0} \lambda_0 \cdot \ell\left( \SIG \circ \Tc(\xb_i), y_i \right) - \frac{1}{|\Dc_1|} \sum_{y_i = 1} \lambda_1 \cdot \ell\left(\SIG \circ \Tc(\xb_i), y_i \right) 
\end{align}
where $\lambdab=(\lambda_0, \lambda_1)$ is the hyperparameter proposed to reinforce to cover a whole data class with specific label. For instance, by using a large value of $\lambda_1$, $\Tc$ can be trained to cover whole data points that have label $y_i=1$. 
After successfully configured the first polytope $C_1$, we train the second network $\Tc_2$ to distinguish data points of another data class, $y_i=0$, in the obtained polytope $C_1$. Repeating this process alternatively, Algorithm \ref{alg: polytopes in order} generally provides a polytope-basis cover with a small number of polytopes.

Now, we provide a detailed illustration of the algorithm's functionality with the example displayed in Figure \ref{fig: polytopes in order}. Here we use $\lambda_{bias}=5$, and $\lambdab=(1, 10)$.
First, the algorithm trains $\Tc_1$ on the entire dataset $\Dc$ using the loss function in \eqref{eq: reinforced bce loss}. After fine-tuning through Algorithm \ref{alg: three-layer}, we obtain the first polytope $C_1$ displayed in Figure \ref{fig: polytopes in order}(b). Note that all orange data points $(\Dc_1)$ are contained in $C_1$.
Next, the second network $\Tc_2$ is trained on $(\Dc_0 \cap C_1) \cup \Dc_1$ using \eqref{eq: reinforced bce loss} with $\lambdab=(1,10)$. $\Tc_2$ is aimed to cover all blue data points $(\Dc_0)$ within $C_1$. After training and fine-tuning, we obtain the second polytope $C_2$ displayed in Figure \ref{fig: polytopes in order}(c).
Similarly, we can find the third polytope $C_3$ by training another network $\Tc_3$ on $\Dc_0 \cup (\Dc_1 \cap C_2)$, displayed in Figure \ref{fig: polytopes in order}(d) and so on. In the example in Figure \ref{fig: polytopes in order}, totally four polytopes are obtained, and visualized through (b) to (e). Lastly, by Theorem \ref{thm: compact}, we can construct a three-layer ReLU network with architecture $\ReLUthree{2}{17}{4}{1}$ that can completely classify this swiss roll dataset, based on the obtained polytope-basis cover. The decision boundary of the constructed three-layer ReLU network is illustrated in Figure \ref{fig: polytopes in order}(f).

\begin{algorithm}[htbp]
\caption{An efficient algorithm for finding a polytope-basis cover}
\label{alg: polytopes in order}
\begin{algorithmic}
\STATE {\bf Require:} training dataset $\Dc =\Dc_0 \cup \Dc_1 = \{(x_i, y_i)\}_{i=1}^n$, hyperparameter $\lambdab = (\lambda_0, \lambda_1)$, $acc_{th}, \lambda_{bias}$, $width$
\STATE $\Cc_P \leftarrow \emptyset$, $\Cc_Q \leftarrow \emptyset$.
\STATE $m \leftarrow width$
\REPEAT
\REPEAT
\STATE Initialize $\Tc_+$ defined in \eqref{eq: TC+}.
\STATE Train $\Tc_+$ on the dataset $\Dc \cap \Cc_P^c$ by gradient descent, under the modified BCE loss \eqref{eq: reinforced bce loss}.
\STATE Fine tune the trained $\Tc_+$ by Algorithm \ref{alg: three-layer}.
\IF{$\Tc_+(\xb) \neq \lambda_{bias}$ for some $\xb\in\Dc_1\cap \Cc_P^c$}
\STATE $m \leftarrow m+1$
\ENDIF
\UNTIL{$\Tc_+(\xb) = \lambda_{bias}$ for all $\xb\in\Dc_1\cap \Cc_P^c$}
\STATE $A \leftarrow \{\xb\in\Rd^d ~|~ \Tc_+(\xb) = \lambda_{bias}\}$ 
\COMMENT{This is a polytope covering $\Dc_1 \cap \Cc_P^c$}
\STATE Add $A$ in $\Cc_P$
\STATE $m \leftarrow width$
\REPEAT
\STATE Initialize $\Tc_-$ defined in \eqref{eq: TC-}.
\STATE Train $\Tc_-$ on $\Dc\cap\Cc_Q^c$ by gradient descent, under the modified BCE loss \eqref{eq: reinforced bce loss}.
\STATE Fine tune the trained $\Tc_-$ by Algorithm \ref{alg: three-layer}.
\IF{$\Tc_-(\xb) \neq -\lambda_{bias}$ for some $\xb\in\Dc_0\cap \Cc_Q^c$}
\STATE $m \leftarrow m+1$
\ENDIF
\UNTIL{$\Tc_-(\xb) = -\lambda_{bias}$ for all $\xb\in\Dc_0\cap \Cc_Q^c$}
\STATE $A \leftarrow \{\xb\in\Rd^d ~|~ \Tc_-(\xb) = -\lambda_{bias}\}$ 
\COMMENT{This is a polytope covering $\Dc_0 \cap \Cc_Q^c$}
\STATE Add $A$ in $\Cc_Q$
\STATE $\Cc \leftarrow \Cc_P \cup \Cc_Q$
\UNTIL{$\Cc$ becomes a polytope-basis cover of $\Dc$ with accuracy greater than $acc_{th}$} 
\STATE {\bf Output:} $\Cc$
\end{algorithmic}
\end{algorithm}

The key advantage of this algorithm lies in its ability to generate a small number of convex polytopes, in contrast to other algorithms we proposed. As demonstrated in Figure \ref{fig: polytopes in order}, it suggests $\ReLUthree{2}{17}{4}{1}$ as a \exact on the given Swiss roll dataset, which appears to be close to optimal. However, due to the iterative nature of training multiple two-layer networks until achieving a complete polytope-basis cover, it typically requires a longer computation time. A detailed comparison with other algorithms is provided in the subsequent section.

\begin{figure}
    \centering
    \includegraphics[width=\textwidth]{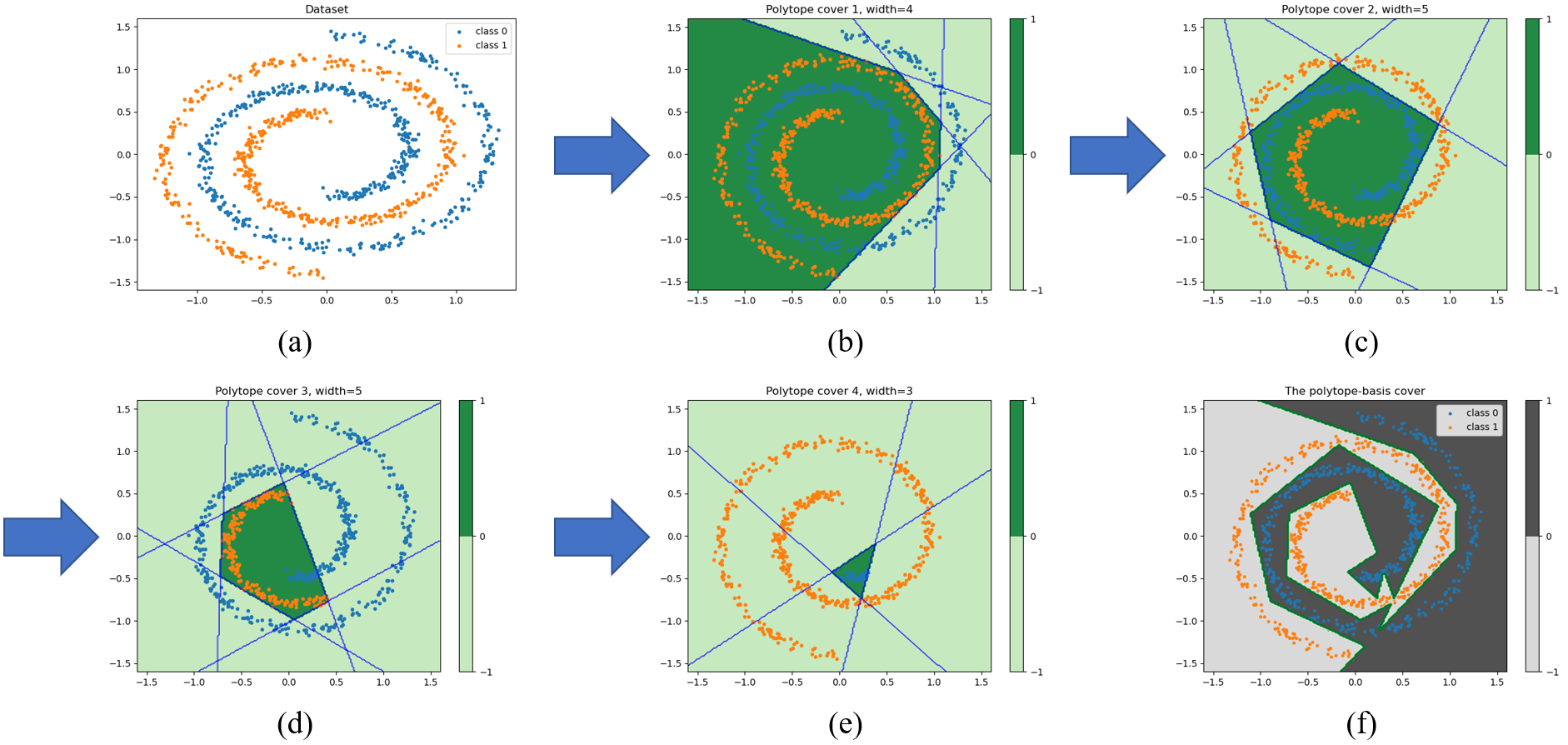}
    \caption{The result of Algorithm \ref{alg: polytopes in order}.
    For the given dataset (a), the algorithm determines the first polytope $C_1$ which contains the whole orange class, as shown in (b). In the next step, it obtains the second polytope $C_2$ which contains the whole blue class inside $C_1$. Other polytopes are similarly derived and illustrated in (c) to (e).
    Totally, the algorithm produces a polytope-basis cover consisting of four polytopes. Theorem \ref{thm: compact} shows that $\ReLUthree{2}{17}{4}{1}$ is a \exact on this dataset, and the decision boundary of the induced network is drawn in (f). 
    }
    \label{fig: polytopes in order} 
\end{figure}

\subsection{Comparison of the Proposed Algorithms} \label{app: compare algorithms} 

Finally, we compare the proposed algorithms (Algorithm \ref{alg: compressing}, \ref{alg: three-layer}, \ref{alg: two-layer}, and \ref{alg: polytopes in order}). First, we provide theoretical results for the proposed algorithms, where their proofs involve demonstrating how a polytope-basis cover can be explicitly constructed from the results of the algorithms. The detailed proof is available in Appendix \ref{app: proof 6}. 

\begin{proposition} \label{prop: algorithms}
    The following statements hold.
    \begin{enumerate}
        \item Let $\Tc$ be a network produced by repeating Algorithm \ref{alg: compressing} sufficiently many times. Then, it satisfies both \eqref{eq: two-layer constant bias} and \eqref{eq: three-layer condition}.
        
        \item Algorithm \ref{alg: three-layer} must terminate in finite time, and it produces a polytope-basis cover of the training dataset $\Dc$ that has the same accuracy with the compressed network $\Nc$. 
        
        \item Algorithm \ref{alg: two-layer} must terminate in finite time, and it produces a polytope-basis cover of the training dataset $\Dc$ that has the same accuracy with the given network $\Nc$.
        
        \item If Algorithm \ref{alg: polytopes in order} terminates in finite time, then it produces a polytope-basis cover of the training dataset $\Dc$.
        
    \end{enumerate}
\end{proposition}

\begin{table}[ht]
\centering
\resizebox{\linewidth}{!}{
\begin{tabular}{cccc}
$\vphantom{\Big|}$ & Algorithm \ref{alg: three-layer} & Algorithm \ref{alg: two-layer}  & Algorithm \ref{alg: polytopes in order}  \\ \hlineB{4}
Input network  $\vphantom{\Big|}$ & a (pretrained) three-layer \eqref{eq: three-layer} & any trained two-layer \eqref{eq: two-layer constant bias}  & - \\ \hline
\# of obtained polytopes $\vphantom{\Big|}$ & at most $J$ & generally large & generally small \\ \hline
\# of obtained faces   $\vphantom{\Big|}$  & known   & unknown  & known \\ \hline
Theoretical guarantee for termination $\vphantom{\Big|}$ & yes  & yes & no\\ \hline
Time consumption  $\vphantom{\Big|}$  & normal & short   & long \\ \hline
Accuracy of the polytope cover $\vphantom{\Big|}$ & unknown  & same with the given network  & $>acc_{th}$ (if it terminates) \\ \hline
\end{tabular}
}
\label{tab: comparison}
\caption{Comparison of proposed algorithms. All these algorithms generate a polytope-basis cover of the given training dataset, but each of them has its own pros and cons.}
\end{table}

The comparison of proposed algorithms is summarized in Table \ref{tab: comparison}. Below, we discuss differences of algorithms for each item in Table \ref{tab: comparison}.

\begin{itemize}
\item \textbf{Input network.}
Algorithm \ref{alg: three-layer} operates on a pretrained three-layer network outlined in \eqref{eq: three-layer}. Algorithm \ref{alg: two-layer} necessitates a fully-trained two-layer ReLU network specified by \eqref{eq: two layer relu}. However, Algorithm \ref{alg: polytopes in order} does not necessitate any pre-existing networks as inputs but generates polytope covers through the training of several two-layer ReLU networks as per \eqref{eq: two-layer constant bias}.

\item \textbf{The number of polytopes and faces.}
As detailed in Appendix \ref{app: two-layer}, Algorithm \ref{alg: two-layer} generally yields multiple polytopes to correctly cover all data points in the dataset. Additionally, it does not calculate the precise number of faces for each polytope. In contrast, Algorithm \ref{alg: three-layer} generates a polytope-basis cover by compressing the given three-layer network. 
Therefore, if the three-layer network $\Nc$ defined in \eqref{eq: three-layer} is the sum of $J$ two-layer networks ($\Tc_j$), the algorithm is guaranteed to produce a polytope-basis cover consisting of no more than $J$ polytopes. The compressing algorithm and Lemma \ref{lem: two-layer convex} provide the exact number of faces for each polytope.
Thirdly, Algorithm \ref{alg: polytopes in order} does not impose any specific lower or upper bounds on the number of polytopes. Similar to Algorithm \ref{alg: three-layer}, it also furnishes the exact number of faces for each polytope.

For instance, we recall the example on the swiss roll dataset: Figures \ref{fig: three-layer}, \ref{fig: general two-layer}, and \ref{fig: polytopes in order} illustrate that the algorithms yield 16, 68, and four polytopes, respectively. Also note that the last algorithms suggest a \exact of the dataset by $\ReLUthree{2}{17}{4}{1}$, which looks sufficiently minimal. 

\item \textbf{Theoretical guarantee.}
Proposition \ref{prop: algorithms} guarantees that Algorithm \ref{alg: three-layer} and \ref{alg: two-layer} must terminate in finite time. However, there is no theoretically guarantee for Algorithm \ref{alg: polytopes in order}. Even though, in all our experiments on synthetic and real-world datasets, it always terminates in finite time and produces a complete polytope-basis cover for the given dataset.

\item \textbf{Time consumption.}
Since Algorithm \ref{alg: two-layer} does not require fine-tuning process, it consumes the shortest time among these algorithms. Algorithm \ref{alg: three-layer} requires only fine-tuning process, so it takes a normal amount of time. However, Algorithm \ref{alg: polytopes in order} tends to spend relatively longer time due to its iterative process of training two-layer networks with fine-tuning until it achieves a complete polytope-basis cover. 
However, it is essential to note that even for real-world datasets like CIFAR10, the practical execution time is still quite reasonable, typically taking only a few minutes. 

\item \textbf{Accuracy of the obtained cover.}
The polytope-basis cover generated by Algorithm \ref{alg: two-layer} preserves the accuracy of the given network (Proposition \ref{prop: algorithms}). In the case of Algorithm \ref{alg: three-layer}, the fine-tuning process introduces the possibility of a different accuracy level compared to the original network. Consequently, the final accuracy cannot be determined beforehand. Regarding Algorithm \ref{alg: polytopes in order}, its accuracy is guaranteed to be greater than the given $acc_{th}$ if it terminates within a finite time.
\end{itemize}

Below, we provide additional experimental results of the proposed algorithms on several synthetic datasets, showcasing visual differences among these algorithms.
We consider three synthetic datasets: XOR, two circles, and two moons datasets, depicted in Figure \ref{fig: synthetic_datasets}. The results of algorithms are shown in Figure \ref{fig: synthetic three-layer}, \ref{fig: synthetic two-layer}, and \ref{fig: synthetic in order}.

\begin{figure}[t]
    \centering
    \includegraphics[width=0.6\columnwidth]{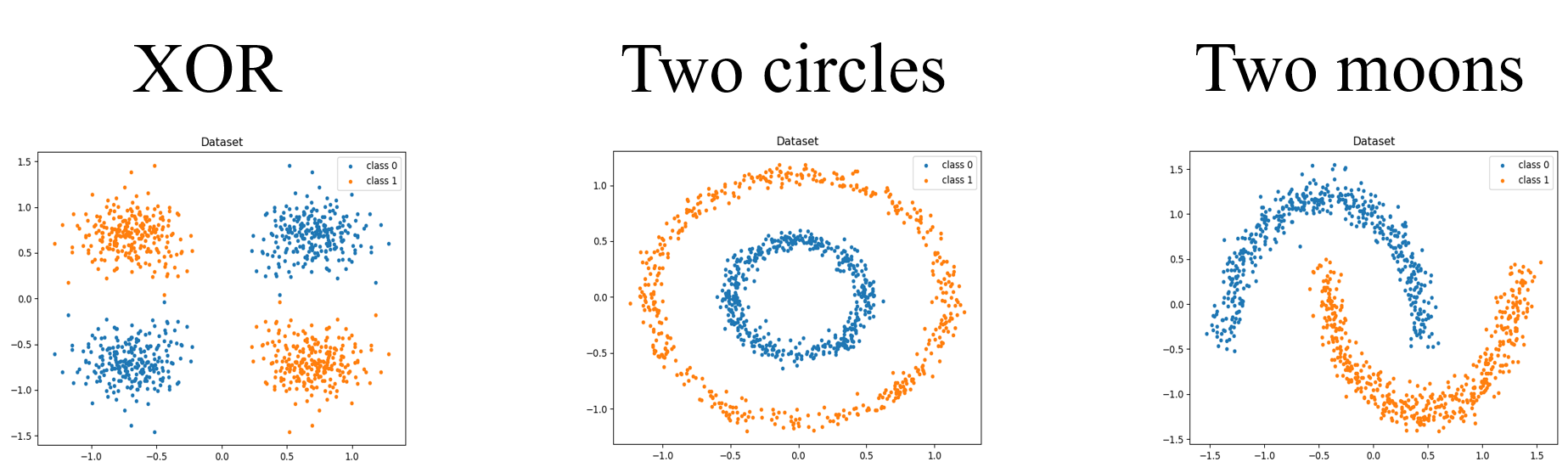}
    \caption{Synthetic datasets - XOR, two circles, and two moons.}
    \label{fig: synthetic_datasets}
\end{figure}

\begin{figure}
    \centering
    \includegraphics[width=0.8\textwidth]{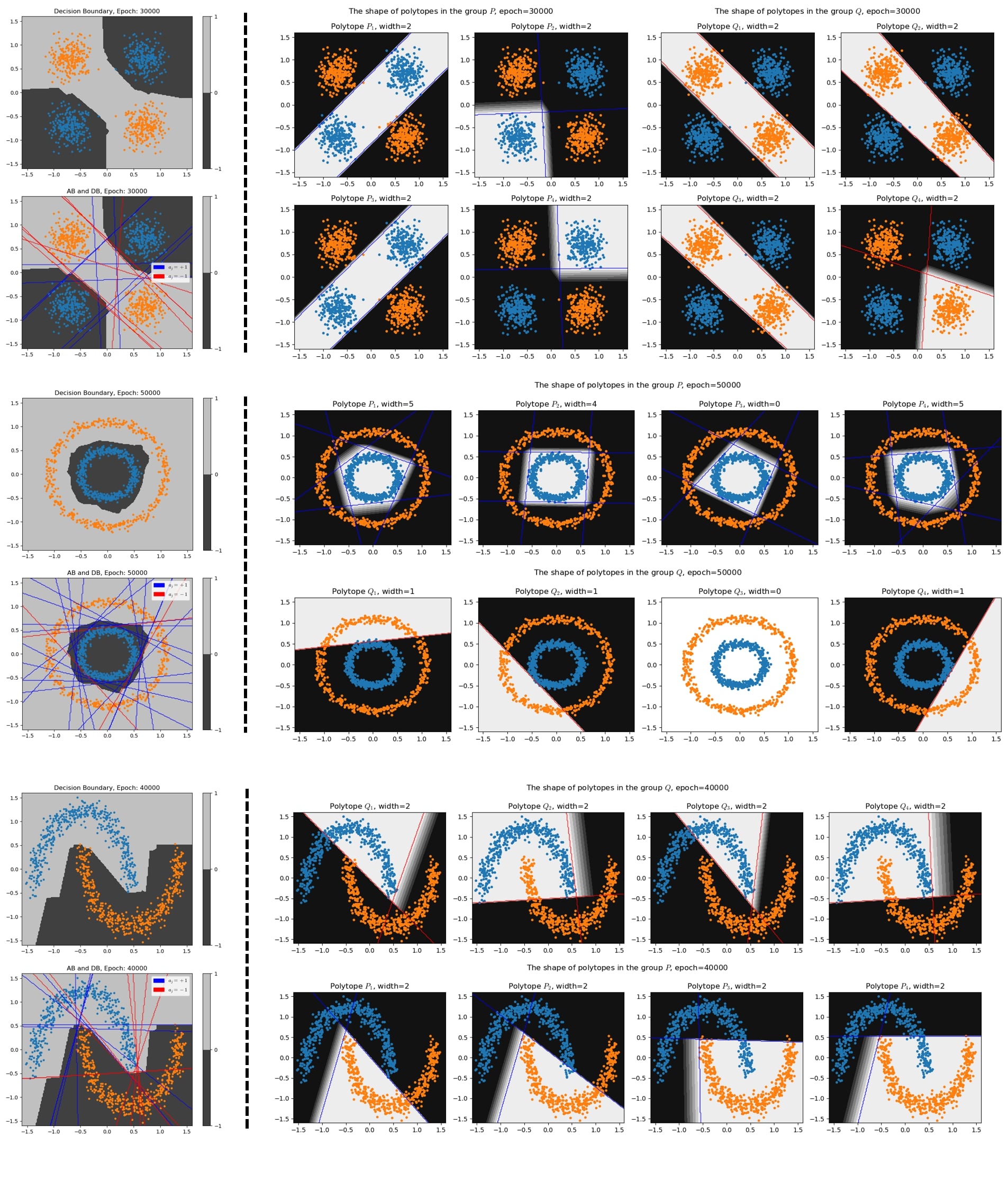}
    \caption{Visualization of Algorithm \ref{alg: three-layer} on the synthetic datasets. These polytope-basis covers are derived from trained three-layer ReLU networks \eqref{eq: three-layer} with the architecture $\ReLUthree{2}{80}{8}{1}$ (i.e., a combination of eight two-layer networks $\Tc_j$ with $\width=10$ neurons)
    }
    \label{fig: synthetic three-layer}
\end{figure}
\begin{figure}
    \centering
    \includegraphics[width=0.85\textwidth]{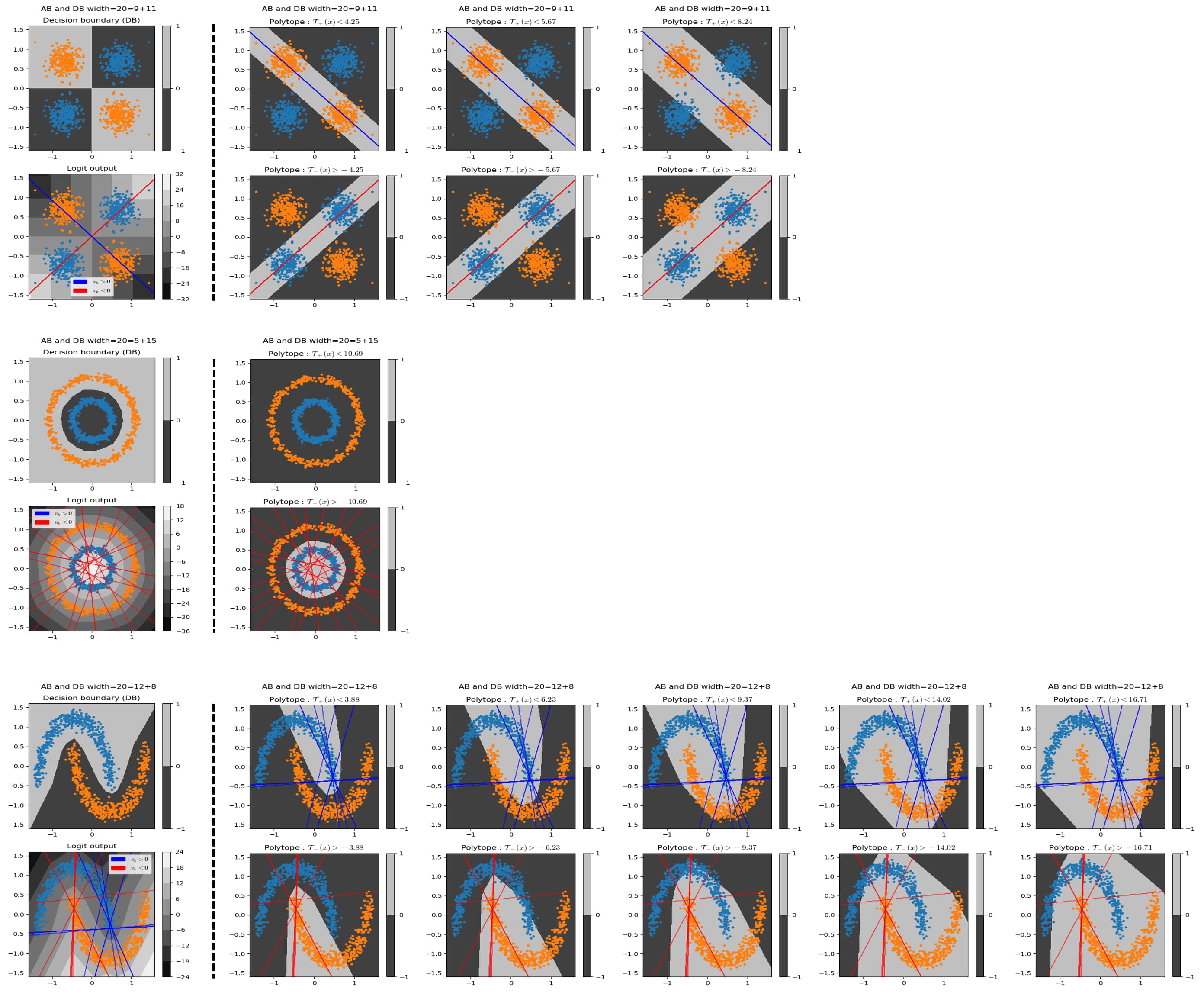}
    \caption{Visualization of Algorithm \ref{alg: two-layer} on the synthetic datasets. These polytope-basis covers are derived from trained two-layer ReLU networks with the architecture $\ReLUtwo{2}{20}{1}$.
    }
    \label{fig: synthetic two-layer}
\end{figure}
\begin{figure}
    \centering
    \includegraphics[width=0.6\textwidth]{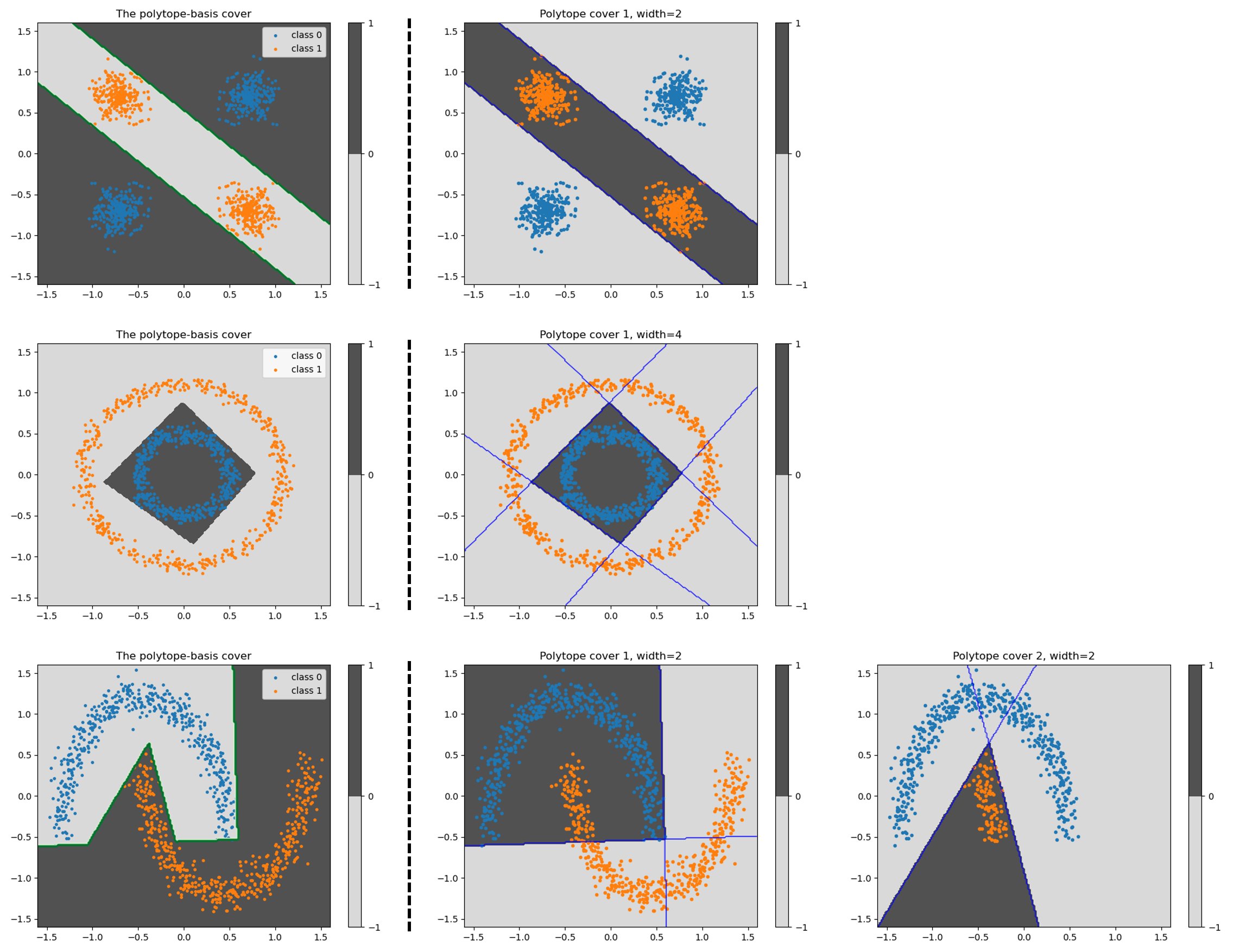}
    \caption{Visualization of Algorithm \ref{alg: polytopes in order} on the synthetic datasets. Empirically, this algorithm provides the smallest number of polytopes and their faces. The obtained polytope-basis covers can be applied to conclude the \exact of neural networks (Remark \ref{rmk: synthetic minimal}).
    }
    \label{fig: synthetic in order}
\end{figure}

It is easily checked that our proposed algorithms indeed generate polytope-basis covers of the given datasets. In each subfigure, the leftmost column represents the decision boundary and activation boundaries of the obtained networks, and the other columns represent each polytope in the obtained polytope-basis cover. 
The obtained polytope-basis covers exhibit the geometric characteristics of datasets, and provide \exacts of neural networks.

\begin{remark} \label{rmk: synthetic minimal}
    Figure \ref{fig: synthetic in order} demonstrates that 'XOR' and `two circles' datasets have single polytope covers, and `two moons' dataset can be covered by two polytopes. From the obtained polytope-basis covers, the \exacts of these datasets are given by 
    \begin{align*}
        \text{XOR : } &\quad \ReLUtwo{2}{2}{1} \\
        \text{Two circles : } &\quad \ReLUtwo{2}{4}{1} \\
        \text{Two moons : } &\quad \ReLUthree{2}{4}{2}{1}.
    \end{align*}
\end{remark}

\newpage
\section{Convergence on the Proposed Networks} \label{sec: convergence}

In this section, we investigate whether gradient descent can converge to the networks we proposed in the main text (cf. Theorem \ref{thm: compact}). Specifically, we focus on two-layer ReLU networks, which are the basic building blocks of the constructions. 
Let $\Nc$ be a two-layer ReLU neural network defined in \eqref{eq: two layer relu}, where $\Theta:=\{v_0\}\cup\{v_k, \wb_k, b_k\}_{k\in[\width]}$ denotes the set of parameters of $\Nc$. For the given dataset $\Dc=\{(\xb_i, y_i)\}_{i=1}^n$, we consider a binary classification problem under the following two loss functions: the mean squared error (MSE) loss and binary cross entropy (BCE) loss. They are defined by 
\begin{align}
    L_{MSE} (\Theta) &:= \frac{1}{2n} \sum_{i=1}^n \Big( \sigma\circ\Nc(\xb_i) - y_i \Big)^2,  \label{eq: MSE loss} \\ 
    L_{BCE} (\Theta) &:= -\frac{1}{n} \sum_{i=1}^n \ell \Big(\SIG \circ \Nc(\xb_i), \; y_i\Big)
    \label{eq: BCE loss}
\end{align}
where $\ell(\Nc,y):= \Nc y + (1-\Nc)(1-y)$. 
Note that we introduce additional activation function $\sigma$ and $\SIG$ to define both loss functions. 
Specifically, we adopt additional ReLU activation on the output layer for the existence of the zero-loss solution in \eqref{eq: MSE loss}.

We now employ the notion of `polyhedrally separable' dataset from the learning theory \citep{astorino2002polyhedral, manwani2010learning}, which is a special case of polytope-basis cover; when a given dataset can be separated by only one convex polytope as depicted in Figure \ref{fig: assumptions} (a).

\begin{definition} 
    We say that the dataset $\Dc=\{(\xb_i, y_i)\}_{i \in [n]}$ is \emph{polyhedrally separable by $C$} if there exists a convex polytope $C$ such that $\xb_i \in C$ if and only if $y_i = 1$ for all $i\in[n]$. 
\end{definition}

We further introduce two notations. First, for a convex polytope $C$ composed of $\width$ faces, we denote its $k$-th face by $\partial C_k$. 
Similarly, $\partial^2 C_k$ denotes the boundary of $\partial C_k$, which refers to  the `edge' part of $C$. 
Second, for a set $A \subset \Rd^d$, $\#(A):= |\{\xb_i\in\Dc ~|~ \xb_i \in A\}|$ denotes the number of data points $\xb_i \in \Dc$ in $A$. 
We further need the following assumptions on the dataset $\Dc$ and network initialization.

\begin{figure*}[t]
    \centering
    \hfill
    \subfigure[]{\includegraphics[width=0.3\textwidth]{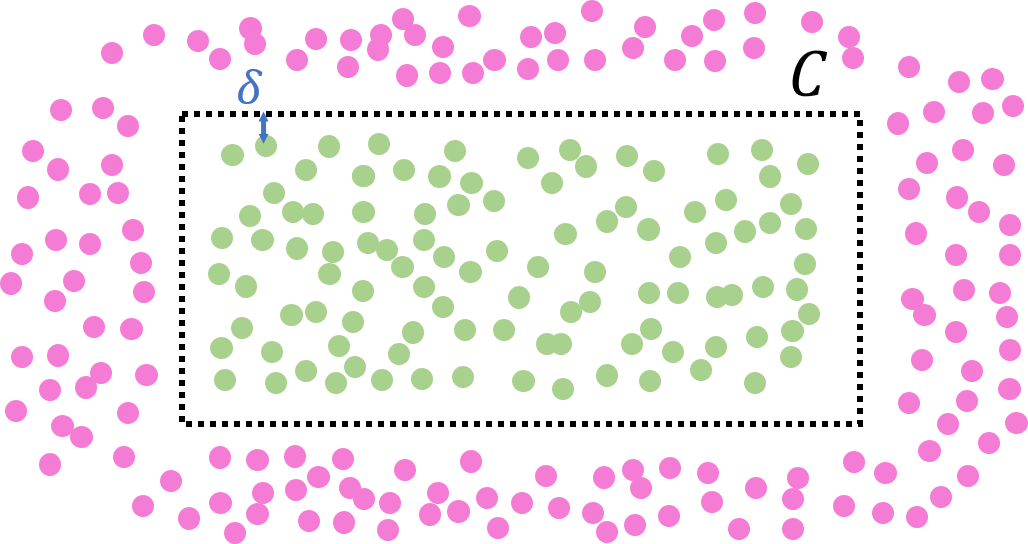}}
    \hfill
    \subfigure[]{\includegraphics[width=0.3\textwidth]{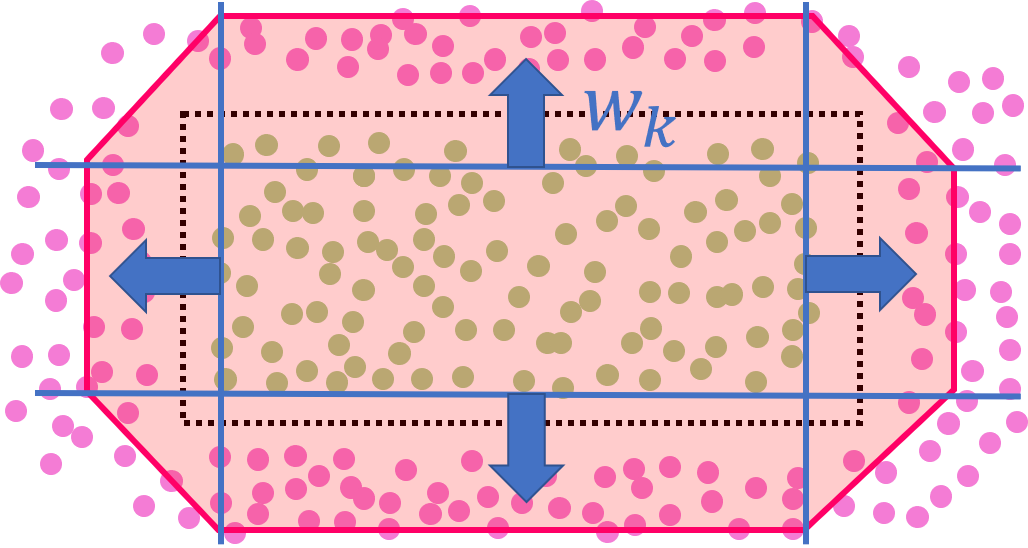}}
    \hfill
    \caption{Assumptions for the dataset and the network initialization.
    (a) Dataset $\Dc$ and a convex polytope $C$ satisfy the Assumption \ref{asmp: dataset}. 
    (b) One example of network initialization satisfying Assumption \ref{asmp: dataset}. The red line displays the decision boundary of $\Nc$.
    }
    \label{fig: assumptions}
\end{figure*}

\begin{assumption}[Dataset and initialization assumptions] \label{asmp: dataset}
    Suppose the dataset $\Dc$ is polyhedrally separable by a convex polytope $C$, which consists of $\width$ faces and strictly contains the origin point. Let $\delta>0$ be the minimum distance between $\xb_i$ and $\partial C$, and $l_k$ be the distance between $\partial C_k$ and the origin point.
    Then, there exist constants $\rho, R>0$ such that for any $k\in[\width]$ and $\delta< r <R$,
    \begin{align} \label{eq: dataset assumption}
        \#\Big(\Bc_{2r}(\partial^2 C_k)\Big) &\le \rho \; \#\Big(\Bc_{r-\delta}(\partial C_k)\Big).\vspace{2mm}
    \end{align}
    Furthermore, the parameters $\{(\wb_k,b_k,v_k)\}_{k\in[\width]}$ of a two-layer ReLU network $\Nc$ defined in \eqref{eq: two layer relu} are initialized such that $\wb_k$ are normal to $\partial C_k$ with outward direction, and satisfying
    \begin{align} \label{eq: assumption 1}
         l_k-R ~<~ l_k +\frac{v_0}{v_k \norm{\wb_k}} ~<~ -\frac{b_k}{\norm{\wb_k}} ~<~ l_k .
    \end{align}
\end{assumption}

The dataset assumption \eqref{eq: dataset assumption} suggests that the data points in the set $\Bc_r(\partial C)$ for small $r$ are predominantly located in close proximity to the faces of the polytope $C$, rather than its corners (Figure \ref{fig: assumptions}(a)). 
The network initialization assumption implies that every neuron $(\wb_k, b_k)$ of $\Nc$ is initialized near $\partial C_k$ as described in Figure \ref{fig: assumptions}(b). 
With these assumptions, we can prove the existence of a discrete path that strictly decrease the loss value to zero.

\begin{theorem} \label{thm: convergence}
    Suppose the dataset $\Dc$ and the two-layer network $\Nc$ in \eqref{eq: two layer relu} satisfy Assumption \ref{asmp: dataset}. Then, 
    \vspace{-2mm}
    \begin{enumerate}
        \item for the MSE loss defined in \eqref{eq: MSE loss}, suppose $v_0$ is initialized such that 
        \begin{align} \label{eq: v0 init MSE}
            \frac{\rho}{1-\rho} \frac{4\width\rho R^2}{\delta^2} <  v_0 < 1 .
        \end{align}
        Then, with step size $ \eta < \min \left\{ \frac{2}{\delta}, \;\frac{2}{\width R}, \;\frac{4\rho \width}{(1-\rho)R} \right\}$,
        there exists a discrete path that the loss value \eqref{eq: MSE loss} strictly decreases to zero. 

        \item For the BCE loss defined in \eqref{eq: BCE loss}, suppose $v_0$ is initialized such that
        \begin{align} \label{eq: v0 init BCE}
            0< v_0 < \log \left(\frac{(1-\rho) \delta}{4\rho R} - 1\right).
        \end{align}
        Then, with step size $\eta < \min \left\{ 1, \;\frac{4\rho R}{(1-\rho)\delta^2}\right\}$, 
        there exists a discrete path that the loss value \eqref{eq: BCE loss} strictly decreases to zero. 
    \end{enumerate}
\end{theorem}

The proof of this theorem can be found in Appendix \ref{app: proof 4}. Theorem \ref{thm: convergence} asserts that the loss landscape has no local minima within this initialization region. If local minima did exist in this region, it would contradict the presence of a loss-decreasing path from the local minima to the global minima. Consequently, this theorem provides strong evidence for the convergence of gradient descent to the global minima.

However, it is important to note that there may still be saddle points where gradient descent could potentially get stuck. In such cases, we believe that stochastic (noisy) gradient descent may help in escaping these saddle points and eventually converging to the global minimum, which has zero error on the training dataset $\Dc$. Therefore, the initialization conditions described in Assumption \ref{asmp: dataset}, \eqref{eq: v0 init MSE}, and \eqref{eq: v0 init BCE} can be understood as necessary conditions for ensuring that the gradient method converges to the global minimum.

Lastly, we mention that Theorem \ref{thm: convergence} can be easily extended to the three-layer network \eqref{eq: three-layer condition} proposed in Theorem \ref{thm: compact}. For such a three-layer network $\Nc$, Theorem \ref{thm: convergence} can be applied to each two-layer subnetwork $\Tc_j$ to generate the loss-decreasing path. 
By combining all these paths, a unified loss-decreasing path for $\Nc$ is formed. This extension underscores the robustness and generality of the convergence properties demonstrated, ensuring that even more complex network architectures retain the desirable characteristics of gradient descent convergence.

\section{Proofs} \label{app: proofs}
\subsection{Proof of Proposition \ref{prop: convex polytope}.} \label{app: proof 1}
The proof of Proposition \ref{prop: convex polytope} is divided into two parts. Firstly, we prove the upper bound by constructing the desired neural network. Secondly, we show the lower bound of widths.
\subsubsection{The upper bound in Proposition \ref{prop: convex polytope}.}

For the given convex polytope $\Xc$, let $h_1,\cdots,h_\width$ be its $\width$ hyperplanes enclosing $C$. Let $\wb_k$ be the unit normal vector of the $k$-th hyperplane $h_k$ oriented inside $C$, as illustrated in Figure \ref{fig: convex polytope}(a).
Then the equation of the $k$-th hyperplane $h_k$ is given by $h_k : \{\xb ~|~\wb_k^\top\xb+b_k=0\}$ for some $b_k\in\Rd$. Let $A_k$ be the intersection of the hyperplane $h_k$ and $C$, which is a face of the polytope $C$. 
Let $\xb$ be any point strictly contained in $C$. Since $\wb_k$ is a unit normal vector, $\wb_k^\top\xb+b_k$ refers the distance between the hyperplane $h_k$ and the point $\xb$. Therefore, the $d$-dimensional Lebesgue measure of $C$ is computed by
\begin{align} \label{eq: volume}
    \mu_d(C) = \frac{1}{d} \sum_{k=1}^\width  (\wb_k^\top\xb+b_k) \cdot\mu_{d-1}(A_k)
\end{align}
where $\mu_{d-1}$ and $\mu_{d}$ refer the $(d-1)$ and $d$-dimensional Lebesgue measures, respectively. Note that \eqref{eq: volume} comes from the volume formula of a convex polytope, which states that the volume is the sum of volume of $\width$ pyramids.
Then LHS of \eqref{eq: volume} is constant, which does not depend on the choice of $\xb\in\Rd^d$. Now, we define a two-layer ReLU network $\Tc$ with the architecture $\ReLUtwo{d}{\width}{1}$ by
\begin{align} \label{eq: polytope}
    \Tc(\xb) := 1 + M\left( \mu_d(C) -\sum_{k=1}^\width \frac{1}{d} \mu_{d-1}(A_k) \cdot \sigma(\wb_k^\top\xb + b_k) \right)
\end{align}
where $M>0$ is a constant would be determined later. 
Note that we have $\Tc(\xb) = 1$ for $\xb \in C$ from the construction. 
However, considering the negative sign, it is worth noting that the equation \eqref{eq: volume} also holds for $\xb\not\in C$. In particular, for $\xb\not\in C$, \eqref{eq: polytope} deduces
\begin{align*}
    \Tc(\xb) &= 1 + M\left(\mu_d(C) - \sum_{k=1}^\width \frac{1}{d} \mu_{d-1}(A_k) \cdot \sigma(\wb_k^\top\xb+b_k)\right) \\
    &= 1 + M\left(\mu_d(C) - \sum_{k=1}^\width \frac{1}{d} \mu_{d-1}(A_k) \cdot (\wb_k^\top\xb+b_k) 
    + \sum_{\{k~:~\wb_k^\top\xb + b_k < 0\}} \frac{1}{d} \mu_{d-1}(A_k) \cdot (\wb_k^\top\xb+b_k)\right) \\
    &= 1 + M \sum_{\{k~:~\wb_k^\top\xb + b_k < 0\}} \frac{1}{d} \mu_{d-1}(A_k) \cdot (\wb_k^\top\xb+b_k) \\
    &<1 .
\end{align*}
Therefore, we conclude that
\begin{align*}
    \Tc(\xb) &= 1 \qquad \text{if } \xb\in C, \\
    \Tc(\xb) &< 1 \qquad \text{otherwise}.
\end{align*}
Lastly, we determine the constant $M$ in $\Tc$ to satisfy the remained property. For the given $\varepsilon>0$, consider the closure of complement of the $\frac{\varepsilon}{2}$-neighborhood of $C$;  $D:=\overline{ \left(\Bc_{{\varepsilon}/{2}}(C)\right)^c}$. Then the previsous result shows that
\begin{align} \label{eq: u}
    \frac{1}{M}(\Tc(\xb) - 1) = \mu_d(C) - \sum_{k=1}^\width \frac{1}{d} \mu_{d-1}(A_k) \cdot \sigma(\wb_k^\top\xb+b_k)
\end{align}
is bounded above by $0$. Furthermore, \eqref{eq: u} is continuous piecewise linear, and has the maximum $0$ if and only if $\xb\in C$. Since $D$ is closed and \eqref{eq: u} is strictly bounded above by $0$ on $D$, \eqref{eq: u} has the finite maximum $M'<0$ on $D$.
\begin{align*}
    \frac{1}{M}(\Tc(\xb) - 1) \le M' < 0 \qquad \text{for } \xb \in D.
\end{align*}

Now, choose $M$ to satisfy $M>-\frac{1}{M'}$. Then if $\xb \not\in B_{\varepsilon}(C)$, we have $\xb \in D$, thus
\begin{align*}
    \Tc(\xb) &= 1 + M\left(\mu_d(C) - \sum_{k=1}^\width \frac{1}{d} \mu_{d-1}(A_k) \cdot \sigma(\wb_k^\top\xb+b_k)\right) \\
    &\le 1+ M \cdot M' \\
    &< 0.
\end{align*}
Therefore, we have constructed a two-layer ReLU network $\Tc$ with the structure $\ReLUtwo{d}{\width}{1}$ such that
\begin{align*}
    \Tc(\xb) &= 1 \qquad \text{if } \xb\in C, \\
    \Tc(\xb) &< 1 \qquad \text{if } \xb\in C^c, \\
    \Tc(\xb) &< 0 \qquad \text{if } \xb\not\in B_\varepsilon(C).
\end{align*}
This completes the proof on the upper bound. 
Lastly, the minimality of depth comes from the fact that a linear function cannot be a \exact on $C$.
\hfill $\square$

\begin{figure}[t]
    \centering
    \subfigure[]{\includegraphics[width=0.3\textwidth]{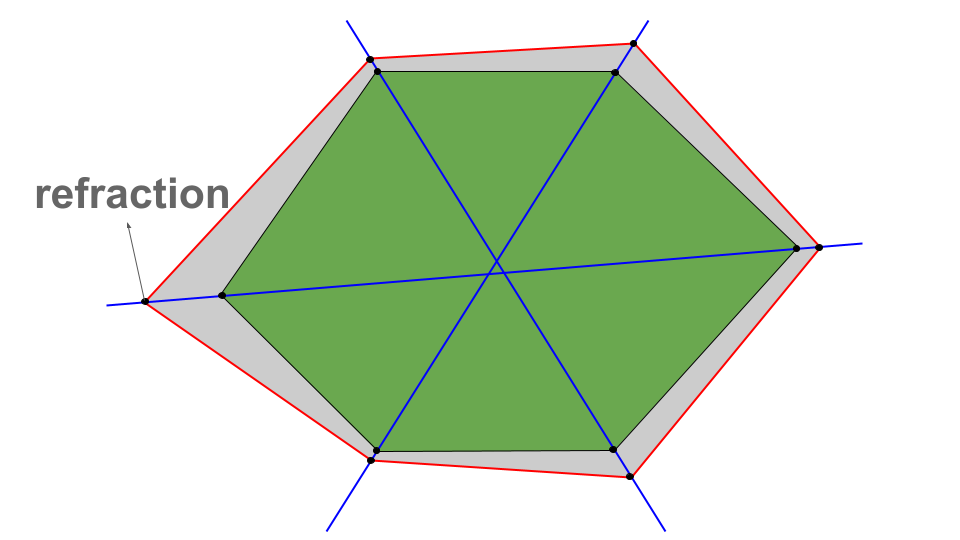}}
    \hfill
    \subfigure[]{\includegraphics[width=0.3\textwidth]{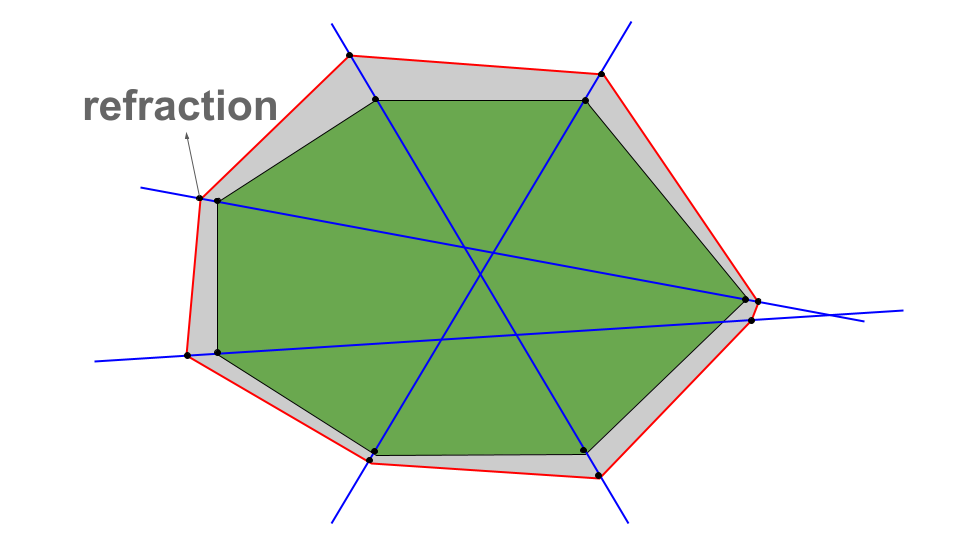}}
    \hfill
    \subfigure[]{\includegraphics[width=0.3\textwidth]{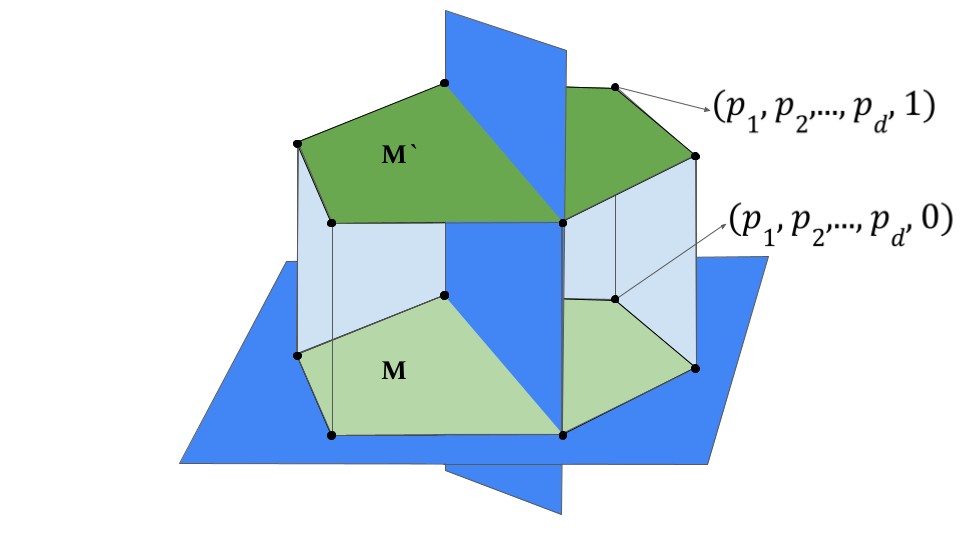}}
    \caption{Proof of Proposition \ref{prop: convex polytope}. 
    }
    \label{fig: convex polytope 2}
\end{figure}

\subsubsection{The lower bound in Proposition \ref{prop: convex polytope}.}
Before proving the lower bound, we introduce a definition on \emph{refraction points.} Let $\Nc(\xb) := \sigma(v_0 + \sum_{i=1}^{d_1} v_i\sigma(\wb_i^\top\xb +b_i))$ be a two-layer network with architecture $\ReLUTwo{d}{d_1}{1}$. Then the set of refraction point is defined by
\begin{align*}
    \{\xb\in\Rd^d ~|~ \Nc(\xb) =0 \quad \text{and} \quad \wb_i^\top\xb+b_i = 0 \quad \text{for some } i \in [l]. \}
\end{align*}
In other words, it is the point where the boundary of a linear partition is `refracted'. 

\textbf{Lower bound for $d=2$ using only refraction points for $k=1$.} Assume that we are given a convex $m$-gon to approximate. Considering the fact that we can approximate the neural network arbitrarily close, we can see that the approximated second layer, i.e. neural network should have at least $m$ refraction points in order to get the shape of polygon. However, if we look from the perspective of first layer neurons, each line has at most $2$ intersection with $m$-gon and it implies that each first layer neuron (or line) can contribute at most $2$ meaningful refraction points for the next layer. If we combine above two results, we can obtain that there should be at least $\ceil{\frac{m}{2}}$ number of neurons in the first layer, in other words $d_1 \geq \ceil{ \frac{m}{2}}$. For example, Figure \ref{fig: convex polytope 2}(a) and (b) demonstrates the refraction points along with potential first layer hyperplanes (blue lines) and converged polytope at the second layer (red polygon) for hexagon and heptagon, respectively. 

\textbf{Proof of the optimality when $d=2$ for $k=1$.} First of all, we should note that on $\Rd^2$, any two convex $m$-gon's given on the general position (i.e. assume sides are non-parallel mutually) can be approximated by the same neural network (same $d_1$ value) considering the fact that we can find an approximator for each given error value $\epsilon$. It implies that we can take optimal possible number of neurons in the first layer, which we will denote by $f(m)$ for any given convex $m$-gon. Let's prove that $f(m) = \lceil \frac{m}{2} \rceil$ for $m\geq5$ along with $f(3)=f(4)=3$. The cases $m=3, 4$ should be handled separately, because  we trivially need at least $d+1=3$ hyperplanes for any shape (Lemma \ref{lem: simplex}), so we start the base case from $m \ge 5$ for $d=2$. 

According to the Lemma \ref{lem: simplex}, it is appearent that for any $m$, the inequality $f(m)\ge 3$ should hold trivially. But if we consider the Figure \ref{fig: convex polytope 2}(a), we can observe that one can approximate any hexagon with $3$ hyperplanes. Apparently, for any pentagon, quadrilateral, and triangle, we can find a corresponding hexagon to include it as a subfigure and rest of the additional vertices of this hexagon can be shrinked to be almost non-exist. It implies that, same number of hyperplanes approximating hexagon can also approximate the polygons with $m\le 5$. This final result yields that $f(m)\le 3$ for $m\le 6$. If we combine these two findings we can get a nice optimality at the fundamental cases, in other words $f(m) = 3$ for $m\in \{3, 4, 5, 6\}$.

Now, assume the contrary that $f(m) \leq \lceil \frac{m}{2} \rceil - 1$, then it is apparent that there is at least one neuron which contributes to the refraction point of at least $2$ vertices (i.e. exactly $2$ vertices considering previous discussion). Now, if we remove the chosen neuron and the associated $2$ vertices and their edges, then the resulting $(m-2)$-gon will be approximated by $f(m)-1$ number of neurons, which implies that $f(m)-1 \geq f(m-2)$. 
Proceeding with the same argument, we can arrive at the conclusion that $f(5)$ or $f(6) \leq 2$; however, we have already proven that $f(5)$ and $f(6)$ are indeed $3$. So, the contradiction at the base case yields the result that $f(m) \geq \lceil \frac{m}{2} \rceil$. 

For the base cases $m=5, 6$, we have already demonstrated that $f(5)=f(6)=3$. Now, take any $m$-gon which has been approximated well with $f(m)=\lceil \frac{m}{2} \rceil$ neurons. Let's add two new vertices to form a new convex polygon with $(m+2)$ vertices, where the newly added vertices are not adjacent. 
Then if we add one new neuron which is the line passing through those two points, we can observe that if given $f(m)$ number of neurons approximate $m$-gon, then $f(m)+1$ can approximate $(m+2)$-gon by triggering $2$ new refraction points. This inductive argument $f(m+2) \leq f(m)+1$ yields the result that if we start from $f(5)=f(6)=3$, we can reach a conclusion that $f(m) \leq \lceil \frac{m}{2} \rceil$. However, we have already shown $f(m) \geq \lceil \frac{m}{2} \rceil$ in the proof above. Therefore, the result follows immediately that the optimal number of neurons in the first hidden layer to approximate any convex polygon with $m$ vertices is $\lceil \frac{m}{2} \rceil$ for $m\geq5$ and $f(3)=f(4)=3$.
\hfill $\square$

\textbf{Lower bound for arbitrary dimension $d$ for $k=1$.} Now we will apply simple induction on the dimensionality to prove the general case for lower bound on the number of first hidden layer neurons. Essentially, we will construct a $d$-dimensional object for $d \geq 2$ such that, one needs at least $d_1 \geq \lceil \frac{m}{2} \rceil + (d-2)$ number of neurons (hyperplanes) to approximate the convex polytope with $l$ faces. We will proceed with an inductive argument,we have already provided a proof for the base case of $d=2$ that $d_1 \geq \lceil \frac{m}{2} \rceil$. 
 
\emph{Inductive step.} Suppose that we have a $d$-dimensional convex polytope $M$ with $v$ number of vertices and $m$ number of faces such that the following inequality should hold: $d_1 \geq \lceil \frac{m}{2} \rceil + (d-2)$. Let's consider the object on $(d+1)$-dimensional space by adding new entry at the end of each coordinate, i.e. any point $(p_1, p_2, ..., p_d)$ on the object will be replaced by the point $(p_1, p_2, ..., p_d, 0)$. Then consider the new shape $M^{\prime}$ formed by considering the extension of convex polytope $M$ on $(d+1)$-dimensional space with all the points from $\{p=(p_1, p_2, ..., p_d, x)~|~\forall x=[0, 1] \text{ and } (p_1, p_2, ..., p_d) \in M\}$. 
Then $M^{\prime}$ will lie on $(d+1)$-dimensional space and it will have $2v$ vertices and $(m+2)$ number of faces, of which $m$ will be determined by the extensions of faces of polytope $M$ along with two faces from $M$ and its duplicate $M^{\prime}$. We can also observe the inductive incrementing idea through the Figure \ref{fig: convex polytope 2}(c), in which polytope $M$ at $d=2$ with $6$ faces has been extended to the $3$-dimensional polytope with $6+2=8$ faces.

If we take a closer look at this construction, we can observe that if we take the intersection of each hyperplane from $d_1$ neurons designed for the approximation of $M^{\prime}$ and polytope $M$, then those intersections will be hyperplane for $d$-dimensional polytope $M$. It implies that in order to approximate $m$ faces of new polytope, the intersections themselves should approximate the $m$ faces of $M$. Furthermore, other than those $m$ faces formed by faces of previous polytope $M$, we should also consider the other $2$ faces, namely $M$ and its duplicate $M^{\prime}$. Those two hyperplanes will require additional $2$ neurons to trigger new refraction points for their approximation. Therefore, there should be at least $d_1 \geq \lceil \frac{m}{2} \rceil + (d-2) + 2$ number of neurons, in which right-hand-side can be equivalently written as $ \lceil \frac{m}{2} \rceil + d = \lceil \frac{m+2}{2} \rceil + (d+1-2)$. So, we were able to prove that to have a neural network of the form $\ReLUTwo{d}{d_1}{1}$ to approximate the convex polytopes with $\width$ faces arbitrarily close, then universally the value of $d_1$ should at least $\lceil \frac{m}{2} \rceil + (d-2)$. 

The result can be also stated that for all $m \geq 2d+1$ one can find a $d$-dimensional convex polytope with $l$ faces such that the minimum required number neurons in the first hidden layer is at least  $\lceil \frac{m}{2} \rceil + (d-2)$. For $m=2d-1$ and $l=2d$, the lower bound becomes $d_1 \geq 2d-1$ as we have already described that $f(3)=f(4)=3$. The lower bound on $m$ comes from the fact that the construction has an inductive fashion to create a new object from previous one by adding $2$ new faces in each step. For the rest of the values of number of faces $m$, i.e. $m < 2d-1$, one can consider the trivial bound of $d+1$. More strongly, in case of $2$-dimensional space, the statement has been proven for all convex polygons that optimal value is indeed $d_1 = \lceil \frac{m}{2} \rceil$.

\textbf{Generalization to arbitrary dimension $d$ and depth $k$.} In the context of manifold representations shaped as convex polytopes with varying depths, we employ an inductive approach to establish lower bounds. Leveraging prior findings on two-layer neural networks, we derive insights applicable to arbitrary dimensions $d$. For any given hyperplane in this setting, a maximum of two distinct refraction points can be identified, a premise that underpins our assumption that each second-layer neuron constitutes a polytope comprised of faces, with no more than twice the number of hyperplanes as the first layer. This result has also been used in the proof of Theorem \ref{thm: betti numbers} and we can observe the trend from the Figure \ref{fig: lower betti}(c).

We transform the general case by considering the facets of second or higher-layer neurons as first-layer neurons (hyperplanes), which represent potential refraction points. This transformation allows us to reduce the problem to a two-layer network by decreasing the depth while augmenting the number of hyperplanes in the first layer. 
More precisely, for a given feasible architecture of the form $\ReLUFour{d}{d_1}{d_2}{\cdots}{d_k}\rightarrow1$, each of $d_2$ number of second layer neurons can contribute at most $2d_1$ hyperplanes along with the $d_1$ hyperplanes in the first layer, which implies total of $d_1+2d_1d_2 = d_1(2d_2+1)$ hyperplanes. In other words, we can transform the above network to another network $\ReLUFour{d}{d_1(2d_2+1)}{d_3}{\cdots}{d_k}\rightarrow1$ by reducing the depth by $1$. By applying the similar process as above, we assert that initial architecture can be effectively transformed into a more robust architecture, $\ReLUTwo{d}{d_1(2d_2+1)(2d_3+1)\dots(2d_k+1)}{1}$.

Consequently, we can generalize lower bounds for convex polytope representations of varying depths, drawing on the insights gained from our two-layer formulation. The ultimate result yields a powerful lower bound as 
$$
d_1 \cdot \prod_{j=2}^k (2d_j+1) \ge
\begin{cases}
    \ceil{\frac{m}{2}} + (d-2), \qquad &\text{ if }\quad  m \geq 2d+1, \\
    2d-1, \qquad &\text{ if }\quad m=2d-1, 2d, \\
    d+1, \qquad &\text{ if }\quad m<2d-1.
\end{cases}
$$
Moreover, the above result is particularly optimal for the case of convex polygons in two dimensions, where $d=2$ and $k=1$, as previously discussed.
\hfill $\square$\\

\begin{remark}
    Rigorously, the lower bound on the network width proposed in Proposition 3.1 can also be understood as the maximum number of faces that a given network can approximate with its polytope. Conversely, to achieve the UAP and approximate any polytope with $\width$ faces, the width of the first hidden layer must be greater than or equal to $\width$. This is precisely explained in Proposition \ref{prop: faces}, which proves that the upper bound proposed in Proposition 3.1 is tight and sufficient to satisfy the UAP. 
\end{remark}

\subsection{Proof of Theorem~\ref{thm: compact}}
\label{app: proof compact}
    By Proposition \ref{prop: convex polytope}, for each set $A\in \Cc=\{ P_1 ,\cdots, P_{n_P}, Q_1, \cdots, Q_{n_Q}\}$, we can construct a two-layer ReLU network $\Tc_A$ with the architecture $\ReLUTwo{d}{\width_A}{1}$ such that $\Tc_A(\xb) = 1$ for $\xb\in A$ and $\Tc_A(\xb)=0$ for $\xb\not\in B_\varepsilon(A)$, where $\width_A$ denotes the number of faces of $A$. Let $a_i := \Tc_{P_i}$ for $i\in [n_P]$ and $b_j := \Tc_{Q_j}$ for $j\in [n_Q]$. 
    Define the output layer by 
    \begin{align*}
        \Nc(\xb) = \sum_{i=1}^{n_P} a_i - \sum_{j=1}^{n_Q} b_j - \frac12.
    \end{align*}
    Then, we obtain the desired network $\Nc$ which has the architecture $\ReLUthree{d}{\width}{(n_P+n_Q)}{1}$.
       \hfill $\square$

\subsection{Proof of Theorem \ref{thm: simplicial complex}} \label{app: proof 2}
    Let $X_1, X_2, \cdots, X_k$ be the $k$ facets of $\Xc$. For each facet $X_i$, we can construct a two-layer ReLU network $\Tc_i$ such that $\Tc_i(\xb)=1$ for $\xb\in X_i$ and $\Tc_i(\xb)<0$ for $\xb\not\in \Bc_\varepsilon(X_i)$ by Lemma \ref{lem: simplex}.
    Then Theorem \ref{thm: compact} gives a neural network $\Nc$ with the architecture $\ReLUthree{d}{d_1}{k}$ with $d_1 = k(d+1)$, therefore, it is a \exact on $\Xc$.
    The remaining goal is to reduce the width of the first layer $d_1$. 
    
    From the construction, we recall that $d_1\le k(d+1)$ comes from the fact where each simplex $X_i$ is covered by a $d$-simplex which has $(d+1)$ hyperplanes. 
    Now consider two $j$-simplices in $\Rd^d$. If $2j+2 \le d+1$, then we can connect all points of the two $j$-simplices in $\Rd^d$, and it becomes a $(2j+2)$-simplex $\Delta^{2j+2}$. Now construct a $d$-simplex $\Delta^{d+1}$ by choosing $(d+1)-(2j+2)$ points in $\Bc_\varepsilon(\Delta^{2j+2})$, whose base is this $(2j+2)$-simplex. Then, by adding two distinguishing hyperplanes at last, we totally consume $(d+3)$ hyperplanes to separate two $j$-simplices. 
    
    Now we apply this argument to each pair of two simplices. The above argument shows that two $j$-simplices separately covered by $2(d+1)$ hyperplanes can be re-covered by $(d+3)$ hyperplanes if $j\le\floor{\frac{d-1}{2}}$, which reduces $(d-1)$ number of hyperplanes.
    In other words, we can save $(d-1)$ hyperplanes for each pair of two $j$-simplices whenever $j\le\floor{\frac{d-1}{2}}$. This provides one improved upper bound of $d_1$:
    \begin{align} \label{eq: a bound of d1}
        d_1 \quad\le\quad k(d+1) - (d-1)\floor{\frac{1}{2} \sum_{j=0}^{\floor{\frac{d-1}{2}}} k_j}.
    \end{align}
    
    Now, we consider another pairing. For $0\le j \le J$, $\Xc$ has $k_j$ $j$-simplex facets. Since each $j$-simplex has $(j+1)$ points, in particular, a $d$-simplex consists of $(d+1)$-points. Therefore, all points in $\floor{\frac{d+1}{j+1}}$ many $j$-simplices can be contained in one $d$-simplex. In this case, these $j$-simplices can be covered by adding $\floor{\frac{d+1}{j+1}}$ hyperplanes more. Thus if we have $k_j$ many $j$-simplices, then the required number of hyperplanes to separately encapsulate the $j$-simplices is less than or equal to
    
    \begin{align}
        \# \left(\text{the number of }d\text{-simplices}\right) &\cdot \# \left( \text{the required number of hyperplanes in each $d$-simplex} \right) \notag \\
        &=\left( \floor{\frac{k_j}{\floor{\frac{d+1}{j+1}}}}+1 \right) \cdot \left( d+1+\floor{\frac{d+1}{j+1}} \right) \notag \\
        &\le \left( k_j \frac{j+1}{d-j}+1\right) \cdot \left( d+1 + \frac{d+1}{j+1} \right) \notag \\
        &< (d+1) \left(\frac{j+2}{j+1}\right) \left( k_j \frac{j+1}{d-j}+1\right) \notag \\
        &= (d+1) \left( k_j \frac{j+2}{d-j} + \frac{j+2}{j+1} \right) \label{eq: however}
    \end{align}
    
    where the inequality is reduced from the property of the floor function: $ a-1 < \floor{a} \le a<\floor{a}+1$ for any $a\in\Rd$. Then another upper bound of $d_1$ is obtained by applying \eqref{eq: however} for all $j \le J$. However, further note that \eqref{eq: however} is greater than the known upper bound $k(d+1)$ if $J>\frac{d}{2}$ ; the sharing of covering simplex is impossible in this case. Therefore, the upper bound of $d_1$ is given by 
    \begin{align}
        d_1  \quad&\le (d+1) \sum_{j\le \frac{d}{2}} \left(  k_j \frac{j+2}{d-j} + \frac{j+2}{j+1} \right) + (d+1)\sum_{j>\frac{d}{2}} k_j \notag\\
        &=(d+1) \left[\sum_{j\le \frac{d}{2}} \left(  k_j \frac{j+2}{d-j} + \frac{j+2}{j+1} \right) + \sum_{j>\frac{d}{2}} k_j \right]
         \label{eq: another bound of d1}
    \end{align}
    
    To sum up, from \eqref{eq: a bound of d1} and \eqref{eq: another bound of d1}, we get the desired result
    \begin{align*}
        d_1 \le
        \min\left\{ k(d+1) - (d-1)\floor{\frac{1}{2} \sum_{j=0}^{\floor{\frac{d-1}{2}}} k_j} 
        ,\; 
        (d+1) \left[\sum_{j\le \frac{d}{2}} \left(  k_j \frac{j+2}{d-j} + \frac{j+2}{j+1} \right) + \sum_{j>\frac{d}{2}} k_j \right] \right\}.
    \end{align*}
    \hfill $\square$

\begin{figure*}
    \centering
    \subfigure[]{\includegraphics[width=0.3\textwidth]{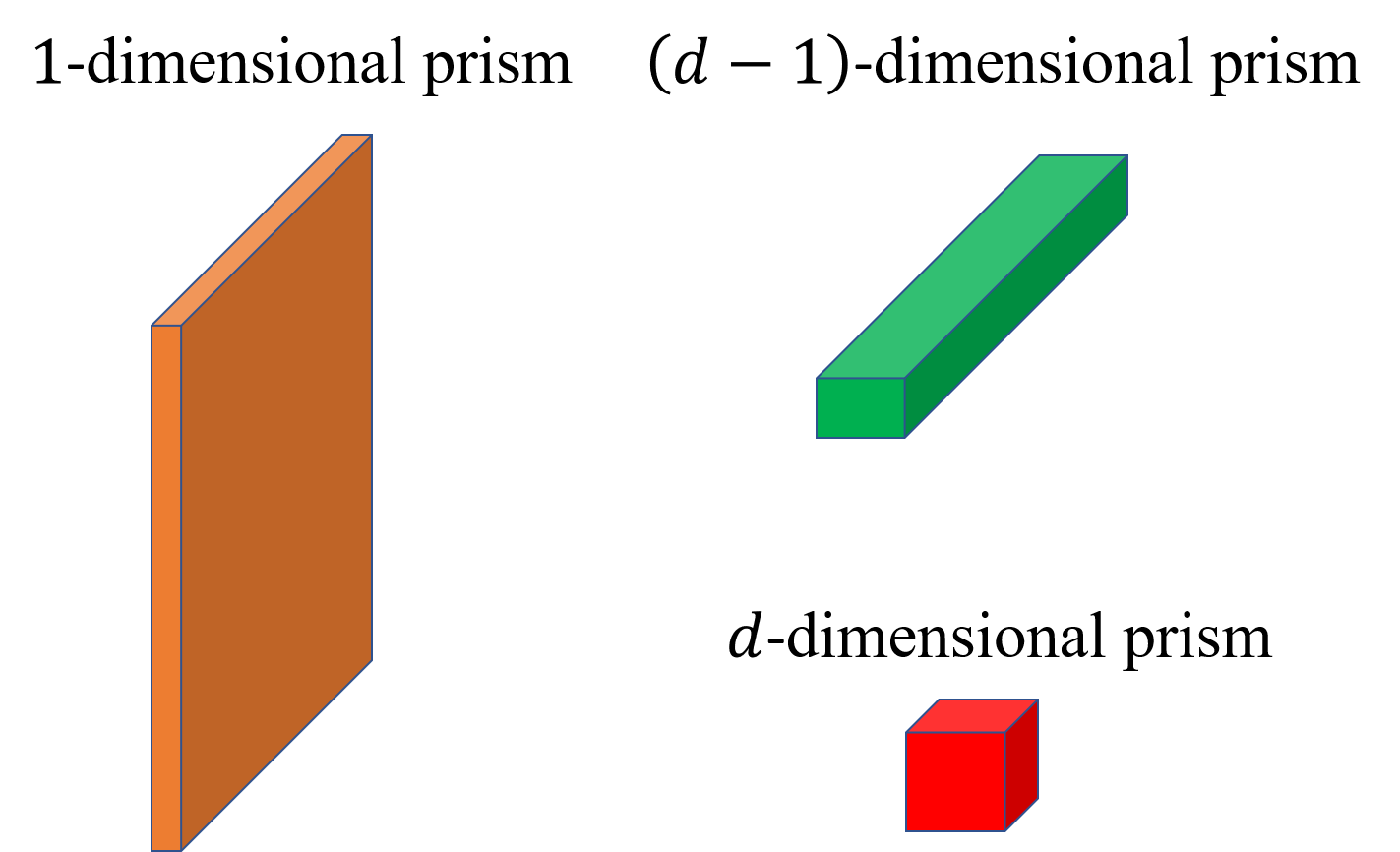}}
    \hfill
    \subfigure[]{\includegraphics[width=0.3\textwidth]{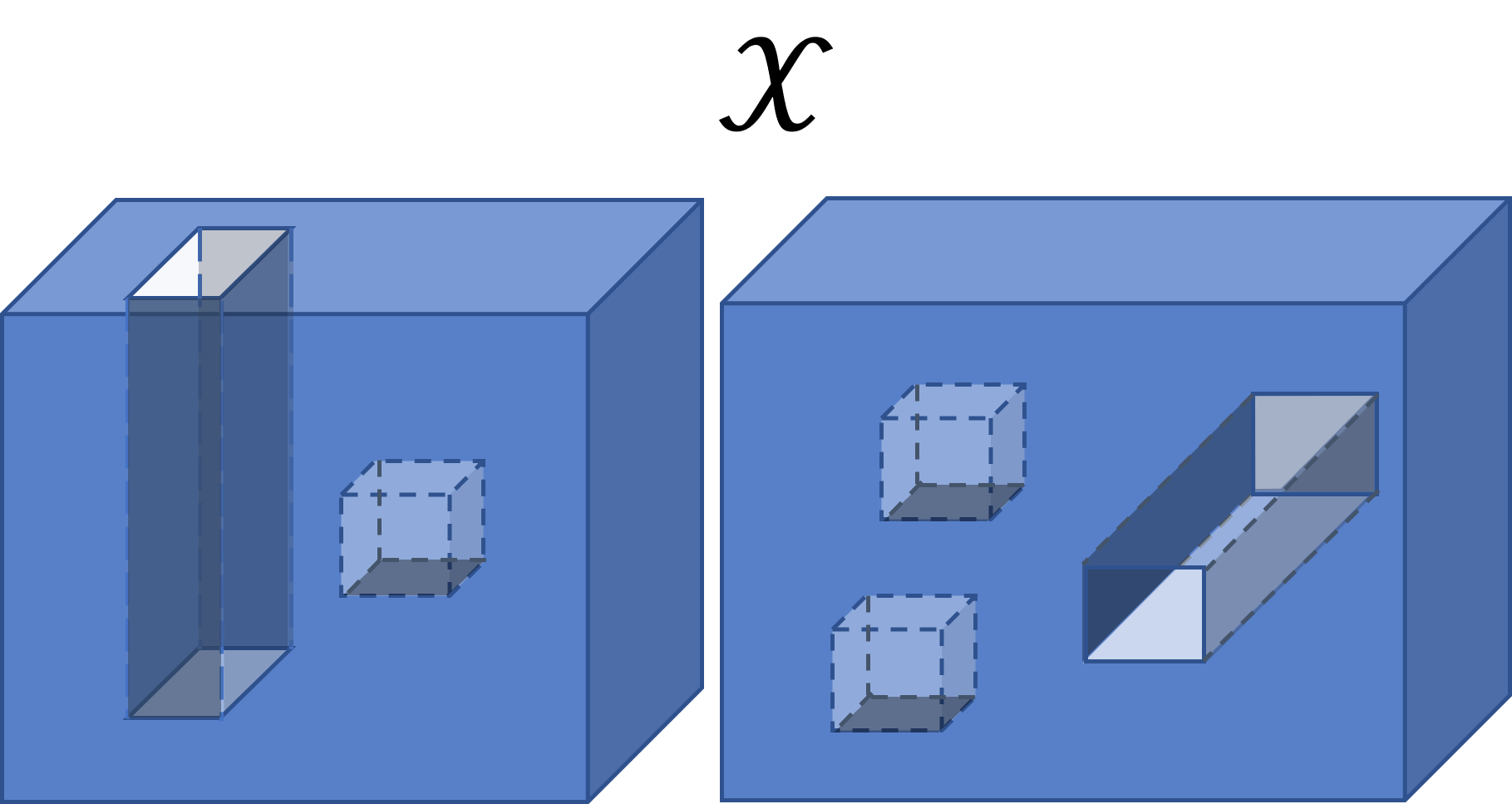}}
    \hfill
    \subfigure[]{\includegraphics[width=0.3\textwidth]{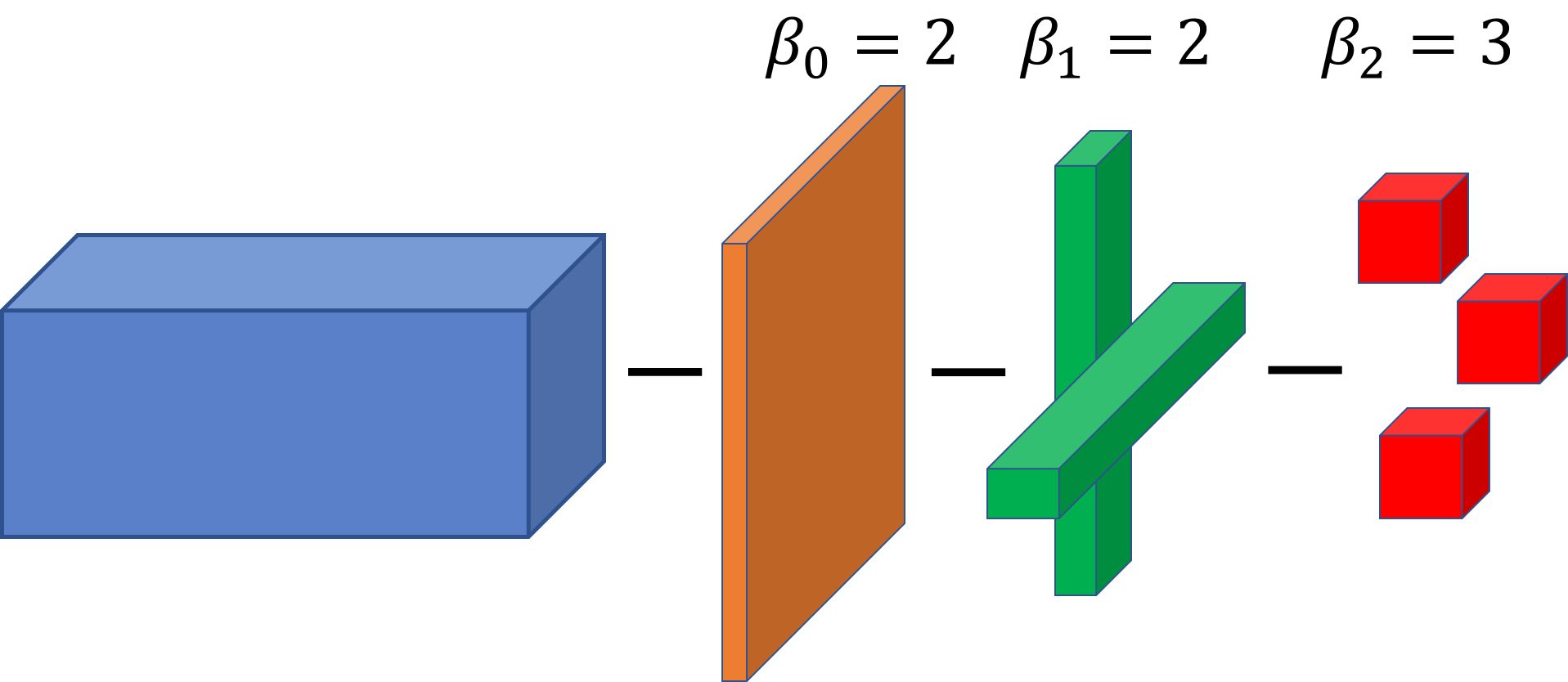}}
    \caption{The Proof of upper bounds in Theorem \ref{thm: betti numbers}.
    (a) Some examples of prismatic polytopes.
    (b) $\Xc$ is a topological space satisfying the assumption in Theorem \ref{thm: betti numbers}.
    (b) The removed prismatic polytopes from $\Xc$ are displayed.
    Theorem \ref{thm: betti numbers} demonstrates that $\ReLUthree{3}{34}{7}{1}$ is a \exact on $\Xc$. 
    }
    \label{fig: betti numbers}
\end{figure*}

\subsection{Proof of Theorem \ref{thm: betti numbers}.} \label{app: proof 3}
The proof consists of two parts: we prove the upper bound first, and second, we show the lower bound.

\subsubsection{The upper bound in Theorem \ref{thm: betti numbers}.}
We establish a terminology about the shape of prismatic polytopes. A prism in $\Rd^3$ consists of a `base' and `height' dimensions, and we generalize it to high dimensional prisms. We define a \textbf{$k$-dimensional prismatic polytope} in $\Rd^d$ as a topological space homeomorphic to $K \times \Rd^{d-k}$, where $\times$ denotes the Cartesian product an $K\subset \Rd^k$ is a compact set which is the `base' of the prism. A \textbf{bounded $k$-dimensional prismatic polytope} is a intersection of a $k$-dimensional prismatic polytope and a bounded convex polytope.

Roughly speaking, an $1$-dimensional prismatic polytope is a thick `hyperplane' in $\Rd^d$, $(d-1)$-dimensional prismatic polytope is a long 'rod,' and a $d$-dimensional prismatic polytope is just a convex polytope, as shown in Figure \ref{fig: betti numbers}(a). 
Then, for $k=1,2,\cdots,d$, removing a $k$-dimensional prism from $\Xc$ generates a $k$-dimensional hole, which increases the $(k-1)$-th Betti number $\beta_{k-1}$. 
In Figure \ref{fig: betti numbers}(b), we provide an example of such cover $\Xc$ in $\Rd^3$. $\Xc$ is described by subtracting six prismatic polytopes from a large bunoid in $\Rd^3$. In this case, the subtracted prismatic polytopes can be understood as a bounded prismatic polytopes with six faces.

Now, we prove the theorem. Recall that the polytope-basis cover $\Xc$ can be described as a subtraction of $\sum_{k=0}^{d-1}\beta_k$ convex polytopes from a sufficiently large convex polytope with $m$ faces. Applying Theorem \ref{thm: compact}, we get 
$$ \ReLUthree{d}{d_1}{\left( \sum_{k=0}^{d-1}\beta_k \right)}{1}
$$
is an upper bound of a \exact, where 
\begin{align*}
    d_1 &\le m + m \cdot \left(\sum_{k=0}^{d-1} \beta_k-1\right) \\
    &= \width\left(\sum_{k=0}^{d-1} \beta_k\right).
\end{align*}

This upper bound of $d_1$ can be further reduced.
For $1 \le k < d$, $\beta_k$ means the number of $k$-dimensional holes in $\Xc$, which was made by punching out a $k$-dimensional prismatic polytope. 
Since $k$-dimensional prisms have $2k$ faces that penetrate $\Xc$, we can reduce $2(d-k-1)$ number of hyperplanes that cover the hole. When $k=0$, it is easy to check that $2(\beta_0-1)$ hyperplanes are required to separate $\beta_0$ connected components. For instance, Figure \ref{fig: betti numbers}(c) shows this process for a topological space given in Figure \ref{fig: betti numbers}(b).
Then, the required total number of hyperplanes is bounded by
\begin{align*}
   d_1 &\le \width + 2(\beta_0-1) + \sum_{k=1}^{d-1} \left(\width-2(d-k-1)\right)\beta_k
\end{align*}
which completes the proof.
\hfill $\square$

\subsubsection{The lower bound in Theorem \ref{thm: betti numbers}.}
\begin{figure}[t]
    \centering
    \subfigure[]{\includegraphics[width=0.23\textwidth]{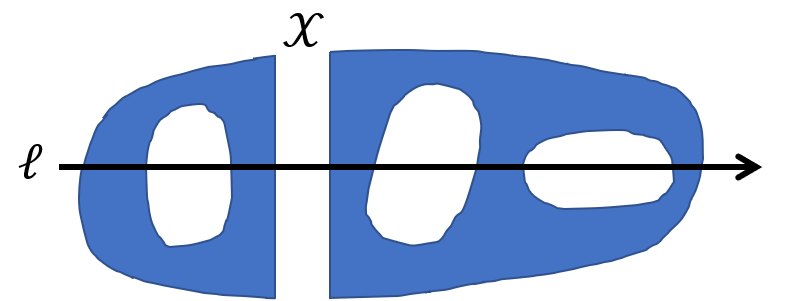}} \hfill
    \subfigure[]{\includegraphics[width=0.23\textwidth]{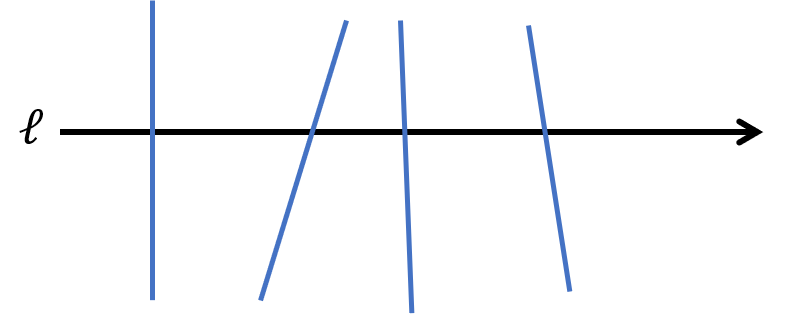}} \hfill
    \subfigure[]{\includegraphics[width=0.23\textwidth]{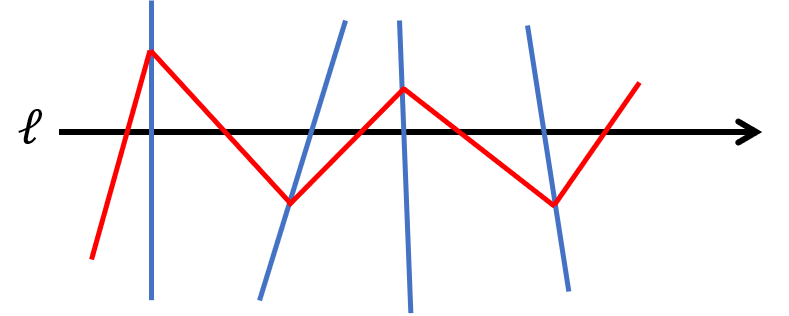}} \hfill
    \subfigure[]{\includegraphics[width=0.23\textwidth]{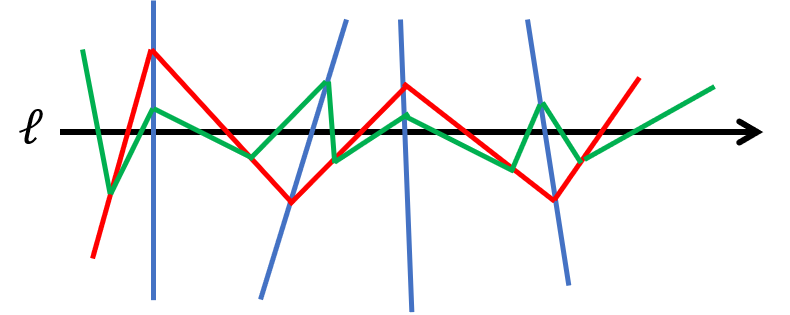}}
    \caption{Proof of lower bounds in Theorem \ref{thm: betti numbers}. 
    (a) Consider a topological space $\Xc$ whose holes intersect with a straight line $\ell$.
    (b) $d_1$ neurons in the first hidden layer of $\Nc$ (blue color) have at most $d_1$ intersection points with $\ell$.
    (c) A neuron in the second layer (red color) has at most $(d_1+1)$ intersection points with $\ell$.
    (d) Similarly, a neuron in the third layer (green color) has at most $d_2(d_1+1)+1$ intersection points with $\ell$.
    }
    \label{fig: lower betti}
\end{figure}

Suppose the given architecture $\ReLUFour{d}{d_1}{d_2}{\cdots}{d_k}\rightarrow1$ is a universally feasible architecture on any topological space $\Xc$ satisfying the assumptions stated in Theorem \ref{thm: betti numbers}. Then, it is enough to consider the `worst' case of dataset to prove a lower bound. We will use the same idea in the proof of Proposition \ref{prop: no topology}. Specifically, for the given Betti numbers $\beta_k$, we consider a topological space $\Xc$ such that every `hole' intersects with a straight line, say $\ell$. Since each hole intersects with $\ell$ at least two points, we conclude that $\Nc$ has at least $2\sum_{k=0}^d\beta_k$ piecewise linear regions on $\ell$ (Figure \ref{fig: lower betti}(a)).

Now we introduce one terminology: from the piecewise linearity of deep ReLU networks, we define a \emph{linear partition region} to be a maximum connected component where the network is affine on. Note also that the boundary of each linear partition region is non-differentiable points of $\Nc$ in $\Rd^d$, which are vanished points of some hidden layers. 

We establish the proof by computing the upper bounds of number of linear partition regions on the straight line $\ell$ made by $\Nc$. 
For $d_1$ neurons in the first hidden layer, the set of vanishing points are $d_1$ hyperplanes in $\Rd^d$, thus it can intersect with $\ell$ at most $d_1$ times (Figure \ref{fig: lower betti}(b)).
Then, consider the vanishing points of the second hidden layer. These points form a bent hyperplane in $\Rd^d$, which is refracted on the intersection with a vanishing hyperplane of the first layer (Figure \ref{fig: lower betti}(c)). Therefore, a vanishing hyperplane of the second hidden layer can intersect with $\ell$ at most $(d_1+1)$ times for each neuron. This concludes that the number of vanishing hyperplanes of the second hidden layers can intersect with $\ell$ at most $d_2(d_1+1)$ times. 
By the same argument, after the third layer, the number of maximum partitions on $\ell$ is bounded by $d_3(d_2(d_1+1)+1)+1$ (Figure \ref{fig: lower betti}(d)), and so on. Then, for the given architecture $\ReLUFour{d}{d_1}{d_2}{\cdots}{d_k}\rightarrow1$, the number of linear partition regions on $\ell$ is bounded by
\begin{align*}
    1+ &d_k + d_k d_{k-1} + d_k d_{k-1} d_{k-2} + \cdots + d_k \cdots d_1 \\
    &= 1 + \sum_{i=1}^k \prod_{j=i}^k d_j .
\end{align*}
Therefore, to be a \exact on $\Xc$, we get
$$  1 + \sum_{i=1}^k \prod_{j=i}^k d_j \ge 2\sum_{k=0}^d \beta_k -1, $$
which completes the proof.
\hfill $\square$

\subsection{Proof of Theorem~\ref{thm: three-layer polytope cover}} \label{app: proof 5}
Recall that $\Tc_j$ satisfies that $\sigma(\Tc_j(\xb_i))=0$ or $\lambdabias$ for all $\xb_i \in \Dc, j \in [J]$. 
From \eqref{eq: two-layer constant bias} and Lemma \ref{lem: two-layer convex},
we know that $C_j := \{\xb~|~\Tc_j(\xb)=\lambdabias\}$ is a convex polytope for each $j\in[J]$. Then, we get
\begin{align*}
    \Nc(\xb) &= -\frac12\lambdabias + \sum_{j=1}^J a_j \sigma(\Tc_j(\xb)) \\
    &= -\frac12\lambdabias + \sum_{j=1}^J a_j \indicator{\xb \in C_j}(\xb).
\end{align*}
Now, we define $\Cc_P := \{ C_j \in \Cc ~|~ a_j =+1\}$ and $\Cc_Q := \{ C_j \in \Cc ~|~ a_j =-1\}$. Then, we get
\begin{align*}
    \Nc(\xb_i) >0  \qquad\Longleftrightarrow\qquad 
    \sum_{C\in\Cc_P} \indicator{\xb_i \in C} > \sum_{C\in\Cc_Q} \indicator{\xb_i \in C}
\end{align*}
for all $i \in [n]$. Therefore, Definition \ref{def: convex polytope cover} establishes that $\Cc$ is a polytope-basis cover of $\Dc$, ensuring its accuracy matches that of $\Nc$.
\hfill $\square$

\subsection{Proof of Proposition \ref{prop: algorithms}} \label{app: proof 6}
Here, we present proofs for each statement. 
\begin{enumerate}        
    \item 
    Firstly, we establish the validity of \eqref{eq: two-layer constant bias}, ensuring $v_{jk}<0$. This condition holds at initialization as the network adheres to \eqref{eq: initialization}. Throughout the algorithm, consisting of neuron removal and neuron scaling, neither action compromises \eqref{eq: two-layer constant bias}. Proposition \ref{prop: balanced property} assures the persistence of \eqref{eq: initialization} under gradient flow. Consequently, \eqref{eq: two-layer constant bias} remains satisfied throughout.
    
    Secondly, we scrutinize the condition \eqref{eq: three-layer condition}. Suppose there exists $\xb_i \in \Dc$ such that $0<\Tc(\xb_i)<\lambdabias$. From Definition \ref{eq: two-layer constant bias}, it implies that 
    \begin{align*}
        -\lambdabias < \Tc(\xb_i)-\lambdabias = \sum_{k\in [m]} v_k\sigma(\wb_k^\top\xb_i+b_k) < 0
    \end{align*}
    Recall that all $v_k$ are negative, by the preceded proof, and ReLU is positive homogeneous. Therefore, by scaling neurons $(v_k, \wb_k, b_k)$ by $(\lambda_{scale} v_k, \lambda_{scale} \wb_k, \lambda_{scale} b_k)$ such that $\wb_k^\top\xb_i+b_k >0$, the network output changes from
    $$ \Tc(\xb_i)-\lambdabias  \qquad \rightarrow\qquad \lambda_{scale}^2 (\Tc(\xb_i)-\lambdabias). $$
    Therefore, by repeating this scaling sufficiently many times, given $\lambda_{scale} > 1$, $(\Tc(\xb_i)-\lambdabias)$ decreases under $-1$. 
    This process applies to all such $\xb_i$ in the dataset, eventually leading to $\sigma(\Tc(\xb_i))=0$ or $\lambdabias$ for all $\xb_i\in\Dc$.
    
    \item 
    In Algorithm \ref{alg: three-layer}, first for loop must terminate after $Epochs$ repetition. Then, the following repeat loop makes all $\Tc_j$ to satisfy $\sigma(\Tc_j(\xb_i))$ is either 0 or 1, for all $\xb_i \in \Dc$. However, by the previous proof, we know Algorithm \ref{alg: compressing} provides a network satisfying both \eqref{eq: two-layer constant bias} and \eqref{eq: three-layer condition} in finite time. Therefore, this algorithm is guaranteed to terminate in finite time. 
    Lastly, since each $\Tc_j$ has binary output $0$ or 1, convex polytopes defined by $C_j:=\{\Tc_j>0\}$ forms a polytope-basis cover, which has the same accuracy with the produced network $\Nc$. 

    \item 
    To prove finite-time termination of Algorithm \ref{alg: two-layer}, it is enough to show that the \textbf{repeat} loop in the algorithm must terminate in finite time. Specifically, we prove the following two statements: 1. the incorrectly covered data $\hat\xb$ is correctly covered after one process in the \textbf{repeat} loop by adding two polytopes, and 2. these added polytopes do not hurt other correctly covered data.

    First, let $\Cc$ be a (constructing) polytope-basis cover and let $\hat\xb$ be an incorrectly covered data by $\Cc$. Then, it means the sign of $\Nc_+(\hat\xb) - \Nc_-(\hat\xb)$ and the sign of $\hat o_i := \sum_{c\in\Cc_P} \indicator{\hat\xb\in C} - \sum_{c\in\Cc_Q} \indicator{\hat\xb\in C} + \frac12$. Let $c$ be the value between $\Nc_-(\hat\xb)$ and $\Nc_+(\hat\xb)$, and define two convex polytopes
    \begin{align*}
        C_+ &:= \{\xb~|~\Nc_+(\xb)<c\} \\
        C_- &:= \{\xb~|~\Nc_-(\xb)>-c\}.
    \end{align*}
    Then, by the definition, $\hat\xb$ is only contained in either $C_+$ or $C_-$, determined by the sign of $\Nc(\hat\xb)$. Therefore, adding these two polytopes to the polytope-basis cover $\Cc$, by $C_- \in \Cc_P$ and $C_+ \in \Cc_Q$, $\hat\xb$ is now correctly classified by the cover $\Cc$.

    Second, we claim that adding above two polytopes do not disrupt other correctly covered data. Recall the approximation of convex functions discussed in Appendix \ref{app: two-layer}. Let $f:\Rd^d \rightarrow \Rd$ be a convex function, and let $P:=\{p_0, p_1, \cdots, p_{J+1}\}$ be a finite partition of real number by
    $$ -\infty = p_0 \le p_1 \le p_2 \le \cdots \le p_J \le p_{J+1}=+\infty
    $$
    Then, $f$ can be approximated by
    \begin{align*}
        f(\xb) \approx p_1 + \sum_{j=1}^J (p_{j+1} - p_j) \indicator{f(\xb)<p_j}.
    \end{align*}
    This approximation can be understood as quantization of the function $f$ by values in $P$. Then, elementary analysis \citep{rudin1976principles} shows that refinement of the partition $P$ only increases the accuracy of the above approximation. I.e., as polytopes added in the constructing polytope-basis cover $\Cc$, the number of incorrectly covered data by $\Cc$ strictly decreases. Since there is finite data points in the training dataset $\Dc$, Algorithm \ref{alg: two-layer} must terminate in finite time. More precisely, the \textbf{repeat} loop in the algorithm must be halted in $n=|\Dc|$ times. 
    
    \item 
    Let $\Cc=\{C_1, C_2, C_3, \cdots, C_J\}$ be the output of Algorithm \ref{alg: polytopes in order}. Then, from its construction described in the algorithm, it implies that
    \begin{center}
    \begin{minipage}[t]{0.7\textwidth} 
        $C_1$ contains all points in $\Dc_0$. \\
        $C_2$ contains all points in $\Dc_1 \cap C_1$. \\
        $C_3$ contains all points in $\Dc_0 \cap C_2$. \\
        \-\ \hspace{2cm} $\vdots$ \\
        $C_{J-1}$ contains all points in $\Dc_{\frac{1+(-1)^{J-1}}{2}} \cap C_{J-2}$. \\
        $C_J$ contains all points in $\Dc_\frac{1+(-1)^J}{2}\cap C_J$, and does not contain the another class.
    \end{minipage}
    \end{center}
    Now, define 
    \begin{align}
        \Nc(\xb) := -\frac12 + \sum_{j=1}^J (-1)^j \sigma(\Tc(\xb)).
    \end{align}
    Then, $\Nc$ is a three-layer ReLU network of the form \eqref{eq: three-layer}. Furthermore, $\Tc_j(\xb_i)$ is either 0 or 1 for all $i\in [n]$ and $j \in [J]$, satisfying the condition \eqref{eq: three-layer condition}. Therefore, Theorem \ref{thm: three-layer polytope cover} verifies that $\Cc := \{C_j\}_{j\in [J]}$ becomes a polytope-basis cover of the dataset. 
\end{enumerate}
\vspace{-0.5cm}\hfill $\square$

\subsection{Proof of Theorem \ref{thm: convergence}.} \label{app: proof 4}
\subsubsection{Proof for the MSE loss \texorpdfstring{\eqref{eq: MSE loss}}{function}.}
The proof is divided into several steps. 
First, for $k\in[\width]$, we define the following sets:
\begin{align}
    A_k &:= \{ \xb\in\Rd^d ~|~ \wb_k^\top\xb+b_k>0 \} \label{eq: ak} \\ 
    B_k &:= \{ \xb\in\Rd^d ~|~ \wb_k^\top\xb+b_k>0 \text{ and }\wb_j^\top\xb+b_j>0 \text{ for } j \neq k \} \label{eq: bk}
\end{align}
I.e., $A_k$ is the region where $k$-th neuron is alive, and $B_k$ is the region where only $k$-th neuron is alive (see Figure \ref{fig: MSE loss convergence dynamics}(b,c)). Similarly, we define 
$$ A_0 :=\{ \xb\in\Rd^d ~|~ \wb_k^\top\xb+b_k<0 \} 
$$
which is the region where all neurons are dead, except the last bias term $v_0$. 
Now, we define the following values for every $k\in[\width]$:
\begin{align}
    l_k &:= \text{the distance between $O$ and }\partial C_k, \nonumber \\
    s_k &:= - \frac{b_k}{\norm{\wb_k}},     \label{eq: sk} \\ 
    t_k &:= - \frac{v_0}{v_k\norm{\wb_k}},  \label{eq: tk} \\
    t &:=\max_{k\in[\width]} \{t_k, \delta \}. \nonumber
\end{align}
Then, the network initialization condition \eqref{eq: assumption 1} gives
\begin{align*}
    0 &< t_k < R, \\
    0 &< s_k < l_k < s_k+t_k.
\end{align*}
In other words, $s_k$ is the distance between the origin point $O$ and the hyperplane $\{\wb_k^\top \xb+b_k =0 \}$. $t_k$ is the length of `height' of the region $B_k$ as depicted in Figure \ref{fig: MSE loss convergence dynamics}(c). To be familiar for these notations, we demonstrate the output $\Nc$ in Figure \ref{fig: MSE loss convergence dynamics}(d) with respect to $\norm{\wb_k}$.

It is clear that $\Nc(\xb)=v_0$ if $\xb \in A_0$, and it linearly decreases to zero for $\xb \in B_k$. When $\xb_i \in B_k$ satisfies $\xb_i^\top \frac{\wb_k}{\norm{\wb_k}} = s_k+t_k$, $\Nc(\xb_i)=0$.
Now we are ready to prove the theorem.

\begin{figure}[t]
    \centering
    \subfigure[]{\includegraphics[width=0.2\textwidth]{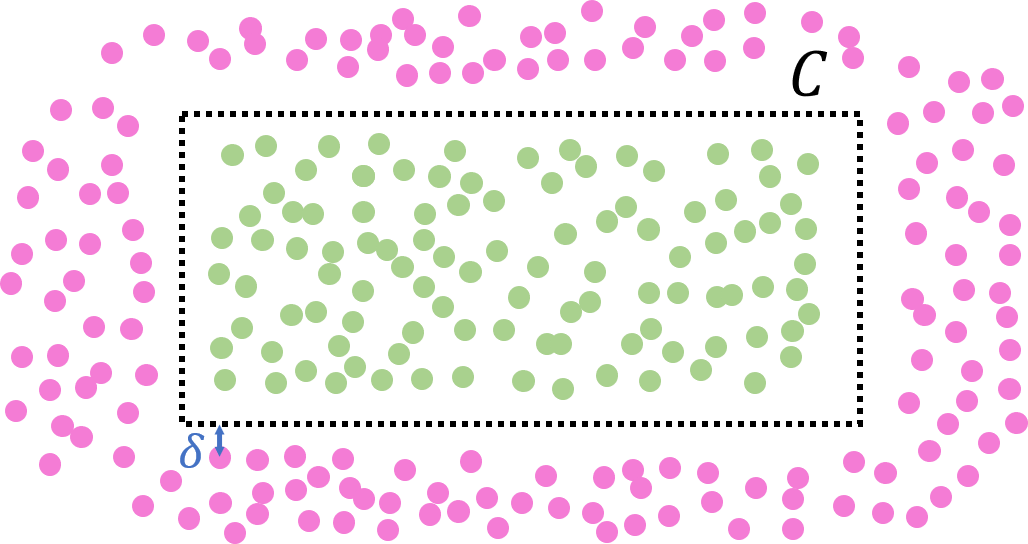}}
    \hfill
    \subfigure[]{\includegraphics[width=0.2\textwidth]{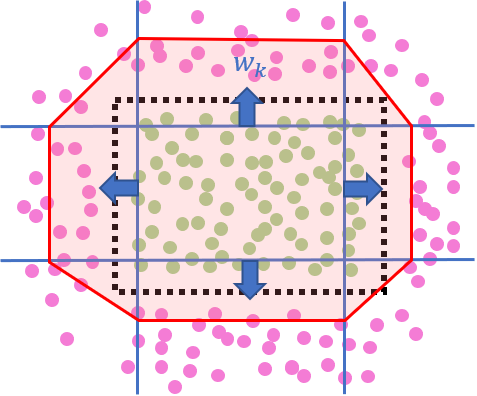}}
    \hfill
    \subfigure[]{\includegraphics[width=0.25\textwidth]{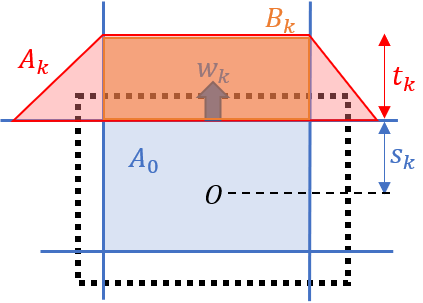}}
    \hfill
    \subfigure[]{\includegraphics[width=0.28\textwidth]{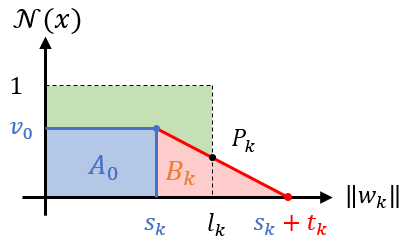}}    
    \caption{Proof of Theorem \ref{thm: convergence}. 
    (a) The given dataset $\Dc$ is polyhedrally separable by a black dashed rectangle $C$.
    (b) Initialization of a two-layer ReLU network $\Nc$.
    (c) For $k\in[\width]$, sets $A_k$ and $B_k$ defined in \eqref{eq: ak} and \eqref{eq: bk} are illustrated. 
    (d) A sideview of the function $\Nc$ with respect to $\norm{\wb_k}$. $s_k$ and $t_k$ are defined in \eqref{eq: sk} and \eqref{eq: tk}. Note that the intersection point $P_k$ is invariant after the update of parameters.
    }
    \label{fig: MSE loss convergence dynamics}
\end{figure}

For the previously defined sets $A_k$ and $B_k$, the MSE loss \eqref{eq: MSE loss} is computed by 
\begin{align}
    L_{MSE} &= \frac{1}{2n}\sum_{i=1}^n (\Nc(\xb_i)-y_i)^2 \nonumber \\
    &= \frac{1}{2n}\sum_{\xb_i \in A_0} (\Nc(\xb_i)-y_i)^2 + \frac{1}{2n}\sum_{\xb_i \in \cup_k B_k} (\Nc(\xb_i)-y_i)^2 + \frac{1}{2n}\sum_{\xb_i \in \cup_k (A_k\backslash B_k)} (\Nc(\xb_i)-y_i)^2 \nonumber \\
    &=: L_1 + L_2 + L_3. \label{eq: L123}
\end{align}
Note that we omitted $\Theta$ notation, the set of all learnable parameters.
We will observe the change of these loss values with respect to one update of parameters. We add prime ($'$) for the updated parameter. 
For the given step size $\eta$, we explicitly provide the update of parameters by
\begin{align*}
    v_0 \quad&\rightarrow\quad v_0' :=v_0 + \Delta v_0, \\
    s_k \quad&\rightarrow\quad s_k' := s_k + \Delta s_k, \\
    t_k \quad&\rightarrow\quad t_k' := t_k + \Delta t_k 
\end{align*}
for all $k\in[\width]$, where 
\begin{align}
    \Delta v_0 &:= 
    \begin{cases}
        \qquad 0 \qquad &\text{if }\quad  \#(\cup_{k \in [l]} A_k) > 0, \\
        -\frac12 (v_0-1)t\eta, \qquad & \text{otherwise},
    \end{cases} 
    \label{eq: dv0} \\
    \Delta s_k &:= 
    \begin{cases}
        \frac{\eta t_k^2}{v_0 + \eta t_k} \cdot \frac{l_k-s_k}{l_k+t_k-s_k}
        \qquad &\text{if }\quad  \#(A_k) > 0, \\
        \qquad 0 &\text{otherwise}, 
    \end{cases}
    \label{eq: dsk} \\
    \Delta t_k &:=
    \begin{cases}
        \Delta s_k - \frac{\eta t_k^2}{v_0 + \eta t_k}
        \qquad &\text{if }\quad  \#(A_k) > 0, \\
        \frac{t_k \Delta v_0 - \eta t_k^2}{v_0 + \eta t_k} &\text{otherwise}.
    \end{cases}
     \label{eq: dtk}
\end{align}
Specifically, $v_0$ is updated if and only if $\cup_{k\in[l]}A_k$ contains a data point, where $s_k$ and $t_k$ are updated exclusively.
The given update terms are proposed to have some invariant quantity. In Figure \ref{fig: MSE loss convergence dynamics}(d), we set $P_k$ to be the output value of $\Nc$ at $l_k$, and update equations in \eqref{eq: dv0}$\sim$\eqref{eq: dtk} are determined to keep this value $P_k$. 
Furthermore, it satisfies that the change of the slope is exactly $-\eta$, i.e.,
\begin{align*}
    \Delta\left(-\frac{v_0}{t_k}\right) &:= -\frac{v_0+\Delta v_0}{t_k + \Delta t_k - \Delta s_k} + \frac{v_0}{t_k}   \\
    &= -\eta.
\end{align*}
Note also that $v_0<1$ and $s_k<l_k$ are increasing, where $t_k>0$ is decreasing.

In the subsequent steps, we examine the change of each loss value. 
The main idea of the proof is computing lower bounds on the reduction of the loss value resulting from one-step update given by \eqref{eq: dv0}$\sim$\eqref{eq: dtk}. 
It is divided into four steps.

\textbf{STEP 1.} First, we consider when $\#(\cup_{k \in [l]} A_k)=0$. In this case, since $L_2=L_3=0$ from the definition \eqref{eq: L123}, it is enough to investigate the change of $L_1$. Recall that
$$ L_1 := \frac{1}{2n}\#(A_0) (v_0-1)^2 . $$
By one-step update of parameters, it becomes $L_1 ~ \rightarrow ~  L_1' := L_1 + \Delta L_1$. Then, 
\begin{align*} 
    \Delta L_1 &= L_1' - L_1 \\ 
    &= \frac{1}{2n} \sum_{\xb_i \in A_0} (v_0+\Delta v_0 -1)^2 - \frac{1}{2n} \sum_{\xb_i \in A_0} (v_0 - 1)^2 \\ 
    &= \frac{1}{2n} \cdot (2v_0-2+\Delta v_0) \Delta v_0 \cdot \#(A_0) \\
    &= -\frac{\#(A_0)}{n} (1-v_0)\Delta v_0 +\frac{\#(A_0)}{2n} (\Delta v_0)^2 \\
    &= -\frac{\#(A_0)}{2n}(1-v_0)^2(t\eta - \frac14 t^2\eta^2) \\
    &< -\frac{\#(A_0)}{2n}(1-v_0)^2 \cdot \frac12 t \eta.
\end{align*}
Note that we use $\eta < \frac{2}{t} < \frac{2}{\delta}$ on the last inequality.
Then, we get
\begin{align}
    L_1' &= L_1 + \Delta L_1 \nonumber \\
    &= \left(1 + \frac{\Delta L_1}{L_1}\right)L_1 \nonumber \\
    &< \left( 1 - \frac{\frac{1}{2n} \#(A_0)(1-v_0)^2 \cdot \frac12 t \eta}{\frac{1}{2n}\#(A_0) (1-v_0)^2}\right) L_1 \nonumber \\ 
    &= \left(1-\frac12 t\eta\right) L_1 \nonumber \\
    &\le \left(1 - \frac12 \delta\eta\right) L_1 \label{eq: the first loss} 
\end{align}
which states that $L_1$ strictly decreases after the update.

\textbf{STEP 2.} 
Now, we consider when $\#(\cup_{k \in [l]} A_k)>0$.
We investigate the second term in \eqref{eq: L123}, defined by
$$ L_2 := \frac{1}{2n}\sum_{k\in[l]} \sum_{\xb_i \in B_k} (\Nc(\xb_i)-y_i)^2 .
$$
Recall that the update of parameters given in \eqref{eq: dv0}$\sim$\eqref{eq: dtk} are chosen to keep $P_k$ value and increasing the absolute value of the slope $-\frac{v_0}{t_k}$ by $\eta$. Therefore, for any $\xb_i \in B_k$, it $\Nc(\xb_i)$ increases (or decreases) if and only if $\xb_i \in B_k \cap C$ (or $\xb_i \in B_k \backslash C$). Therefore, $|\Nc(\xb_i) - y_i|$ always strictly decreases after the update, which implies that 
\begin{align}  \label{eq: the second loss}
    \Delta L_2 := L_2' - L_2 <0 .
\end{align}

\textbf{STEP 3.} 
We observe the last term in \eqref{eq: L123} when $\#(\cup_{k \in [\width]} A_k)>0$, which is the most technical part in this proof. Recall that
$$ L_3 := \frac{1}{2n} \sum_{\xb_i \in \cup_k (A_k \backslash B_k)} (\Nc(\xb_i) - y_i)^2. $$
The goal of this step is showing that the absolute change of $L_3$ is less than it of $L_2$, i.e., $|\Delta L_3| < \Delta L_2$. The idea is based on the sparsity of the data distribution in $\Bc_r (\partial^2 C)$; near the neighborhood of `edge' parts of the polytope $C$. 

Note that for each $k\in [\width]$, obviously we have $(A_k \backslash B_k ) \subset \Bc_t (\partial^2 C_k)$ from the linearity of $\Nc$ (see Figure \ref{fig: MSE loss convergence dynamics}(c) and (d)). It is also worth noting that if $t_k\le \delta$, then $L_3=0$ because there is no $\xb_i$ in $\cup_{k\in[\width]} (A_k\backslash B_k)$, and we have nothing to do. Thus we mostly consider $t_k>\delta$ cases.

Let $\Nc'$ be the network after the one-step update from $\Nc$. The difference of output is $\Delta \Nc(\xb):= \Nc'(\xb) - \Nc(\xb)$. 
Recall that parameters $v_0, s_k, t_k$ follow the updated rule \eqref{eq: dv0} $\sim$ \eqref{eq: dtk} such that network have a constant output on $\partial C_k \cap B_k$ (Figure \ref{fig: MSE loss convergence dynamics}(c) and (d)). This implies that both networks $\Nc$ and $\Nc'$ have fixed outputs for $\partial C_k \cap B_k$, and then the affine space connecting those fixed points also has the fixed output which comes from the piecewise linearity of $\Nc$.

\textbf{STEP 3-1} First, we compute an upper bound of $|\Delta L_3|$. 
Since $\Nc(\xb)$ is piecewise linear, we consider the input space partition in $A_k \backslash B_k$ where $\Nc$ is linear on. Observing the `corner' parts of the polytope $C$ (see Figure \ref{fig: MSE loss convergence dynamics}(c) and (d)), each partition is intersection of some neurons of $\Nc$. Choose one partition $P \subset A_k \backslash B_k$, and let $J_P \subset [\width]$ be the index set of $P$ that $\wb_j$ is activated on $P$ if and only if $j \in J_P$, or namely, $P = \bigcap_{j \in J_P} A_j$. 
Then obviously, the partition $P$ is contained in a ball with radius $\max_{j\in J_P} t_j \le t < R$. On the contrary, any partition $P$ is contained in $t_k$-radius ball from $\partial_k$ for some $k$. Using this, we can disjointly separate the partitions to $\Qc_k$ such that
\begin{enumerate}
    \item $Q_k \subset ( A_k \backslash B_k )$
    \item Every $P \in \cup_{k\in[\width]}( A_k \backslash B_k )$ is exactly contained in one of $Q_k$.
    \item Every $P \in Q_k $ can be bounded by a ball with radius $t_k$.
\end{enumerate}
Note that $\Qc_k$ is a collection of partitions, which can be empty. 
Using this, we decompose $L_3$ by the following way. This is just rearranging the terms in $L_3$.
\begin{align*}
    L_3 &= \frac12 \sum_{\xb_i \in \cup_{k \in [\width]} (A_k \backslash B_k)} (\Nc(\xb_i) - y_i)^2 \\
    &= \frac12 \sum_{k \in [\width]} \sum_{\xb_i \in \Qc_k} (\Nc(\xb_i) - y_i)^2 \\
    &=: \frac12 \sum_{k \in [\width]} L_{{3,k}}. 
\end{align*}

Now, we bound the change of network output $\Delta\Nc(\xb_i)$ for $\xb_i \in P \in \Qc_k$.
\begin{align*}
    |\Delta \Nc (\xb_i) | &= \left| \Nc(\xb_i) - \Nc'(\xb_i) \right| \\
    &\le \left| \sum_{j \in J_P} \Delta\left( -\frac{v_0}{t_j} \right) t_k \right| \\
    &= \sum_{j \in J_P} \eta R \\
    &\le l R \eta.
\end{align*}

Above inequalities come from the fact that, the change of linear value is bounded by product of the change of slope and the maximum diameter of the set.
Finally, for a $k \in [\width]$, we compute an upper bound of the loss variation of $L_{3,k}$.
\begin{align} 
    |\Delta L_{3, k}| 
    &= \left|\frac{1}{2n} \sum_{\xb_i \in  (A_k \backslash B_k)}  (\Nc(\xb_i) + \Delta \Nc(\xb_i)-y_i)^2 - \frac{1}{2n} \sum_{\xb_i \in  (A_k \backslash B_k)}  (\Nc(\xb_i)-y_i)^2\right| \nonumber \\
    &= \left| \frac{1}{n} \sum_{\xb_i \in  (A_k \backslash B_k)} \left(\Nc(\xb_i)-y_i + \frac12 \Delta \Nc(\xb_i)\right) \cdot \Delta \Nc(\xb_i) \right| \nonumber \\
    &\le  \frac{1}{n} \sum_{\xb_i \in (A_k \backslash B_k)} \left( |\Nc(\xb_i)-y_i| \cdot | \Delta \Nc(\xb_i) | + \frac12|\Delta \Nc(\xb_i)|^2 \right) \nonumber \\ 
    &\le  \frac{1}{n} \sum_{\xb_i \in (A_k \backslash B_k)} \left( 1\cdot | \Delta \Nc(\xb_i) | + \frac12|\Delta \Nc(\xb_i)|^2 \right) \nonumber \\
    &\le \frac1n  \#\left( \Bc_{t_k}(\partial^2 C_k) \right) \cdot (\width R\eta + \frac12 \width^2R^2\eta^2) \nonumber \\
    &\le \frac2n \#\Big( \Bc_{t_k}(\partial^2 C_k) \Big) \cdot \width R\eta \nonumber \\
    &\le \frac{2\width R\eta}{n} \cdot \rho \; \#(\Bc_{t_k}(\partial C_k)) . \label{eq: Delta L3}
\end{align}
Note that we used $\eta < \frac{2}{\width R}$ to bound the quadratic term $\eta^2$.

\textbf{STEP 3-2.} Now, we compute a similar bound for $\Delta L_2$. It can be decomposed to the sum on each $B_k$.
\begin{align*}
    L_2 &= \frac{1}{2n} \sum_{k \in [\width]} \sum_{\xb_i \in B_k} (\Nc(\xb_i) - y_i)^2 \\
    &=: \sum_{k \in [\width]} L_{2,k}.
\end{align*} 

We use the fact that each data point $\xb_i$ is far from $\partial C$ at least $\delta$. 
From the definition, we get $\Delta \Nc(\xb_i) > \delta \eta$.
Note that if $t_k < 2\delta$, then $L_{3,k}$ strictly decreases and we have nothing to do. Otherwise, when $t_k > 2\delta$, we have $R>2\delta$ and there is data far from $\delta$ distance from $\partial B_k$. For such data point $\xb_i$, we get
\begin{align*}
    \Nc(\xb_i) - 0 &= v_0 - \frac{v_0}{t_k}(t_k - \delta) \\
    &= v_0 \frac{\delta}{t_k} \\
    &> \frac{v_0\delta}{R}
\end{align*}
and
\begin{align*}
    1 - \Nc(\xb_i) &= 1- \left(v_0- \frac{v_0}{t_k}(t_k - \delta)\right) \\
    &= 1- \frac{v_0}{t_k}\delta \\
    &> 1 - \frac12 v_0 \\
    &> \frac12 v_0 \\
    &> \frac{v_0\delta}{R}.
\end{align*}
Therefore, we have shown that
\begin{align*}
    \min_{ s_k+\delta \le~ \norm{\xb_i} ~\le s_k+t_k-\delta} 
    |\Nc(\xb_i) - y_i| \ge \frac{v_0 \delta}{R}.
\end{align*}
Now we induce the lower bound of $\Delta L_{2,k}$.
\begin{align}
    \Delta L_{2,k}  &= \frac{1}{2n} \sum_{\xb_i \in B_k} \left( (\Nc(\xb_i)+\Delta \Nc(\xb_i) - y_i)^2 - (\Nc(\xb_i) - y_i)^2\right) \nonumber \\
    &= \frac1n \sum_{\xb_i \in B_k} \left( \Nc(\xb_i)-y_i + \frac12 \Delta \Nc(\xb_i) \right) \cdot \Delta \Nc(\xb_i) \nonumber \\
    &= \frac1n \sum_{\xb_i \in B_k} \left( -|(\Nc(\xb_i)-y_i)| \cdot \left|\Delta \Nc(\xb_i)\right| + \frac12 \left|\Delta \Nc(\xb_i)\right|^2 \right) \nonumber \\
    &\le \frac1n \sum_{\xb_i \in B_k} \left( -\min_{ \delta \le~ \norm{\xb_i}-s_k ~\le t_k-\delta} |\Nc(\xb_i) - y_i| \cdot \eta\delta + \frac12 R^2\eta^2 \right)
    \nonumber \\
    &\le -\frac1n \#\left(\Bc_{\frac{t_k}{2}-\delta}(\partial C_k) - \Bc_{t_k}(\partial^2 C_k)\right) \cdot \left(\frac{v_0 \delta}{R} \cdot \delta\eta - \frac12 R^2 \eta^2 \right)  \nonumber \\
    &\le - \frac1n (1-\rho) \#(\Bc_{t_k}(\partial C_k)) \cdot \left(\frac{v_0 \delta^2\eta}{R} - \frac12 R^2\eta^2 \right) \nonumber \\
    &< \frac1n \frac{(1-\rho) v_0 \delta^2 \eta}{2R} \; \#(\Bc_{t_k}(\partial C_k)) . \label{eq: Delta L2}
\end{align} 
Note that Assumption \ref{asmp: dataset} on dataset $\Dc$ is used to induce this inequality. Now, we compare \eqref{eq: Delta L2} and \eqref{eq: Delta L3} with initialization condition \eqref{eq: v0 init MSE} for $v_0$. Then, we finally get
\begin{align*}
    \left| \Delta L_{3,k}\right| 
    &\le \frac{2lR\eta\rho}{n} \; \#(\Bc_{t_k}(\partial C_k)) \\
    &< \frac{(1-\rho) v_0 \delta^2}{2nR}\eta  \cdot \#(\Bc_{t_k}(\partial C_k)) \\
    &< -\Delta L_{2,k}
\end{align*}
for every $k \in [\width]$. By summing up, we conclude $ |\Delta L_3 | < -\Delta L_2$ or,
\begin{align} \label{eq: the third loss}
   \Delta L_2 + \Delta L_3 <0.
\end{align}

\textbf{STEP 4.}
Finally, we combine all results in the previous steps.
When $\#(\cup_{k\in[\width]}A_k) >0$, only $L_2$ and $L_3$ are changed, then one step update gives
\begin{align*}
    L' &= L + \Delta L \\
    &= L_1 + L_2 + L_3 + \Delta L_1 + \Delta L_2 + \Delta L_3 \\
    &< L_1 + L_2 + L_3
\end{align*}
from \eqref{eq: the second loss} and \eqref{eq: the third loss}.
Furthermore, since $s_k<l_k$ increases and $t_k>0$ decreases, the updated parameters satisfy the assumption \eqref{eq: assumption 1} again. 
Using mathematical induction, we can repeat above steps until $\#(\cup_{k\in[\width]}A_k) =0$. After achieving $\#(\cup_{k\in[\width]}A_k) =0$, we get $L_2 = L_3 =0$ from their definition \eqref{eq: L123}. Then, the remained loss $L_1$ exponentially decreases to zero because
\begin{align*}
    L' &= L_1 + \Delta L_1 \\
    & \le \left(1 - \frac12 \delta \eta\right) L_1 \\
    & \le \left(1 - \frac12 \delta \eta\right) L
\end{align*}
from \eqref{eq: the first loss}. 
This 
completes the proof.
\hfill $\square$

\subsubsection{Proof for the BCE loss \texorpdfstring{\eqref{eq: BCE loss}}{function}.}
The proof idea is similar with the previous proof. We use the same definitions for $A_k, B_k, s_k, t_k, l_k$, and other notations. The BCE loss \eqref{eq: BCE loss} is rearranged by

\begin{align}
    L_{BCE} &= -\frac1n \sum_{i=1}^n \left(y_i \log \SIG \circ \Nc(\xb_i) + (1-y_i) \log (1 - \SIG \circ \Nc(\xb_i)) \right)  \nonumber \\
    &= -\frac1n \sum_{\xb_i \in A_0}  \log \SIG \circ \Nc(\xb_i) \nonumber \\
    &  -\frac1n \sum_{k \in [\width]} \sum_{\xb_i \in B_k} \left(y_i \log \SIG \circ \Nc(\xb_i) + (1-y_i) \log (1 - \SIG \circ \Nc(\xb_i)) \right) \label{eq: BCELoss123} \\
    & -\frac1n \sum_{k \in [\width]} \sum_{\xb_i \in (A_k \backslash B_k)} \left(y_i \log \SIG \circ \Nc(\xb_i) + (1-y_i) \log (1 - \SIG \circ \Nc(\xb_i)) \right) \nonumber \\
    &=: L_1 + L_2 + L_3. \nonumber 
\end{align}

Before we start, we compute the derivative and its bound of some functions. For $\zeta \in \Rd$, define
\begin{align*}
    f(\zeta) &:= \log \SIG (\zeta) , \\
    g(\zeta) &:= \log (1 - \SIG (\zeta) ).
\end{align*}
Then their derivatives are given by
\begin{align*}
    \frac{d}{d\zeta} f(\zeta) &:= 1 - \SIG (\zeta), \\
    \frac{d}{d\zeta} g(\zeta) &:= - \SIG (\zeta).
\end{align*}
From the mean value theorem (MVT), we get
\begin{align*}
    f(\zeta+\Delta \zeta) &= f(\zeta) + f'(\zeta) \Delta \zeta + \frac12 f''(\tilde\zeta) (\Delta \zeta)^2 \\
    &\ge f(\zeta) + (1-\SIG(\zeta)) \Delta \zeta - \frac12 (\Delta \zeta)^2
\end{align*}
and
\begin{align*}
    f(\zeta +\Delta \zeta) - f(\zeta) 
    &= (1-\SIG(\tilde\zeta))\Delta \zeta \\
    &\le \Delta \zeta.
\end{align*}
Now we defin the update of parameters. 
For $k\in[l]$, the update of $v_0$ is given by
\begin{align}
        \Delta v_0 &:= 
    \begin{cases}
        \qquad 0 \qquad &\text{if }\quad \#(\cup_{k \in [l]} A_k) > 0, \\
        (1-\SIG(v_0))\eta. \qquad & \text{otherwise}
    \end{cases} \label{eq: dv0 bce}
\end{align}
For $\Delta s_k$ and $\Delta t_k$, we adopt the same update defined in \eqref{eq: dsk} and \eqref{eq: dtk}. 
Namely, the update of parameters preserves the value of $\Nc$ on $l_k$ and the change of slope is set to $-\eta$. We repeat the analogous arguments in the previous proof.

\textbf{STEP 1.} 
Firstly, we consider the first loss term $L_1$ in \eqref{eq: BCELoss123} when $\#(\cup_{k \in [\width]} A_k) = 0$. It is changed by
\begin{align*}
    \Delta L_1 &= L_1' - L_1 \\
    &= - \frac1n \sum_{\xb_i \in A_0} \log \SIG (v_0 + \Delta v_0) 
    +\frac1n \sum_{\xb_i \in A_0} \log \SIG (v_0) \\
    &= -\frac{\#(A_0)}{n} (f(v_0+\Delta v_0) - f(v_0)) \\
    &\le -\frac{\#(A_0)}{n} \left( (1-\SIG (v_0)) \Delta v_0 - \frac12 (\Delta v_0 )^2\right) \\
    &= -\frac{\#(A_0)}{n} \left( (1 - \SIG(v_0))^2\eta - \frac12 (1 - \SIG(v_0))^2\eta^2 \right) \\
    &< -\frac{\#(A_0)}{n} \frac12 (1 - \SIG(v_0))^2\eta. 
\end{align*}
Therefore, $L_1$ strictly decreases. Note that we used $\eta <1$ to bound the $\eta^2$ term. 

\textbf{STEP 2.}
Secondly, we consider when $\#(\cup_{k \in [\width]} A_k) > 0$.
As discussed in the previous subsection, $\Nc(\xb_i)$ strictly increases (or decreases) if and only if $y_i = 1$ (or 0, respectively) because the slope $-\frac{v_0}{t_k}$ changes $-\eta$. This shows that $\Delta L_2 < 0$.

\textbf{STEP 3.}
Thirdly, we observe $\Delta L_2$ and $|\Delta L_3|$ when $\#(\cup_{k \in [\width]} A_k) > 0$.
We compute a bound of $\Delta L_3$ first. For any $k \in [\width]$,
\begin{align*}
    |\Delta L_3| 
    &= \frac1n\bigg| \; \sum_{\xb_i \in (A_k \backslash B_k)} y_i(f(\Nc(\xb_i)+\Delta \Nc(\xb_i)) - f(\Nc(\xb_i))) \\
    &\hspace{3cm} + (1-y_i)(g(\Nc(\xb_i)+\Delta \Nc(\xb_i)) - g(\Nc(\xb_i)))\; \bigg| \\
    &\le \frac1n \sum_{\xb_i \in (A_k \backslash B_k)} 
    \Big|(f(\Nc(\xb_i)+\Delta \Nc(\xb_i)) - f(\Nc(\xb_i)))\Big| 
    +\Big|(g(\Nc(\xb_i)+\Delta \Nc(\xb_i)) - g(\Nc(\xb_i)))\Big| \\
    &< \frac1n \sum_{\xb_i \in (A_k \backslash B_k)}  2|\Delta \Nc(\xb_i)| \\
    &< \frac{2}{n} \; \#(\Bc_{t_k}(\partial^2 C_k)) \cdot \max_{\xb_i \in (A_k \backslash B_k)} |\Delta \Nc(\xb_i)| \\ 
    &< \frac{2R\eta}{n} \; \#(\Bc_{t_k}(\partial^2 C_k)). 
\end{align*}

We obtain a similar bound for $\Delta L_{2,k}$. Let $V_0:=\log\left(\frac{(1-\rho)\delta}{4\rho R} - 1\right)$ be the upper bound of initialization of $v_0$. Note also that $\SIG(V_0) = 1-\frac{4\rho R}{(1-\rho)\delta}$ and $\eta < \frac{1-\SIG(v_0)}{\delta}$. Then,
\begin{align*}
    \Delta L_{2,k}
    &= -\frac1n \sum_{\xb_i \in B_k} \bigg( y_i(f(\Nc(\xb_i)+\Delta \Nc(\xb_i)) - f(\Nc(\xb_i))) \\
    &\hspace{1cm} + (1-y_i)(g(\Nc(\xb_i)+\Delta \Nc(\xb_i)) - g(\Nc(\xb_i)))\bigg) \\
    &= -\frac1n \sum_{\xb_i \in B_k, y_i=1} \bigg( (f(\Nc(\xb_i)+\Delta \Nc(\xb_i)) - f(\Nc(\xb_i))) \\
    &\hspace{1cm} -\frac1n \sum_{\xb_i \in B_k, y_i=0} (g(\Nc(\xb_i)+\Delta \Nc(\xb_i)) - g(\Nc(\xb_i))) \bigg) \\
    &< -\frac1n \sum_{\xb_i \in B_k, y_i=1} \left((1-\SIG \circ \Nc(\xb_i)) \Delta \Nc(\xb_i) - \frac12 (\Delta \Nc(\xb_i))^2\right) \\
    &\hspace{1cm} -\frac1n \sum_{\xb_i \in B_k, y_i=0} \left( -\SIG \circ \Nc(\xb_i) \cdot \Delta \Nc(\xb_i) - \frac12 (\Delta \Nc(\xb_i))^2\right)  \\
    &< -\frac1n \sum_{s_k-l_k+\delta<h_i<-\delta} \left((1-\SIG \circ \Nc(\xb_i)) \Delta \Nc(\xb_i) - \frac12 (\Delta \Nc(\xb_i))^2\right) \\
    &\hspace{1cm} -\frac1n \sum_{\delta<h_i<s_k+t_k-\delta} \left(-\SIG \circ \Nc(\xb_i) \cdot \Delta \Nc(\xb_i) - \frac12 (\Delta \Nc(\xb_i))^2\right) \\
    &< -\frac1n \sum_{s_k-l_k+\delta<h_i<-\delta} \left((1-\SIG(V_0)) \delta \eta - \frac12 \delta^2 \eta^2 \right) \\
    &\hspace{1cm} -\frac1n \sum_{\delta<h_i<s_k+t_k-\delta} \left(\SIG(0) \cdot \delta \eta - \frac12 \delta^2 \eta^2 \right) \\
    &< -\frac1n \; \#\left(\Bc_{\frac{t_k}{2}-\delta}(\partial C_k) - \Bc_{t_k}(\partial^2 C_k)\right) \cdot  \left( (1-\SIG(V_0)) \delta \eta -\frac12 \delta^2 \eta^2 \right) \\
    &< -\frac1n (1-\rho) \#\left( \Bc_{t_k}(\partial C_k)\right) \cdot \frac12 (1-\SIG(V_0)) \delta \eta.
\end{align*}

Therefore, 
\begin{align*}
    |\Delta L_{3,k}|
    &< \frac{2R\eta}{n} \rho\; \#(\Bc_{t_k}(\partial C_k)) \\
    &< \frac1n (1-\rho) \#\left( \Bc_{t_k}(\partial C_k)\right) \cdot \frac12 (1-\SIG(V_0)) \delta \eta \\
    &< -L_{2,k}
\end{align*}
and we get $\Delta L_{2} + \Delta L_3 <0$. 

\textbf{STEP 4.} Finally, we combine results in the previous steps.
When $\#(\cup_{k \in [\width]} A_k) > 0$, $v_0$ is bounded by $V_0$ and we get $\Delta L_2 + \Delta L_3 <0$ from \textbf{STEP 3}. After update, since $s_k<l_k$ increases and $t_k>0$ decreases, the updated parameters satisfy Assumption \ref{asmp: dataset} again. It is repeated with strictly decreasing loss until reaching $\#(\cup_{k \in [\width]} A_k) = 0$.
After that, $v_0$ begins to strictly increase, which strictly decreases all $L_1$, $L_2$, and $L_3$.
Further, the update equation \eqref{eq: dv0 bce} provides $v_0$ goes to infinity. Therefore, $\Nc(\xb_i) \rightarrow \infty$ if and only if it label $y_i=1$, concludes $L_{BCE}$ converges to zero. 

This completes the whole proof of Theorem \ref{thm: convergence}. 
\hfill $\square$
\section{Additional Propositions and Lemmas} \label{app: rebuttal}

\begin{lemma} \label{lem: simplex}
    Let $0\le m \le d$ be integers, and $\Delta^m$ be an $m$-simplex in $\Rd^d$. For a given $\varepsilon>0$, there exists a two-layer ReLU network $\Tc: \Rd^d \rightarrow \Rd$ with the architecture $\ReLUtwo{d}{(d+1)}{1}$ such that
    \begin{align*}
        \Tc(\xb) &= 1       \qquad \text{if } \xb \in \Delta^m, \\
        \Tc(\xb) &\le 1     \qquad \text{if } \xb \in B_{\varepsilon}(\Delta^m), \\
        \Tc(\xb) &< 0       \qquad \text{if } \xb \not\in B_\varepsilon(\Delta^m).
    \end{align*}
    Furthermore, the minimal width of such two-layer ReLU networks with the architecture $\ReLUtwo{d}{d_1}{1}$ is exactly $d_1=d+1$.
\end{lemma}
\begin{proof}
    We prove the existence part first.
    For the given $m$-simplex $\Delta^m$, pick $(d-m)$ distinct points in $\Bc_\varepsilon(\Delta^m)$. By connecting all these points with the points of $\Delta^m$, we obtain a $d$-simplex contained in $\Bc_\varepsilon(\Delta^m)$, which is a convex polytope.
    By Proposition \ref{prop: convex polytope}, there exists a neural network $\Tc:\Rd^d \rightarrow \Rd$ with the architecture $\ReLUtwo{d}{d_1}{1}$ that satisfies the desired properties.

    Now, we prove the minimality part. For every $\varepsilon>0$, suppose there exists a two-layer ReLU network $\Tc(\xb) := \sum_{i=1}^{d_1} v_i \sigma( \wb_i^\top \xb+b_i) + v_0$ with $d_1 \le d$ such that $\Tc(\xb)=1$ for $\xb\in\Delta^m$ and  $\Tc(\xb)<0$ for $\xb \not\in \Bc_\varepsilon(\Delta^m)$.
    First, we claim that the set of weight vectors $\{\wb_1, \cdots, \wb_{d_1}\}$ spans $\Rd^d$. If the set cannot span $\Rd^d$, then there exists a nonzero vector $\ub \in \Rd^d - \text{span}<\wb_1, \cdots, \wb_{d_1}>$. Then, from $\Tc(\xb)=1$ for $\xb\in \Delta^m$, we get
    \begin{align*}
        \Tc(\xb+t\ub) &= \sum_{i=1}^{d_1} v_i \sigma(\wb_i^\top(\xb+t\ub)+b_i) + v_0 \\
        &= \sum_{i=1}^{d_1} v_i \sigma(\wb_i^\top\xb+b_i) + v_0 \\
        &= \Tc(\xb) \\
        &=1
    \end{align*}
    for any $t\in\Rd$. This contradicts to the condition $\Tc(\xb)<0$ for $\xb \not \in \Bc_\varepsilon(\Delta^m)$. Therefore, the set of weight vectors must span $\Rd^d$.

    From the above claim, we further deduce that $d_1 \ge d$. Since we start with the assumption $d_1 \le d$, thus $d_1=d$. Then, we conclude that the set of weight vectors $\{\wb_1, \cdots, \wb_{d_1}\}$ is a basis of $\Rd^d$.
    Now, we focus on the sign of $v_0$. Suppose $v_0 \ge 0$. Define 
    \begin{align*}
        A:= \bigcap_{i=1}^{d_1} \{\xb ~|~ \wb_i^\top\xb+b_i <0 \},
    \end{align*}
    which is an unbounded set since the set $\{\wb_i\}$ is linearly independent. Then for $\xb\in A$, we get $\Tc(\xb)=v_0 \ge 0$. This contradicts to the assumption $\Tc(\xb)<0$ for all $\xb \not \in \Bc_\varepsilon(\Delta^m)$. 
    Therefore, $v_0<0$. 
    
    Lastly, we consider the sign of $v_i$. Since $\Tc(\xb)=1>0$ for $\xb\in \Delta^m$ and $v_0<0$, there exists some positive $v_i>0$, say, $v_1>0$.
    Similar to the above argument, we define
    \begin{align*}
        B:= \left\{\xb ~|~ v_1\wb_1^\top\xb+b_1 +v_0 >0 \right\}\; \bigcap_{i=2}^{d_1} \left\{\xb ~|~ \wb_i^\top\xb+b_i <0 \right\},
    \end{align*}
    which is also nonempty and unbounded. Then, for $\xb\in B$, we have 
    \begin{align*} 
        \Tc(\xb) &= \sum_{i=1}^{d_1} v_i \sigma(\wb_i^\top\xb+b_i)+v_0 \\
        &= v_1 \wb_1^\top\xb+b_1 + v_0 \\
        &>0.
    \end{align*}
    Since $B$ is unbounded, this implies that $\Tc(\xb)>0$ over the unbounded subset in $\Rd^d$, which contradicts to the condition $\Tc(\xb)<0$ for all $\xb\not\in \Bc_\varepsilon(\Delta^m)$. This completes the whole proof, which shows that the minimum width of two-layer ReLU network is exactly $d+1$.
\end{proof}

\begin{proposition} \label{prop: no topology} 
    Let $\Xc\subset \Rd^d$ be a topological space and $\Ac$ be a neural network architecture that is a \exact on $\Xc$. Then, there exists a topological space $\Xc'$ which is homeomorphic to $\Xc$, but $\Ac$ is not a \exact on ${\Xc'}$.
\end{proposition} 
\begin{proof}
We use the similar technique introduced in \cite{telgarsky2015representation}.
Before we start, recall that a network $\Nc$ with the architecture $\Ac$ is a piecewise linear function on $\Rd^d$. Thus $\Rd^d$ can be partitioned into finitely many regions, where $\Nc$ is linear on each region. Let $M$ be the maximum number of such regions, that networks with the architecture $\Ac$ can partition. I.e., any network with the architecture $\Ac$ has linear regions at most $M$ partitions in $\Rd^d$. 

Now, we consider a contractible topological space $\Yc$ which has zig-zag shape as described in Figure \ref{fig: no topology}(b), where the number of sawtooths is greater than $M + 2$. We define another topological space $\Xc':=\Xc\#\Yc$, where $\#$ denotes the connected sum. Note that we can glue $\Yc$ to $\Xc$ preserving the number of sawtooths in $\Yc$, because $\Xc$ is bounded. Then $\Xc'$ is homeomorphic to $\Xc$ since $\Yc$ is contractible.

Finally, we prove the proposition using contradiction. Suppose there exists a deep ReLU network $\Nc'$ with the same architecture $\Ac$, which can approximate $\indicator{\Xc'}$ under the given error bound $\varepsilon>0$. 
Then, by the $\Yc$ part in $\Xc'$, there exists a straight line $\ell$ that intersects $\Xc'$ more than $M+3$ times. Therefore, to approximate $\indicator{\Xc'}$ sufficiently close, $\Nc'$ must have at least $M+1$ linear regions on $\ell$. However, $\Nc'$ can have at most $M$ linear regions in $\Rd^d$ from the definition of $M$. This contradiction completes the proof.
\end{proof}

\begin{theorem} \label{thm: regression}
    Let $d_x,d_y \in \Nd$ and $p\ge1$.
    Then, the set of three-layer ReLU networks is dense in $L^p(\Rd^{d_x}, [0,1]^{d_y})$.
    Furthermore, let $f : \Rd^{d_x} \rightarrow [0,1]^{d_y}$ be a compactly supported function whose Lipschitz constant is $L$.
    Then, for any $\varepsilon>0$, there exists a three-layer ReLU network $\Nc$ with the architecture 
    \begin{align*}
        d_x \;\stackrel{\sigma}{\rightarrow}\;\ReLUtwo{(2n d_x d_y)\;}{\;(n d_y)\;}{d_y}
    \end{align*}
    such that $\norm{\Nc - f}_{L^p(\Rd^{d_x})}<\varepsilon$. 
    Here, $n=\varepsilon^{-d_x} \left( 1+ (\sqrt{d_x}L)^p \right)^{{d_x}/{p}} = O(\varepsilon^{-d_x})$.
\end{theorem}
\begin{proof}
    Fist we recall a result in real analysis: the set of compactly supported continuous functions is dense in $L^p(\Rd^{d_x})$ for $p \ge 1$ \citep[Theorem 3.14]{rudin1976principles}. 
    Therefore, it is enough to prove the second statement; which claims that any compactly supported Lipschitz function can be universally approximated by three-layer ReLU networks.
    
    We consider $d_y=1$ case first. 
    Let $f\in \Rd^{d_x} \rightarrow [0,1]$ be Lipschitz, and let $L$ be its Lipschitz constant. Without loss of generality, suppose the support of $f$ is contained in $[0,1]^{d_x}$.
    Let $\delta>0$ be the small number which will be determined. Now we partition $[0,1]^{d_x}$ by regular $d_x$-dimensional cubes with length $\delta$. Now, consider estimating the definite integral uses a Riemann sum over cubes.
    The total number of cubes are $n:=(\frac1\delta)^{d_x}$, and we number these cubes by $C_1, C_2, \cdots, C_n$. For each cube $C_i$, by Proposition \ref{prop: convex polytope}, we can define a two-layer ReLU network $\Tc_i$ with the architecture $\ReLUtwo{d_x}{2d_x}{1}$ such that $\Tc_i(\xb)=1$ in $C_i$ and $\Tc_i(\xb)=0$ for $\xb \not\in B_{r}(C_i)$ with $r:=\frac{1}{2d_x} \frac{\delta^{p+1}}{1+\delta^p}$.
    Then for any $\xb_i \in C_i$, we get
    \begin{align*}
        \int_{B_{r}(C_i)} |f-f(\xb_i)\Tc_i|^p \; d\mu
        &= \int_{C_i} |f-f(\xb_i)\Tc_i|^p \; d\mu + \int_{B_{r}(C_i) \backslash C_i} |f-f(\xb_i)\Tc_i|^p \; d\mu \\
        &\le \int_{C_i} (\sqrt{d_x}L\delta)^p \; d\mu  + \int_{B_{r}(C_i) \backslash C_i} 1^p \; d\mu \\
        &\le (\sqrt{d_x}L\delta)^p \cdot \delta^{d_x} + \left[ (\delta+2r)^{d_x} - \delta^{d_x} \right] \\
        &=  (\sqrt{d_x}L)^p \cdot \delta^{d_x+p} + \left[\left(1+\frac{2r}{\delta}\right)^{d_x} -1\right] \delta^{d_x}  \\
        &< \left[ (\sqrt{d_x}L)^p + 1 \right] \delta^{d_x+p}. 
    \end{align*}
    Note that we use two inequalities, $|f(\xb) - f(\xb_i)| \le L\sqrt{d_x}\delta$ for $\xb\in C_i$ and $(1+a)^k < \frac{1}{1-ak}$ for $0<a<\frac{1}{k}$. 
    Then, the above equation implies the $L^p$ distance between $f$ and $f(\xb_i)\Tc_i$ in $C_i$ is bounded by the above value. 
    Now we define a three-layer neural network $\Nc$ by
    \begin{align*}
        \Nc(\xb) := \sum_{i=1}^n f(\xb_i) \Tc_i(\xb),
    \end{align*} 
    which is a Riemann sum over the $n$ cubes partitions.
    Then $\Nc$ has the architecture ${d_x}\stackrel{\sigma}{\rightarrow}\ReLUtwo{(2n d_x)}{n}{1}$ and satisfies
    \begin{align*}
        \int_{\Rd^{d_x}} |f-\Nc|^p d\mu 
        &= \int_{B_r([0,1]^{d_x})} |f-\Nc|^p \; d\mu \\
        &< \sum_{i=1}^n \int_{B_r(C_i)} |f-f(\xb_i)\Tc_i|^p \; d\mu \\
        &\le \left[ (\sqrt{d_x}L)^p + 1 \right] n \delta^{d_x+p}. \\
        &= \left[ (\sqrt{d_x}L)^p + 1\right] \delta^p.
    \end{align*}
    Therefore, take $\delta <\varepsilon (1+(\sqrt{d_x}L)^p)^{-\frac1p}$ for given $\varepsilon$, we conclude that $\norm{f-\Nc}_{L^p([0,1]^{d_x})} < \varepsilon$. 
    From this choice of $\delta$, we get
    \begin{align*}
        n = \delta^{-d_x} > \varepsilon^{-d_x} \left( 1+ (\sqrt{d_x}L)^p \right)^{{d_x}/{p}} 
        = O(\varepsilon^{-d_x}).
    \end{align*}
    If $d_y>1$, we can obtain the desired network by concatenating $d_y$ networks, thus the architecture is 
    \begin{align*}
       {d_x}\stackrel{\sigma}{\rightarrow}\ReLUtwo{(2n d_x d_y)}{(n d_y)}{d_y}.
    \end{align*}
\end{proof}

\begin{figure*}[t]
    \centering
    \subfigure[]{\includegraphics[width=0.2\textwidth]{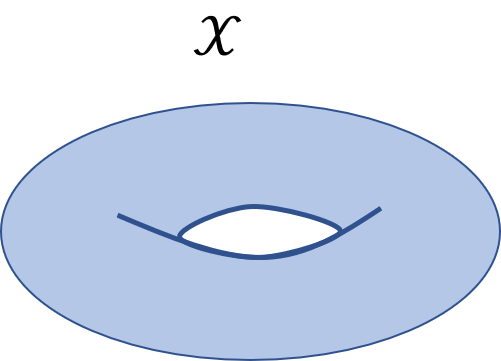}}
    \hfill
    \subfigure[]{\includegraphics[width=0.35\textwidth]{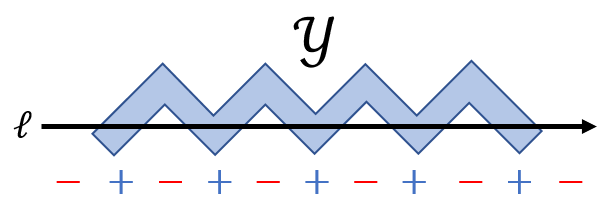}}
    \hfill
    \subfigure[]{\includegraphics[width=0.35\textwidth]{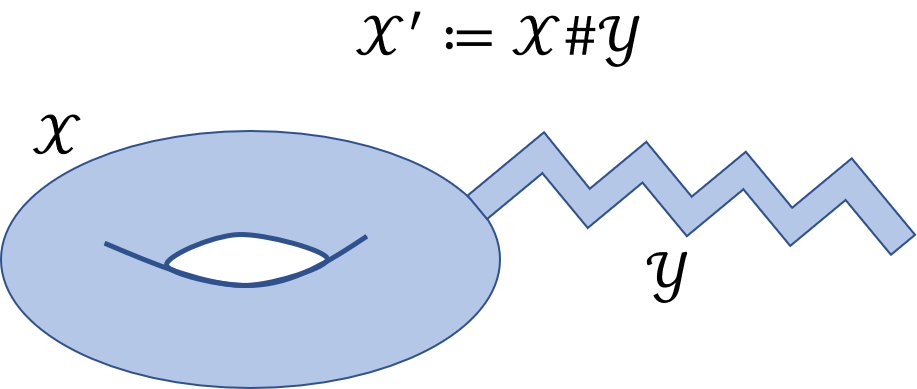}}
    \caption{Proof of Proposition \ref{prop: no topology}. 
    (a) $\Xc$ is a given topological space, and $\Ac$ is a \exact on $\Xc$. 
    (b) $\Yc$ is a zig-zag shaped long band, which is a contractible space. There exists a straight line $\ell$ such that $\Yc$ and $\ell$ has sufficiently many intersection points, so that $\Ac$ cannot approximate $\Yc$.
    (c) $\Xc'$ is the connected sum of $\Xc$ and $\Yc$, which is homeomorphic with $\Xc$. However, $\Ac$ is not a \exact on $\Xc'$.
    }
    \label{fig: no topology}
\end{figure*}

\begin{lemma} \label{lem: two-layer convex}
    Let $\Tc$ be a two-layer ReLU network defined in \eqref{eq: two-layer constant bias}. Then, the classification region $R := \{\xb\in\Rd^d ~|~ \Tc(\xb) > 0\}$ is a convex polytope. Specifically, if the subset $S:=\{\xb\in\Rd^d ~|~ \Tc(\xb)=\lambdabias\}$ is nonempty, then it is a convex polytope with $\width$ faces. 
\end{lemma}
\begin{proof}
    First, we prove that $\Tc$ is a concave function. Note that $\sigma$ is convex thus $v_k\sigma(\wb_k^\top\xb+b_k)$ is a concave function with respect to input $\xb$, and the sum of concave functions is again concave (we use all $v_k<0$ here).
    Therefore, $\Tc$ is a concave function, and the region $R:=\{\xb~|~\Tc(\xb)>0\}$ is convex. The piecewise linearity of $\Tc$ implies that $R$ forms a convex polytope.
    
    Now we consider the subset $S:= \{\xb~|~\Tc(\xb)=\lambdabias\}$. Since all $v_k<0$, $\xb \in S$ if and only if $\wb_k^\top\xb+b_k \le 0$ for all $k\in[\width]$. Then, $S$ is a convex polytope with $\width$ faces by Definition \ref{def: convex polytope}.
\end{proof}

\begin{proposition}[Theorem 2.1 in \cite{du2018algorithmic}, two-layer version]
\label{prop: balanced property}
    Let $\Nc(\xb) := v_0 + \sum_{k=1}^l v_k \sigma(\wb_k^\top\xb+b_k)$ be a two-layer ReLU network, and $L=\frac1n \sum_{i=1}^n \ell(\Nc(\xb_i), y_i)$ be the loss function.
    Then, on the gradient flow, for all $k\in[l]$, the quantity
    \begin{align} \label{eq: balanced}
        v_k^2 - \norm{\wb_k}^2 -b_k^2
    \end{align}
    is invariant.
\end{proposition}
\begin{proof}
    The proof is written in \cite{du2018algorithmic}, and we provide here for completeness. 
    The gradient of each component is computed by
    \begin{align*}
        \frac{\partial L}{\partial v_k} &= \frac1n \sum_{i=1}^n \frac{\partial \ell}{\partial \Nc(\xb_i)} \cdot \sigma({\wb_k^\top\xb_i+b_k}), \\
        \frac{\partial L}{\partial \wb_k} &= \frac1n \sum_{i=1}^n \frac{\partial \ell}{\partial \Nc(\xb_i)} \cdot v_k \indicator{\wb_k^\top\xb_i+b_k>0} \xb_i, \\
        \frac{\partial L}{\partial b_k} &= \frac1n \sum_{i=1}^n \frac{\partial \ell}{\partial \Nc(\xb_i)} \cdot v_k \indicator{\wb_k^\top\xb_i+b_k>0}.
    \end{align*}
    Then, it is easy to check that 
    $$ v_k \frac{\partial L}{\partial v_k} = \wb_k^\top (\frac{\partial L}{\partial \wb_k}) + b_k \cdot \frac{\partial L}{\partial v_k}. $$
    Now, we differentiate \eqref{eq: balanced}. It gives
    \begin{align*}
        \frac{d}{dt} (v_k^2 - \norm{\wb_k}^2 -b_k^2)
        &= 2 v_k \frac{dv_k}{dt} - 2 \wb_k^\top (\frac{d\wb_k}{dt}) - 2 b_k \frac{db_k}{dt} \\
        &= 2\left( -v_k \frac{\partial L}{\partial v_k} + \wb_k^\top (\frac{\partial L}{\partial \wb_k}) + b_k\frac{\partial L}{\partial b_k} \right) \\
        &=0
    \end{align*}
    for all $t$. Therefore, \eqref{eq: balanced} is constant.
\end{proof}

\begin{proposition} \label{prop: faces}
    Consider the neural network architecture $\ReLUThree{d}{d_1}{d_2}\ReLUtwo{\cdots}{d_D}{1}$ and a convex polytope $C$ with $m$ faces. Then,
    \begin{enumerate}
        \item if $d_1 \ge m$ and $d_2 \ge 1$, then it is a \exact on $C$.
        \item if $\max_j d_j \le m-1$, then it may not be a \exact on some polytope $C$.
        \item if $d_1 \le m-2$, then it may not be a \exact on some polytope $C$.
    \end{enumerate}
\end{proposition}
\begin{proof}
    The proof is accomplished by two strategies: for a feasible architecture, we explicitly construct such neural networks. For the negative statements, we prove them by providing some counterexamples. 
    \begin{enumerate}
        \item 
        Proposition \ref{prop: convex polytope} shows that $\ReLUtwo{d}{m}{1}$ is a \exact. Therefore, for $d_1 \ge m$, then taking the identity for all other layers, it becomes a \exact.

        \item 
        When $d_1 \le \width-1$, there is a $\width$-faces convex polytope $C$ that cannot be approximated by the given network architecture. The simplest example is a half-space ($\width=1$). 
        
        Below, we provide another non-trivial example: 
        Let $C$ be a $d$-simplex in $\Rd^d$, thus $\width=d+1$.
        Suppose $\max_j d_j \le \width-1 = d$. Then, by Lemma \ref{lem: decreasing widths}, we conclude that the classified regions are always unbounded. Therefore, it cannot approximate a bounded polytope $C$.

        \item
        From the above proof, recall the $d$-simplex $C$ (thus $m=d+1$). 
        If $d_1 \le m-2 = d-1$, we provide a counterexample proving that it cannot be approximated by a ReLU network with the architecture $\ReLUfour{d}{(m-2)}{\cdots}{d_D}{1}$. 
        Let $\wb_1, \cdots, \wb_{\width-2}$ be the weight vectors of the first layer. Since the dimension of $span<\wb_k>$ is equal or less than $\width-2 = d-1$, there exists a nonzero vector $\hat\wb \in span<\{\wb_k\}_{k\in[m]}> ^ \perp$. In other words, $\hat \wb ^\top \wb_k = 0$ for all $k\in[m-2]$. Then, it implies $\Nc(\xb + t\hat\wb) = \Nc(\xb)$ for all $t\in\Rd$. Therefore, $\Nc$ cannot approximate the bounded polytope $C$. 
    \end{enumerate}
\end{proof}
\begin{lemma} \label{lem: decreasing widths}
    Let $\Nc$ be a deep ReLU network where all hidden dimension is equal or smaller than the input dimension $d$. Suppose $\mu(\{\Nc(\xb)>0 \})>0$. Then, $\mu(\{\Nc(\xb)>0 \})$ is either $0$ or $\infty$. In other words, the classification region is either measure-zero or unbounded.
\end{lemma}
\begin{proof}
    \citet[Theorem 2]{beise2021decision} showed that if all hidden layers have width equal or smaller than the input dimension, then the connected components of every decision region are unbounded. Therefore, $\mu(\{\Nc(\xb)>0 \})$ is either $0$ or $\infty$, depends on whether $\{\Nc(\xb)>0 \}$ is empty or not.
\end{proof}

\end{document}